\documentclass[10pt]{article} %
\input{header}
\def\1{\bm{1}}

\newcommand{\hgrad}[1]{{\nabla \widetilde{f}}(#1)}
\newcommand{\htgrad}[1]{{\nabla \widetilde{f}^\tau}(#1)}
\newcommand{\taugrad}[1]{{\nabla f^\tau}(#1)}

\def\ftau{f^\tau}

\def\indicator#1{{\mathbbm{1}\left\{ #1 \right\}}}

\def\pistar{\pi^*}

\def\pit{{\pi_{\theta_t}}}

\def\pitheta{\pi_{\theta}}

\def\thetat{{\theta_t}}
\def\tht{{\theta_t}}
\def\thtt{{\theta_{t+1}}}
\def\thetatt{{\theta_{t+1}}}

\def\etat{{\eta_t}}

\DeclarePairedDelimiter\parens{\lparen}{\rparen}
\DeclarePairedDelimiter\abs{\lvert}{\rvert}
\DeclarePairedDelimiter\norm{\lVert}{\rVert}
\DeclarePairedDelimiter\braces{\lbrace}{\rbrace} %
\DeclarePairedDelimiter\bracks{\lbrack}{\rbrack} %
\DeclarePairedDelimiter\angles{\langle}{\rangle}

\def\ev#1{{\E \bracks*{ #1 }}}

\def\dpd#1{{\angles*{ #1 }}}
\def\eps{{\epsilon}}

\def\supnorm#1{\norm*{ #1 }_\infty}
\def\normsq#1{{\norm*{ #1 }^2_2}}
\def\tnorm#1{{\norm*{ #1 }_2}}

\DeclareMathAlphabet{\mathsfit}{\encodingdefault}{\sfdefault}{m}{sl}
\SetMathAlphabet{\mathsfit}{bold}{\encodingdefault}{\sfdefault}{bx}{n}

\def\gA{{\mathcal{A}}}

\def\gO{{\mathcal{O}}}
\def\gP{{\mathcal{P}}}

\def\gS{{\mathcal{S}}}

\def\sI{{\mathbb{I}}}

\newcommand{\E}{\mathbb{E}}

\newcommand{\R}{\mathbb{R}}

\DeclareMathOperator*{\argmax}{arg\,max}
\DeclareMathOperator*{\argmin}{arg\,min}

\def\cT{{\mathcal{T}}}

\newcommand{\grad}[1]{\nabla f(#1)}

\newcommand{\gradf}[1]{\nabla f(#1)}

\usepackage[accepted]{rlc}

\SetKwRepeat{Do}{do}{while}
\SetKwComment{Comment}{/* }{ */}

\title{Towards Principled, Practical Policy Gradient for Bandits and Tabular MDPs}

\author{Michael Lu \\ michael\_lu\_3@sfu.ca  \\ Simon Fraser University 
\AND  Matin Aghaei \\ matin\_aghaei@sfu.ca \\  Simon Fraser University
\AND 
Anant Raj \\ araj@inria.fr \\ SIERRA Project Team (Inria) 
\AND Sharan Vaswani \\  vaswani.sharan@gmail.com \\ Simon Fraser University}

\begin{document}

\maketitle

\begin{abstract}
We consider (stochastic) softmax policy gradient (PG) methods for bandits and tabular Markov decision processes (MDPs). While the PG objective is non-concave, recent research has used the objective's smoothness and gradient domination properties to achieve convergence to an optimal policy. However, these theoretical results require setting the algorithm parameters according to unknown problem-dependent quantities (e.g. the optimal action or the true reward vector in a bandit problem). To address this issue, we borrow ideas from the optimization literature to design practical, principled PG methods in both the exact and stochastic settings. In the exact setting, we employ an Armijo line-search to set the step-size for softmax PG and demonstrate a linear convergence rate. In the stochastic setting, we utilize exponentially decreasing step-sizes, and characterize the convergence rate of the resulting algorithm. We show that the proposed algorithm offers similar theoretical guarantees as the state-of-the art results, but does not require the knowledge of oracle-like quantities. For the multi-armed bandit setting, our techniques result in a theoretically-principled PG algorithm that does not require explicit exploration, the knowledge of the reward gap, the reward distributions, or the noise. Finally, we empirically compare the proposed methods to PG approaches that require oracle knowledge, and demonstrate competitive performance. 
\end{abstract}

\section{Introduction}
Policy gradient (PG) methods have played a vital role in the achievements of deep reinforcement learning (RL) \citep{10.5555/3009657.3009806, schulman2017proximal}. Recent theoretical research~\citep{agarwal2021theory,mei2020global,mei2021understanding,bhandari2021linear,lan2023policy,shani2020adaptive} have analyzed PG methods in simplified settings, exploiting the objective's properties to guarantee global convergence to an optimal policy. We focus on \emph{softmax policy gradient methods} that parameterize the policy using the softmax function, and consider the \emph{tabular parameterization} for which the number of parameters scales with the number of states and actions. For this class of methods, recent studies have established global convergence rates in both the exact~\citep{mei2020global, mei2021understanding,agarwal2021theory} and stochastic (inexact) settings~\citep{mei2021understanding, mei2022the, mei2023stochastic, yuan2022general}. 

Specifically, in the exact setting where the rewards and transition probabilities are known,~\citet{agarwal2021theory} proved that softmax PG can attain asymptotic convergence to an optimal policy despite the non-concave nature of the PG objective.~\citet{mei2020global} improve this result and quantify the rate of convergence, proving that softmax PG requires $\gO(\nicefrac{1}{\eps})$ iterations to converge to an $\epsilon$-optimal policy. 
On the other hand, when using the tabular parameterization in the exact setting, natural policy gradient (NPG)~\citep{kakade2001natural} and geometry-aware normalized policy gradient (GNPG)~\citep{mei2021leveraging} have been shown to achieve a linear convergence~\citep{bhandari2021linear, cen2022fast, lan2023policy, xiao2022convergence} matching policy iteration.

In the stochastic setting where the rewards and transition probabilities are unknown and algorithms require sampling from the environment,~\citep{zhang2020global} first proved that REINFORCE~\citep{williams1992simple, sutton1999policy} converges to a first-order stationary point at an $\tilde{\gO}(\nicefrac{1}{\eps^2})$ rate.~\citet{mei2021understanding, mei2022the} analyzed the convergence of stochastic softmax PG, proving that it requires $\gO(\nicefrac{1}{\eps^2})$ iterations to converge to an $\eps$-optimal policy. However, the resulting algorithm requires the full gradient (which in turn requires the knowledge of the environment) to set algorithm parameters, making it impractical in the stochastic setting. Similarly,~\citet{yuan2022general} proved that stochastic softmax PG converges to an optimal policy at a slower $\tilde{\gO}(\nicefrac{1}{\eps^3})$ rate. However, this result requires knowledge of the optimal action making it vacuous. More recently,~\citet{mei2023stochastic} analyzed stochastic softmax PG in the multi-armed bandit setting and proved that it converges to the optimal arm at an $\gO(\nicefrac{1}{\eps})$ rate. Unfortunately, the algorithm requires knowledge of the reward gap which is typically unknown for bandit problems. 

Consequently, while the above convergence results are notable, the methods that stem from them are impractical. The impracticality arises from the methods' dependence on oracle-like knowledge of the environment, which includes factors such as the optimal action~\citep{yuan2022general}, reward gap~\citep{mei2023stochastic} and even access to the full gradient~\citep{mei2021understanding} in stochastic settings. The need for this oracle-like knowledge renders these methods ineffective because they assume access to information sufficient to derive an optimal policy. In this paper, our objective is to \emph{design practical softmax PG methods while retaining theoretical convergence guarantees to the optimal policy}. We believe that this is an important first step towards developing practical but theoretically-principled PG methods in the general function approximation setting. To this end, we make the following contributions. \par

\textbf{Contribution 1}: In~\cref{section:pg_wo_entropy}, we first consider the exact setting as a test bed for analyzing softmax PG. In this setting, theoretical step-sizes that enable convergence to the optimal policy are often too conservative in practice. We present a practical approach by employing an Armijo line-search~\citep{armijo1966} to set the step-size for softmax PG. Armijo line-search enables adaptation to the objective's local smoothness which results in larger step-sizes and improved empirical performance. 
Furthermore, we design an alternative line-search condition that takes advantage of the objective's non-uniform smoothness and enables softmax PG to use larger step-sizes. The resulting algorithm achieves linear convergence matching GNPG~\citep{mei2021leveraging}.
    
\textbf{Contribution 2}: In Section \ref{section:spg_wo_entropy}, we consider the stochastic setting where the policy gradient is estimated using finitely many interactions with an environment.  To design a practical softmax PG algorithm that can adapt to the stochasticity, we utilize exponentially decreasing step-sizes~\citep{li2021second, vaswani2022towards}. The resulting algorithm matches the $\tilde{\gO}(\nicefrac{1}{\eps^3})$ rate of~\citet{yuan2022general} without the knowledge of oracle-like information. In order to attain faster convergence, we use the strong growth condition (SGC)~\citep{schmidt2013fast,vaswani2019fast} satisfied by the PG objective~\citep{mei2023stochastic}. We prove that the same algorithm with exponentially decreasing step-sizes is robust to unknown problem-dependent constants and can effectively interpolate between the fast $\tilde{\gO}(\nicefrac{1}{\eps})$ and slow $\tilde{\gO}(\nicefrac{1}{\eps^3})$ rate.

\textbf{Contribution 3}: Finally, in~\cref{section:exp}, we experimentally benchmark the proposed algorithms in the bandit setting. Our empirical results indicate that the proposed algorithms have comparable performance as baselines that require oracle-like knowledge. 

\textbf{Contribution 4}: In~\cref{appendix:entropy}, we study the use of entropy regularization for PG methods in both the exact and stochastic settings. Entropy regularization has been successfully used in RL~\citep{haarnoja2018soft, hiraoka2022dropout}. It helps smooth the objective function, enabling PG methods to escape flat regions and allowing the use of larger step-sizes~\citep{ahmed2019understanding}. Although entropy regularization allows for faster convergence, it results in convergence to a biased policy. 

We introduce a practical multi-stage algorithm that iteratively reduces the entropy regularization and ensures convergence to the optimal policy. The resulting algorithm does not require the knowledge of any problem dependent constants such as the reward gap (as in prior work~\citep{mei2020global}). 
Under additional assumptions, we prove that softmax PG with entropy regularization converges to the optimal policy at an $\tilde{\gO}(\nicefrac{1}{\eps})$ rate in the exact setting and at an $\tilde{\gO}(\nicefrac{1}{\eps^3})$ rate in the stochastic setting. Although we do not prove a theoretical advantage of  entropy regularization; in practice, we find that adding entropy enables the resulting algorithms to be more robust to ``bad'' initializations.

\section{Problem Setup \& Background}
\label{sec:background}
An infinite-horizon discounted Markov decision process (MDP) \citep{puterman2014markov} 
is defined by tuple $\parens*{\gS, \gA, \gP, r, \rho, \gamma}$, where $\gS$ is the set of states, $\gA$ is the set of actions, $\gP : \gS \times \gA \rightarrow \Delta_{\gS}$ is the transition probability function, $\rho \in \Delta_{\gS}$ is the initial state distribution, $r : \gS \times \gA \rightarrow [0, 1]$ is the reward function, and $\gamma \in [0, 1)$ is the discount factor. We will only consider \emph{tabular MDPs}, assuming that the state and action spaces are finite and define $S := |\gS|$ and $A := |\gA|$. For policy $\pi$, the \textit{action-value function} $Q^{\pi} : \gS \times \gA \rightarrow \R$ is defined as: $Q^{\pi}(s, a) := \E\bracks*{\sum_{t=0}^\infty \gamma^t r(s_t, a_t)}$, with $s_0 = s$, $a_0 = a$ and for $t \geq 1$, $s_{t+1} \sim p(\cdot | s_t, a_t)$ and $a_{t+1} \sim \pi(\cdot | s_t)$. The corresponding \textit{value function} $V^{\pi}: \gS \rightarrow \R$ is defined as $V^{\pi}(s) := \E_{a \sim \pi(\cdot | s)}[Q^{\pi}(s, a)]$. The \textit{advantage function} $A^{\pi}: \gS \times \gA \rightarrow \R$ is defined as $A^{\pi}(s, a) := Q^{\pi}(s, a) - V^{\pi}(s)$. For state $s \in \gS$, we define $\Pr^\pi[s_t = s \, | s_0]$ to be the probability of visiting state $s$ at time $t$ under policy $\pi$ when starting at state $s_0$. The \textit{discounted state visitation distribution} is denoted by $d^{\pi}_{s_0} \in \Delta_\gS$ and defined as $d^{\pi}_{s_0} := (1 - \gamma) \sum_{t=0}^\infty \gamma^t \Pr^\pi[s_t = s \, | s_0]$.

Given a class of feasible policies $\Pi$, the policy optimization objective is: $\max_{\pi \in \Pi} J(\pi) := \E_{s \sim \rho}[V^{\pi}(s)]$. For brevity, we define $V^{\pi}(\rho) := \E_{s \sim \rho}[V^{\pi}(s)]$. We denote the optimal policy as $\pistar = \argmax_{\pi \in \Pi} J(\pi)$. Throughout this paper, we will consider both the general MDP setting and the \emph{bandits} setting. For the bandit setting, $S = 1$ and $\gamma = 1$, and the corresponding objective is to find a policy that maximizes $\E[\langle \pi, r \rangle]$ where the expectation is over the stochastic rewards. 

In this work, we consider policies with a \emph{softmax tabular parameterization}, i.e. for parameters $\theta \in \R^{S \times A}$, the set $\Pi$ consists of policies $\pitheta: \gS \rightarrow \Delta_{\gA}$ parameterized using the softmax function such that $\pitheta(a | s) = \nicefrac{\exp(\theta(s, a))}{\sum_{a' \in \gA}\exp(\theta(s, a'))}$. Such a tabular parameterization has been recently used to study the theoretical properties of policy gradient methods~\citep{agarwal2021theory,mei2020global}. 
Throughout, we will present our results considering $f(\theta)$ as an abstract objective with specific properties, and when required, instantiate it in the general MDP or bandits setting.
In the general MDP setting, $f(\theta) := V^{\pitheta}(\rho)$, while in the bandits setting, $f(\theta) :=  \langle \pitheta, r \rangle$. With this abstraction, we hope that our results can be easily generalized to other settings such as constrained MDPs~\citep{altman2021constrained} or convex MDPs~\citep{zahavy2021reward,zhang2020variational}. Next, we specify the properties of $f$ that will be used to analyze the convergence of PG methods. 

\begin{table}[h]
\centering
\setlength\tabcolsep{4.7pt}
\begin{tabular}{|c|c|c|c|c|c|c|}
\hline
Setting     &  $f(\theta)$  &  $[\nabla f(\theta)]_{s, a}$  & $L$ &  $L_1$  & $\nu$ & $C(\theta)$  \\ \hline
Bandits & $\langle \pitheta, r \rangle$ & $\pitheta(a) \, [r(a) - \dpd{\pitheta, r}]$ & $5/2$  & $3$  & $\frac{\sqrt{2}}{\Delta^*}$ &  $\pitheta(a^*)$ \\ \hline
MDP  & $V^{\pitheta}(\rho)$ & $\frac{d^{\pitheta}(s) \, \pitheta(a | s) \, A^{\pitheta}(s, a)}{1 - \gamma}$  & $\frac{8}{(1 - \gamma)^3}$  & $\bracks*{3 + \frac{2 \, C_\infty - (1 - \gamma)}{(1 -\gamma) \gamma}} \, \sqrt{S}$ &  $\frac{\sqrt{2}}{(1- \gamma) \, \Delta^*}$ & $\frac{\min_s \pitheta(a^*(s) | s)}{\sqrt{S} \, \supnorm{\frac{d^{\pi^*}_\rho}  {d^{\pitheta}_\rho}}}$  \\ 
& & & & & & \\ \hline
\end{tabular}
\caption{Function and gradient expressions, (non)-uniform smoothness, non-uniform and reversed \L ojasiewciz properties for bandits and general tabular MDPs with $\xi = 0$~\citep{mei2020global}. 
Here, $a^*$ is index of the optimal arm in the bandit problem , $C_\infty := \max_\pi \supnorm{\frac{d^\pi_\rho}{\rho}}$ is the distribution mismatch ratio~\citep{agarwal2021theory}, and $\Delta^* := \min_{s} Q^*(s, a^*(s)) - \max_{a(s) \neq a^*(s)}Q^*(s, a(s))$ is the reward gap corresponding to the optimal policy.}
\label{table:c_theta}
\end{table}

First, we note that $f$ is a non-concave function for both bandits and general MDPs~\citep[Proposition 1]{mei2020global}. However, in both cases, it is twice-differentiable and $L$-smooth, i.e. for all $\theta$, there exists a constant $L \in (0, \infty)$, $\nabla^2 f(\theta) \preceq L I_{SA}$. Since this property holds for all $\theta$ and $L$ is a constant independent of $\theta$, we refer to this as \emph{uniform smoothness}. For both bandits and general MDPs, $f$ also satisfies a notion of \emph{non-uniform smoothness}, i.e. for all $\theta$, there exists a $L_{1} \in (0, \infty)$ such that $\nabla^2 f(\theta) \preceq L_1 \, \norm{\nabla f(\theta)} I_{SA}$. Intuitively, non-uniform smoothness states that the landscape is flatter closer to a stationary point $\tilde{\theta}$, meaning that as $\theta \rightarrow \tilde{\theta}$, $\nabla^2 f(\theta) \rightarrow \mathbf{0}$, i.e. the Hessian becomes degenerate. Together, the uniform and non-uniform smoothness properties are related to the $(L_0, L_1)$ smoothness recently used to study the optimization of transformer models~\citep{zhang2019gradient}. 

Since the rewards are bounded, $f(\theta)$ is upper-bounded by a value $f^* := \max_{\theta} f(\theta)$. Furthermore, $f$ satisfies a \emph{non-uniform \L ojasiewciz condition}, i.e. for all $\theta$, there exists a $C(\theta) \in (0, \infty)$ and $\xi \in [0, 1]$ such that $\norm{\grad{\theta}}_2 \geq C(\theta) \, \abs{f^* - f(\theta)}^{1 - \xi}$~\citep{mei2020global}. For the special case where $C(\theta)$ is an absolute constant and $\xi = \nicefrac{1}{2}$, this condition matches the well studied Polyak \L ojasiewciz (P\L) condition~\citep{POLYAK1963864,karimi2016linear}. The \L ojasiewciz condition states that every stationary point $\tilde{\theta}$ (s.t. $\nabla f(\tilde{\theta}) = 0$) is also a global maximum s.t. $f(\tilde{\theta}) = f^*$. This condition enables the convergence of local ascent methods such as PG to an optimal solution $\theta^* := \argmax_\theta f(\theta)$ despite the problem's non-concavity~\citep{karimi2016linear,mei2020global,agarwal2021theory}. Finally, $f$ satisfies a \textit{reversed \L ojasiewciz} condition, i.e. for all $\theta$, there exists a $\nu > 0$ such that $\norm{\grad{\thetat}} \leq \nu \, (f^* - f(\theta))$~\citep{mei2020global}. This condition bounds how quickly the gradient norm vanishes near the optimal solution.~\cref{table:c_theta} summarizes both the uniform and non-uniform smoothness and \L ojasiewciz properties for bandits and general MDPs. 

Similar to~\citet{mei2020global}, we assume a uniform starting state distribution, i.e. $\forall s \in \gS$, $\rho(s) = \nicefrac{1}{S}$ and hence $C_\infty \leq \frac{1}{\min_s \rho(s)} < \infty$. This is a common assumption in the policy gradient literature that obviates the need for exploration in the general MDP setting and allows us to exclusively focus on the optimization aspects. We note that for both these settings, the optimal policy is deterministic~\citep{puterman2014markov} i.e. in the general MDP setting, for each state $s \in \gS$, there is an action $a^*(s) \in \gA$ such that $\pistar(a^*(s)|s) = 1$ and for all $a \neq a^*(s)$, $\pistar(a|s) = 0$. This implies that when using the softmax tabular parameterization, $\theta^*(s, a^*(s)) \rightarrow \infty$ and for all $a \neq a^*(s)$, $\theta^*(s, a) \rightarrow -\infty$. This property is similar to that for logistic regression for classification on linearly separable data~\citep{ji2018risk}.      

In the next section, we will use the above properties of $f$ and study the convergence of PG methods in the exact setting.

\section{Policy Gradient in the Exact Setting}
\label{section:pg_wo_entropy}
We first consider the exact setting that assumes complete knowledge of the rewards and transition probabilities, and consequently enables the exact calculation of the policy gradient. 
This setting has been used as a test bed to study the convergence properties of PG methods~\citep{bhandari2021linear,agarwal2021theory,mei2020global}. 

\emph{Softmax policy gradient} (softmax PG) uses gradient ascent to iteratively maximize $f(\theta)$. In particular, at iteration $t \in [T]$, softmax PG uses a step-size of $\etat$ and has the following update:
\vspace{0.6ex}
\begin{update}(Softmax PG, True Gradient)
$\thetatt = \thetat + \etat \grad{\thetat}$. 
\label{update:dpg} 
\end{update}
Refer to~\cref{table:c_theta} for the gradient expressions of the policy gradient $\nabla f(\theta)$ in both the bandits and general MDP cases.

In this setting,~\citet{mei2020global} prove that softmax PG converges to an optimal solution at an $\gO(\nicefrac{1}{T})$ rate, implying that the algorithm requires $\gO(\nicefrac{1}{\eps})$ iterations to guarantee that $f^* - f(\theta_{T+1}) \leq \epsilon$. From a policy optimization perspective, this implies that softmax PG can return a stochastic policy whose value function is $\eps$ close to the optimal policy's value function. In order to achieve this convergence,~\citet{mei2020global} requires using a constant step-size $\etat = \eta = \nicefrac{1}{L}$. Furthermore, for any $\etat \in (0, 1]$,~\citet[Theorem 9]{mei2020global} proves an $\Omega(\nicefrac{1}{\eps})$ lower-bound showing that this rate is tight. 

In most scenarios, we can only obtain a loose upper-bound on the smoothness $L$. This over-estimation of $L$ implies that the resulting step-size is typically smaller than necessary, often resulting in worse empirical performance.  In practice, when doing gradient ascent with access to the exact gradient, it is standard to employ a \emph{line-search}~\citep{armijo1966, nocedal1999numerical} to adaptively set the step-size in each iteration. This results in faster empirical convergence while requiring minimal tuning, and preserving the rate of convergence. Hence, we propose to use a backtracking Armijo line-search~\citep{armijo1966} to adaptively set the step-size for softmax PG. 

At every iteration $t$, backtracking Armijo line-search starts from an initial guess for the step-size ($\eta_{\mathrm{max}}$) and backtracks until the \textit{Armijo condition} is satisfied. In particular, the procedure thus returns the largest step-size $\etat$ such that following condition is satisfied:     
\begin{align}
f(\thetat + \eta_t \grad{\thetat}) \geq f(\thetat) + h \etat \normsq{\grad{\thetat}} \,, \; \; \; \text{(Armijo condition)}
\label{eq:armijo}
\end{align}
where $h \in (0, 1)$ is a hyper-parameter. For smooth functions, the backtracking procedure is guaranteed to terminate and return a step-size $\etat$ that satisfies $\etat \geq \min \braces*{\nicefrac{2(1-h)}{L}, \eta_{\mathrm{max}}}$. Hence, Armijo line-search guarantees improvement in the function value (ensuring monotonic policy improvement at each iteration $t$), while selecting a step-size larger than the $\nicefrac{1}{L}$ step-size used in~\citet{mei2020global}. 

The following theorem shows that using the Armijo line-search preserves the theoretical $\gO(\nicefrac{1}{T})$ convergence rate.  
\vspace{1ex}
\begin{restatable}{theorem}{theorempgls}\label{theorem:pg_line_search}
Assuming $f$ is (i) $L$-smooth, (ii) satisfies the non-uniform \L ojasiewciz condition with $\xi = 0$, and (iii) $\mu := \inf_{t \geq 1} [C(\thetat)]^2 > 0$, using \cref{update:dpg} and Armijo line-search to set the step-size results in the following convergence:
\begin{equation}\label{eq:linesearch_ascent}
    f^* - f(\theta_{T+1}) \leq \max \braces*{\frac{L}{2\,h\,(1-h)}, \frac{1}{h\,\eta_{\mathrm{max}}}} \, \frac{1}{\mu \, T}
\end{equation}
where $h \in (0, 1)$ and $\eta_{\max}$ is the upper-bound on the step-size.
\end{restatable} 
While assumptions (i) and (ii) are satisfied for both the general MDP and bandit settings, we need to ensure that assumption (iii) also holds. We first note that this property holds for a constant step-size $\etat = \eta = \nicefrac{1}{L}$~\citep[Lemma 5, Lemma 9]{mei2020global}. However, the proof can be extended to any varying step-size sequence that guarantees ascent ($f(\thetatt) \geq f(\thetat)$) in every iteration. When using the Armijo line-search to set the step-size, this condition is satisfied by definition, thus guaranteeing that $\mu := \inf_{t \geq 1}[C(\tht)]^2 > 0$.

The Armijo condition in~\cref{eq:armijo} takes advantage of the objective's uniform smoothness in order to attain an $\gO(\nicefrac{1}{T})$ convergence. In our initial experiments, we observed that for most iterations, the maximum step-size $\eta_{\max}$ satisfies the Armijo condition, and is hence returned by the line-search procedure. By using a sufficiently large $\eta_{\max}$ or by progressively increasing the maximum step-size as a function of $t$, the resulting algorithm converges at a linear rate. This is because the objective satisfies a non-uniform smoothness property and the optimization landscape becomes flatter as the gradient norm decreases closer to the solution. This enables the use of larger step-sizes than those returned by the Armijo line-search when using a fixed $\eta_{\max}$. In order to take advantage of the non-uniform smoothness more explicitly, we design an alternative line-search on the logarithm of the suboptimality. Formally, we use the following condition:
\begin{equation}\label{eq:trans_line_search}
   \ln(f^* - f(\thetat + \etat \, \grad{\thetat})) \leq  \ln(f^* - f(\thetat)) - h \, \etat \, \frac{\normsq{\grad{\thetat}}}{f^* - f(\tht)}  
   \quad \text{(Armijo condition for log-loss)}.
\end{equation}
When using the above condition, \cref{lemma:trans_armijo_step_size_lb} guarantees that the backtracking line-search procedure terminates and returns $\etat \geq \min\braces*{\eta_{\max}, \frac{2(1-h)}{L_1 \, \nu \, [f^* - f(\tht)]}}$ (refer to \cref{table:c_theta} for the values of $L_1$ and $\nu$). Hence, the resulting line-search accepts step-sizes proportional to $\frac{1}{f^* - f(\tht)}$, meaning that as the optimization progresses and $f(\tht) \to  f^*$, larger step-sizes can be used.

The following theorem (proved in \cref{appendix:proof_pg_line_search_exp_new}) characterizes the rate of convergence of softmax PG when using the Armijo condition for the log-loss in \cref{eq:trans_line_search}.
\vspace{1ex}
\begin{restatable}{theorem}{theorempglsexpnew}\label{theorem:pg_line_search_exp_new}
For a given $\eps \in (0, 1)$, assuming $f$ is (i) $L_1$ non-uniform smooth, (ii) satisfies the non-uniform \L ojasiewciz condition with $\xi = 0$, (iii) $\mu := \inf_{t \geq 1}[C(\thetat)]^2 > 0$, (iv) $f$ satisfies a reversed \L ojasiewciz condition with $\nu > 0$, using \cref{update:dpg} with backtracking line-search using the Armijo condition in \cref{eq:trans_line_search} and setting $\eta_{\max} = \nicefrac{C}{\eps}$ results in the following convergence: \\ If $f^* - f(\thetat) > \eps$ for all $t \in [1, T]$, then,
\begin{equation}
   f^* - f(\theta_{T+1}) \leq [f^* - f(\theta_1)] \, \exp\parens*{-\min\left\{C \, h, \frac{2\,h\,(1-h)}{L_1 \, \nu}\right\} \, \mu \, T}
\label{eq:armijo-log-loss}
\end{equation}
where $C > 0$ and $h \in (0, 1)$ are hyper-parameters. Otherwise $\min_{t \in [1, T]} f^* - f(\tht) \leq \eps$.
\end{restatable}
\vspace{-2ex}
For a target $\eps$, setting $T = \gO\left(\log\left(\nicefrac{1}{\eps}\right)\right)$ iterations results in a linear convergence rate.  In comparison to \cref{theorem:pg_line_search}, using the Armijo condition in \cref{eq:armijo-log-loss} enables the use of larger step-sizes resulting in a faster ($\gO(\nicefrac{1}{\eps})$ vs $\gO\left(\log\left( \nicefrac{1}{\eps}\right)\right))$ rate. However, the Armijo condition in~\cref{eq:armijo-log-loss} requires the knowledge of $f^*$, making the resulting method less practical. This requirement is similar to the Polyak step-size~\citep{polyak1987introduction} used for gradient descent. For future work, we hope to remove this dependence of $f^*$. In comparison, the geometry-aware normalized policy gradient (GNPG) approach introduced in \citet{mei2021understanding} also explicitly exploits this non-uniform smoothness and exhibits a convergence rate of $\gO\left(\log\left(\nicefrac{1}{\eps}\right)\right)$. However, in the general MDP setting, GNPG requires the knowledge of unknown constants such as the concentrability coefficient $C_\infty := \max_\pi \supnorm{d^\pi_\rho / \rho}$ to determine the step-size, making it impractical. 
In concurrent work, \citet{liu2024elementary} show that softmax PG with \textit{any} constant step-size can attain an $\Theta(\nicefrac{1}{\eps})$ convergence to the optimal policy. Moreover, they prove that softmax PG with a specific adaptive step-size scheme that only depends on the advantage function and the policy (PG-A) can attain a fast $\gO\left(\log\left(\nicefrac{1}{\eps}\right)\right)$ convergence.

\begin{figure}[!h]
  \begin{center}       
  \includegraphics[scale=0.5]{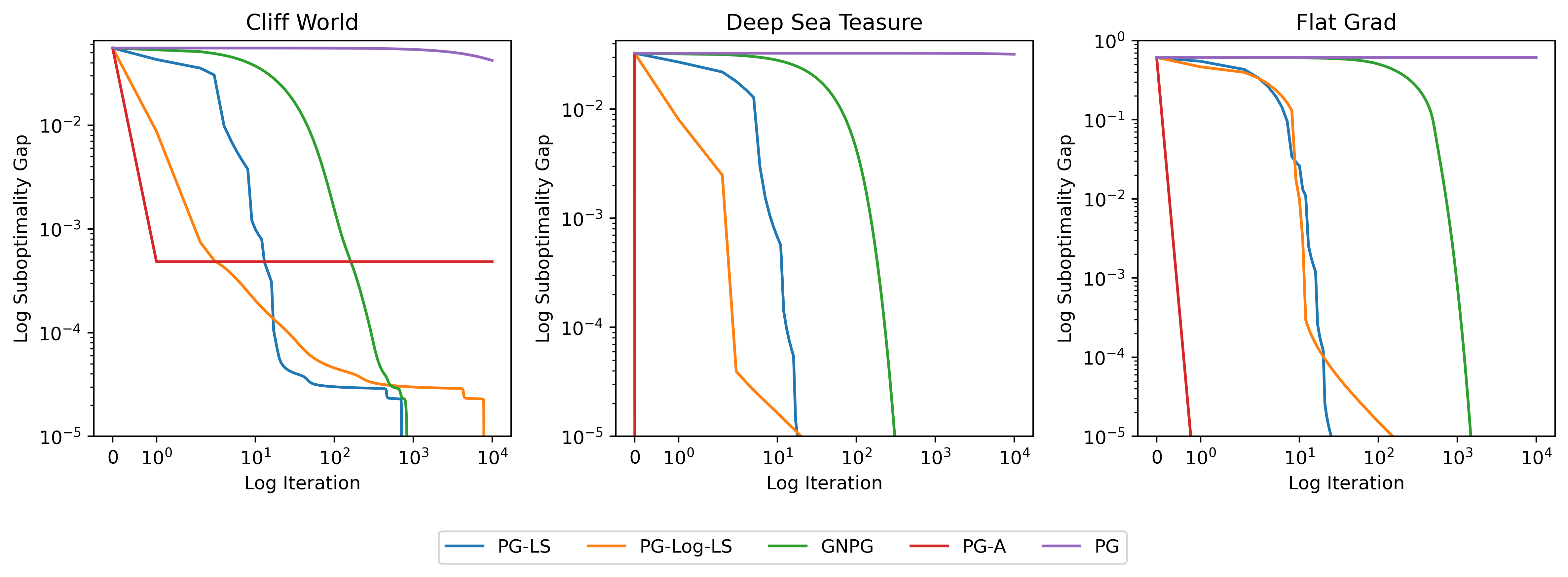}
  \caption{Comparing softmax PG that (i) uses a step-size that satisfies the Armijo condition in~\cref{eq:armijo} (denoted as \texttt{PG-LS}), (ii) uses a step-size that satisfies the Armijo condition in~\cref{eq:trans_line_search} (\texttt{PG-Log-LS}) to GNPG (\texttt{GNPG}), PG-A (\texttt{PG-A}) and PG with a fixed step-size (\texttt{PG}) in the tabular MDP setting.}
    \label{fig:det_mdp}
  \end{center}
  \vspace{-3ex}
\end{figure}
In~\cref{fig:det_mdp}, we compare the presented line-search methods with the Armijo condition in \cref{eq:armijo} and the Armijo condition on the log-loss in \cref{eq:trans_line_search} to GNPG, PG-A and PG with a constant step-size on three tabular MDP environments (see \cref{appendix:experiments} for details). For the methods that use backtracking line-search, we set $\eta_{\max} = \frac{1}{\eps}$ with $\eps = 10^{-4}$ and $h=0.5$. For $\texttt{GNPG}$, we use the step-size of $\etat = \frac{(1 - \gamma) \, \gamma}{6 \, (1 - \gamma) + 4 \, (S^{-1} - (1 - \gamma))}$. Since $C_\infty$ is unknown, we upper-bound it as: $C_\infty \leq \min_s \frac{1}{\rho(s)} = \frac{1}{S}$. For $\texttt{PG-A}$, we use the theoretical step-size $\etat = \frac{1}{\min_{s \in \hat{\gS}_t} \max_a \abs{\hat{A}_t(s, a)}}$ where $\hat{A}_t(s, a) := \pit(a | s) \, A^{\pit}(s, a)$ and $\hat{\gS}_t := \{ s \in \gS \mid \hat{A}_t(s, a) > 0\}$.
Finally for $\texttt{PG}$ we use a constant step-size of $\etat = \frac{1}{L} = \frac{(1 - \gamma)^3}{8}$.
We observe that \texttt{PG-LS} is comparable to \texttt{GNPG} and \texttt{PG-A} while \texttt{PG-Log-LS} can better exploit the non-uniform smoothness, enabling larger step-sizes as the algorithm approaches the optimal policy. The performance of \texttt{PG} is negligible due to the loose upper-bound of $L$, resulting in a conservative step-size. In~\cref{appendix:experiments}, we plot the wall-clock time to justify the performance gains of the proposed methods.

In the next section, we study the more realistic stochastic setting where the rewards and transition probabilities are unknown, and the policy gradients need to be estimated via interactions with the environment. Although GNPG and NPG can obtain faster convergence rates in the exact setting, they are not guaranteed to converge to the optimal policy in the stochastic setting~\citep{mei2021understanding}. This is because these methods are too aggressive and can quickly commit to sub-optimal actions. Consequently, we restrict ourselves to softmax PG in the stochastic setting.

\section{Policy Gradient in the Stochastic Setting}\label{section:spg_wo_entropy}
In this section, we analyze softmax PG with an estimated (stochastic) policy gradient. In~\cref{sec:sspg}, we construct PG estimators that are unbiased and have bounded variance. We design a PG algorithm that uses the stochastic policy gradient along with exponentially decreasing step-sizes~\citep{li2021second,vaswani2022towards}. In~\cref{sec:sspg-convergence}, we prove that the resulting algorithm can obtain convergence rates comparable to the state-of-the-art, but do not require oracle-like knowledge of the environment. Finally, in~\cref{sec:sspg-faster}, we exploit the fact that the variance in the stochastic gradients decreases as the algorithm approaches a stationary point, and prove that the same stochastic softmax PG algorithm can obtain a faster convergence rate.  

\subsection{Stochastic Softmax Policy Gradient}
\label{sec:sspg}
For illustrative purposes, we mainly focus on the bandit setting in the main paper. In the stochastic multi-armed bandit setting~\citep{lattimore2020bandit}, each action (arm) has an underlying unknown reward distribution. In every iteration $t$, the algorithm chooses an action to pull and receives a stochastic reward sampled from the distribution of the corresponding arm. The stochastic softmax PG algorithm maintains a distribution $\pit \in \Delta_{\gA}$ over the actions. In each iteration $t \in [1, T]$, the algorithm samples an action $a_t \sim \pit$ and receives reward $R_t \sim P_{a_t}$ where $P_{a_t}$ is the reward distribution of arm $a_t$. The reward $R_t$ is used to construct the on-policy importance sampling (IS) reward estimate $\hat{r}_t(a) = \frac{\indicator{a_t = a}}{\pit(a)} \, R_t$ for each $a \in \gA$. The IS reward estimate is then used to form the stochastic gradient $\hgrad{\tht}$ such that $\hgrad{\tht}(a) = \pit(a) [\hat{r}_t(a) - \langle \pit, \hat{r}_t \rangle]$. \citet[Lemma 5]{mei2021understanding} showed that the resulting stochastic gradients are (i) unbiased i.e.  
$\E[\hgrad{\theta}] = \grad{\theta}$ and have (ii) bounded variance i.e. $\E\normsq{\hgrad{\theta} - \grad{\theta}} \leq \sigma^2$. Similarly, we can construct gradient estimators that are unbiased and have bounded variance for general MDPs (refer to \cref{appendix:proof_sgc_mdp}). Given these estimators, the resulting stochastic softmax PG algorithm has the following update: 
\vspace{0.6ex}
\begin{update}(Stochastic Softmax PG, Importance Sampling)\label{update:spg} 
$\thetatt = \thetat + \etat \hgrad{\thetat}$.
\end{update}
We note that this update has also been used in~\citet{yuan2022general,mei2021understanding} that attain global convergence to the optimal solution in both the bandit and general MDP settings. In order to prove theoretical convergence,~\citet{yuan2022general} used the knowledge of $\mu := \inf_{t \geq 1} [C(\thetat)]^2$ when setting the step-size. However, in both the bandit and general MDP settings (see \cref{table:c_theta} for details) $C(\theta)$ and consequently $\mu$ depends on the optimal action. This makes the resulting algorithm impractical. On the other hand,~\citet{mei2021understanding} require the full gradient to set the step-size and obtain global convergence. Since the full gradient is not available in the stochastic setting, it is not practical to use their algorithm. \cref{table:spg_compare} summarizes the global convergence rates for stochastic softmax PG and the method's step-size dependencies. 
\begin{table}[t]
\centering
\begin{tabular}{|c|c|c|}
\hline
    &  Convergence Rate  &  Knowledge required to set $\eta$ \\ \hline
\citet{mei2021understanding} & $\gO(\nicefrac{1}{\eps^2})$ & $\norm{\grad{\theta}}$ \\ \hline
\citet{yuan2022general} & $\gO(\nicefrac{1}{\eps^3})$ & $\pi^*$ \\ \hline
\citet{mei2023stochastic} & $\gO(\nicefrac{1}{\eps})$ & mean reward vector $r$ \\ \hline
\textbf{This work} & Interpolates between $\tilde{\gO}(\nicefrac{1}{\eps})$ $\&$ $   \tilde{\gO}(\nicefrac{1}{\eps^3})$  & $T$ \\ \hline 
\end{tabular}
\caption{Global convergence rates and knowledge required to set the step-size $\eta$ for each method in the bandits setting. Our proposed method achieves comparable convergence rates to prior state-of-the-art results without any oracle-like knowledge.}
\label{table:spg_compare}
\end{table}

We make use of exponentially decaying step-sizes~\citep{li2021second,vaswani2022towards} that have been previously used for stochastic gradient descent when minimizing smooth non-convex functions satisfying the P\L-inequality~\citep{POLYAK1963864,karimi2016linear}. In this setting, the benefit of exponentially decaying step-sizes is that they can achieve (up to poly-logarithmic terms) the best known convergence rates without the knowledge of $\sigma^2$ or $\mu$. Given the knowledge of $T$, the step-size in iteration $t$ is set as: $\eta_t = \eta_0 \, \alpha^t$  where $\eta_0$ is the initial step-size, $\alpha = \left(\frac{\beta}{T}\right)^\frac{1}{T}$ and $\beta \geq 1$. Although $\beta$ is a hyper-parameter, we emphasize that it does not depend on any problem-dependent constants. We leverage these step-sizes for designing a stochastic softmax PG algorithm and characterize its convergence in the next section.  
 
\subsection{Theoretical Convergence}
\label{sec:sspg-convergence}
By using the proof techniques from~\citet{yuan2022general} and~\citet{li2021second}, we prove the following theorem in~\cref{appendix:proof_spg_ess}. 
\vspace{1ex}
\begin{restatable}{theorem}{thereomspgess}\label{theorem:spg_ess}
For a given $\eps \in (0, 1)$, assuming $f$ is (i) $L$-smooth, (ii) satisfies the non-uniform \L ojasiewciz condition with $\xi = 0$, (iii) $\mu := \left[\E\left[\inf_{t \geq 1} [C(\thetat)]^{-2} \right]\right]^{-1} > 0$, using~\cref{update:spg} with (a) unbiased stochastic gradients whose variance is bounded by $\sigma^2$ and (b) exponentially decreasing step-sizes ${\eta_t = \eta_0 \, \alpha^t}$  where $\eta_0 = \frac{1}{L}$ and $\alpha = \left(\frac{\beta}{T}\right)^\frac{1}{T}$, $\beta \geq 1$ results in the following convergence: \\ If $\E[f^* - f(\thetat)] > \eps$ for all $t \in [1, T]$, then, 
\begin{align}
 \E[f^* - f(\theta_{T+1})] &\leq \E[f^* - f(\theta_1)] \, C_1 \, \exp \parens*{-\frac{\alpha \, \eps \,  T}{\kappa \, \ln(\nicefrac{T}{\beta})}}  + \frac{C_1 \, C_2}{2 \, L}  \frac{\ln^2\parens*{\frac{T}{\beta}} \, \sigma^2}{\eps^2 \, T} 
\end{align}
where $\kappa := \frac{2 \, L}{\mu}$, $C_1 := \exp \parens*{\frac{2 \, \beta}{\kappa \, \ln(\nicefrac{T}{\beta})}}$ and $C_2 := \frac{4 \, \kappa^2}{e^2 \, \alpha^2}$.
Otherwise $\min_{t \in [1, T]} \E[f^* - f(\thetat)] \leq \eps$.
\end{restatable}
\vspace{-2ex}
In order to ensure that assumption (iii) holds, let us consider the bandit setting where $C(\theta) = \pitheta(a^*)$. To guarantee that $\mu := \left[\E\left[\inf_{t \geq 1} [C(\thetat)]^{-2}\right]\right]^{-1} > 0$, we must ensure that $\pi_{\theta_0}(a^*) > 0$. Since $T$ is finite and $\theta_0$, $\etat$ and the stochastic gradients are bounded (refer to \cref{lemma:bandit_sg_bounded,lemma:mdp_sg_bounded} in \cref{appendix:spg_proofs}), no parameter including $\theta(a^*)$ can diverge to $-\infty$, guaranteeing that $\pitheta(a^*) > 0$. 

To determine the resulting convergence rate, let us first analyze the case when $\sigma^2 = 0$. In this case, given a target $\eps$, we set $T = \gO(\nicefrac{1}{\eps}\, \log(\nicefrac{1}{\eps}))$ iterations to make the first term $\gO(\eps)$. On the other hand, when $\sigma^2 > 0$ and the second term of $\tilde{\gO}\parens*{\nicefrac{\sigma^2}{\eps^2 T}}$ dominates, we set $T = \tilde{\gO}(\nicefrac{1}{\eps^3})$ iterations to make the second term $\gO(\eps)$. Putting both cases together, in order to make the sub-optimality $\gO(\eps)$, we can set $T = \max\{\tilde{\gO}\parens*{\nicefrac{1}{\eps}, \nicefrac{\sigma^2}{\eps^{3}}}\}$. This convergence rate matches that in~\citet{yuan2022general} without requiring the knowledge of $\mu$. We emphasize that the above convergence rate holds without the knowledge of any oracle-like information. 

The previous result assumes that the variance $\sigma^2$ is constant w.r.t. $\theta$. However, it has been observed that the noise depends on $\theta$, and decreases as the algorithm gets closer to a stationary point since the policy become more deterministic. Next, we leverage this property to prove faster rates.

\subsubsection{Faster Rates}
\label{sec:sspg-faster}
In the bandit setting,~\citet{mei2023stochastic} formalized the above intuition, and proved that the stochastic gradient $\hgrad{\theta}$ satisfies the strong growth condition (SGC)~\citep{schmidt2013fast,vaswani2019fast} implying that 
$\E \normsq{\hgrad{\theta}}  \leq \varrho \, \norm{\grad{\theta}}$ for a problem-dependent $\varrho > 1$. This implies that the variance decreases as the algorithm approaches a stationary point and $\norm{\nabla f(\theta)} \rightarrow 0$. For the bandit setting, using~\cref{update:spg} and the knowledge of $\varrho$ to set the step-size,~\citet{mei2023stochastic} can attain a faster $\gO(\nicefrac{1}{\eps})$ convergence rate. We generalize the above SGC result to the general MDP setting in~\cref{theorem:mdp_sgc} (proved in \cref{appendix:proof_sgc_mdp}). 
\vspace{1ex}
\begin{restatable}{theorem}{theoremsgc}\label{theorem:sgc}
Using \cref{update:spg}, we have for all $\theta$, $\E\normsq{\hgrad{\theta}} \leq \varrho \, \norm{\grad{\theta}}_2$, 
where $\varrho := \frac{8  \, A^{3/2}}{\Delta^2}$ in the bandit setting with $\Delta := \min_{a \neq a'}\abs{r(a) - r(a')}$
and  $\varrho =  \frac{4 \, A^{3/2} \, S^{1/2}}{(1 - \gamma)^4 \, \Delta^2}$ in the tabular MDP setting with $\Delta := \min_s \min_{a \neq a'} \abs{Q^{\pitheta}(s, a) - Q^{\pitheta}(s, a')}$.
\end{restatable}
\vspace{-1ex}
However, in the bandit setting, $\varrho$ depends on the unknown \emph{reward gap} $\Delta := \min_{a \neq a'} \abs{r(a) - r(a')}$ and we prove that this dependence is necessary (\cref{proposition:spg_rho} in \cref{appendix:spg_proofs}). This makes the resulting algorithm ineffective in most practical cases. Hence, we aim to develop a practical algorithm that can automatically adapt to $\varrho$ and result in a faster convergence. In \cref{theorem:spg_ess_sgc}, proved in~\cref{appendix:proof_spg_ess_sgc}, we show that the same stochastic softmax PG algorithm (with exponentially decreasing step-sizes) can attain such fast convergence. In addition to the properties in~\cref{theorem:spg_ess}, we exploit the function's non-uniform smoothness, the SGC and the boundedness of stochastic gradients to prove this result. 
\vspace{1ex}
\begin{restatable}{theorem}{theoremspgesssgc}\label{theorem:spg_ess_sgc}
For a given $\eps \in (0, 1)$, assuming $f$ is (i) $L_1$ non-uniform smooth, (ii) satisfies the non-uniform \L ojasiewciz condition with $\xi = 0$, (iii) $\mu :=  \left[\E\left[\inf_{t \geq 1} [C(\thetat)]^{-2}\right]\right]^{-1} > 0$, using~\cref{update:spg} with unbiased stochastic gradients that are (a) bounded, i.e. $\norm{\hgrad{\theta}} \leq B$ and satisfy the strong growth condition with $\varrho$ and (b) exponentially decreasing step-sizes $\etat = \eta_0 \, \alpha^t$ where $\eta_0 < \frac{1}{L_1^2 B}$ and $\alpha = \parens*{\frac{\beta}{T}}^{\frac{1}{T}}$, $\beta \geq 1$, results in the following convergence: \\ If $\E[f^* - f(\thetat)] > \eps$ for all $t \in [1, T]$, then,
\begin{align}
\E[f^* - f(\theta_{T+1})] &\leq \E[f^* - f(\theta_1)] \, C_1 \, \exp\left(-\frac{\alpha \, \eps \, T}{\kappa \, \ln(T)} \right) +  \frac{C_2 \, \sum_{t=1}^{T_0 - 1} \, \E[f^* - f(\thetat)]}{\eps^2 \, T^2}  
\end{align}
where $\kappa := \frac{2}{\mu \, \eta_0}$, $C_1 := \exp\parens*{\frac{2 \, \beta}{\kappa \, \ln(\nicefrac{T}{\beta})}}$, $C_2 := \exp\parens*{\frac{2 \, \beta}{\kappa \, \ln(\nicefrac{T}{\beta})}} \frac{16 \, \varrho \, L \, \kappa^2}{e^2 \, \alpha^2} \ln^2(\nicefrac{T}{\beta})$, ${T_0 := T\, \max \left\{\frac{\ln(\varrho \, \eta_0)}{\ln(\nicefrac{T}{\beta})}, 0 \right\}}$. Otherwise ${\min_{t \in [1, T]} \E[f^* - f(\thetat)] \leq \eps}$.
\end{restatable}
Similar to~\cref{theorem:spg_ess}, assumption (iii) is true when $\pi_{\theta_0}(a^*) > 0$ and $T$ is finite. In~\cref{lemma:bandit_sg_bounded,lemma:mdp_sg_bounded} (proved in \cref{appendix:spg_proofs}), we prove that the stochastic gradients are bounded in both the bandit and MDP settings. In the above result, $T_0$ represents the iteration when the step-size is small enough to take advantage of the SGC. Given the knowledge of $\varrho$, we can set set $\eta_0 \leq \nicefrac{1}{\varrho}$ in which case $T_0 = 0$. In this case, setting $T = \tilde{\gO}(\nicefrac{1}{\eps})$ iterations enables us to obtain a ``fast'' $\gO(\nicefrac{1}{\eps})$ rate. Since $\varrho$ is unknown in general, setting $\eta_0$ to be large can result in $T_0 = \gO(T)$ in the worst case. In this case, the second term of order $\tilde{\gO}(\nicefrac{1}{\eps^2 T})$ dominates. In this case, setting $T = \gO(\nicefrac{1}{\eps^3})$ iterations results in a ``slow'' $\tilde{\gO}\parens*{\nicefrac{1}{\eps^3}}$ rate. Hence, the resulting algorithm is robust to $\varrho$ and depending on how $\eta_0$ is set, it can interpolate between the ``slow'' and ``fast'' rates.

Below, we instantiate~\cref{theorem:spg_ess_sgc} in the bandit setting. 

\clearpage 
\begin{restatable}{corollary}{corollarybanditssgcsgp}\label{corollary:spg_ess_sgc}
In the bandit setting, for a given $\eps \in (0, 1)$, using~\cref{update:spg} with exponentially decreasing step-sizes $\eta_t = \eta_0 \, \alpha^t$ where $\eta_0 \leq \frac{1}{18}$, $\alpha = \left(\frac{\beta}{T}\right)^\frac{1}{T}$, $\beta \geq 1$  results in the following convergence: \\ If $\E[(\pistar - \pi_{\thetat})^\top r] \geq \eps$ for all $t \in [1, T]$, then,
\begin{equation}
 \E[(\pistar - \pi_{\theta_{T+1}})^\top r] \leq 
 \E[(\pistar - \pi_{\theta_{1}})^\top r]
 \, C_1 \, \exp \parens*{-\frac{\alpha \, \eps \,  T}{\kappa \, \ln(\nicefrac{T}{\beta})}} + \frac{C_2 \, \sum_{t=1}^{T_0 - 1} \E[(\pistar - \pi_{\thetat})^\top r]}{\eps^2 \, T^2} 
\end{equation}
where $\kappa := \frac{2}{\mu \, \eta_0}$, $C_1 := \exp\parens*{\frac{2 \, \beta}{\kappa \, \ln(\nicefrac{T}{\beta})}}$, $C_2 := \exp\parens*{\frac{2\beta}{\kappa \, \ln(\nicefrac{T}{\beta})}} \, \frac{32 \, \varrho \, \kappa^2}{5 \, e^2 \, \alpha^2} \, \ln^2(\nicefrac{T}{\beta})$, $T_0 := T \, \max \left\{\frac{\ln(4 \,  \varrho \, \eta_0)}{\ln(\nicefrac{T}{\beta})}, 0 \right\}$, $\rho = \frac{8 \, A^{3/2}}{\Delta^2}$ and $\mu := \left[\E\left[\min_{t \in [1, T]} [\pit(a^*)]^{-2}\right]\right]^{-1}$ > 0. Otherwise ${\min_{t \in [1, T]} \E[(\pistar - \pit)^\top r] \leq \eps}$.
\end{restatable} 

In the multi-armed bandit setting, using stochastic softmax PG with exponentially decreasing step-sizes allows for implicit automatic exploration without requiring the knowledge of any problem-dependent constants such as the reward gap. Unlike~\citet{mei2023stochastic}, we note that the above result does not imply asymptotic convergence to the optimal arm. This difference stems from the fact that~\citet{mei2023stochastic} uses a constant step-size, while the above result requires a decreasing step-size that asymptotically goes to zero. Compared to the standard algorithms for multi-armed bandits such as upper confidence bound (UCB) \citep{auer2002finite} which requires the knowledge of the noise magnitude to design confidence intervals or Thompson sampling (TS) \citep{agrawal2012analysis} which requires knowledge of the reward distribution, stochastic softmax PG does not require such information.

In the next section, we empirically validate our theoretical results and compare the proposed methods to prior algorithms in the bandits setting.

\section{Experimental Evaluation \protect\footnote{The code to reproduce results is available \href{https://github.com/sudo-michael/practical-pg}{here}}}\label{section:exp}
\begin{figure}[ht]
  \begin{center}
    \includegraphics[scale=0.5]{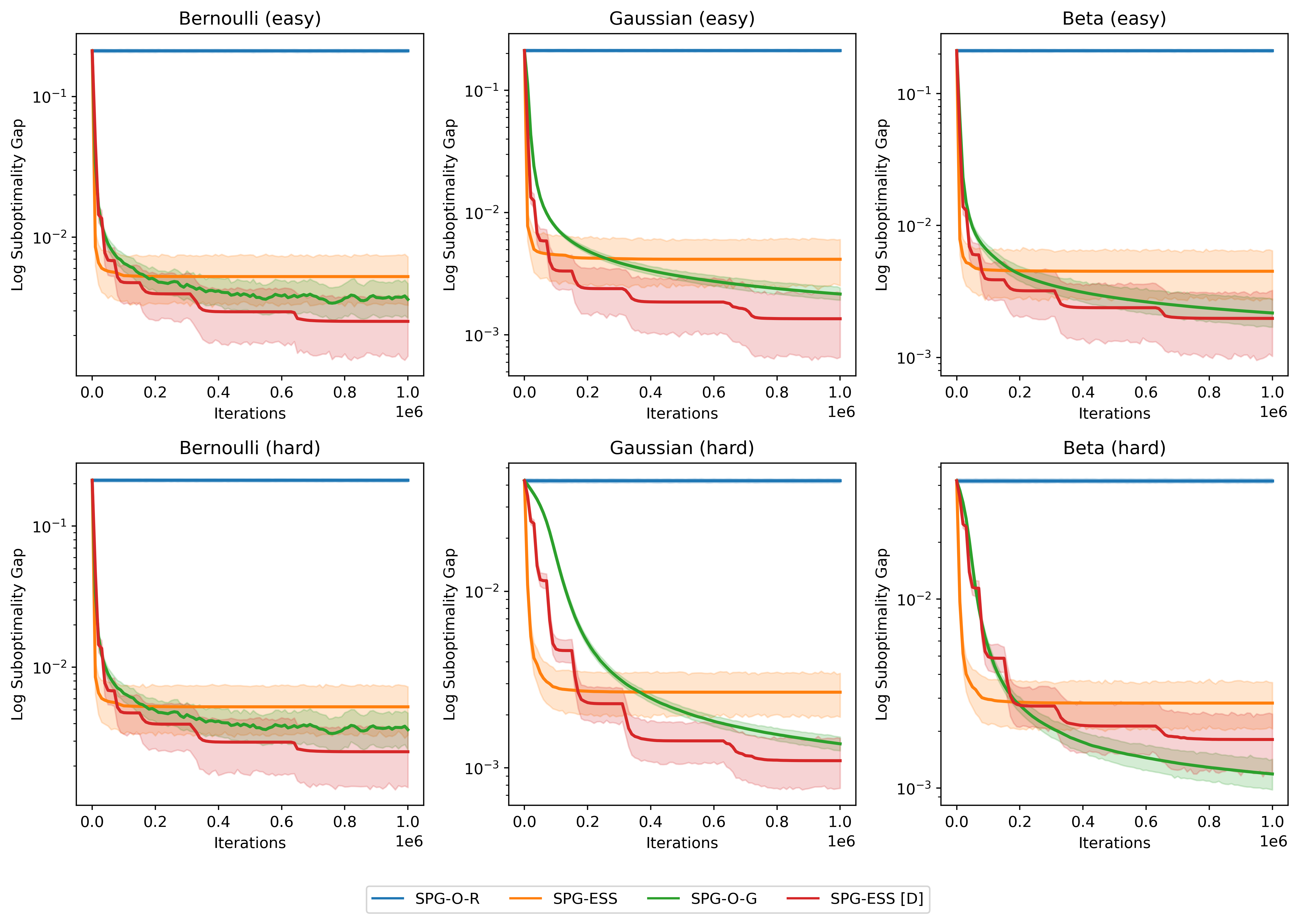}
    \caption{Expected sub-optimality gap across various environments. \texttt{SPG-ESS} and \texttt{SPG-ESS [D]} is comparable to \texttt{SPG-O-G} and \texttt{SPG-O-R} without using any oracle-like knowledge of the environment.}
    \label{fig:spg}
  \end{center}
\end{figure}

We evaluate the methods in multi-armed bandit environments with $A = 10$. For each environment, we compare the various algorithms on the basis of their expected sub-optimality gap $\E[(\pistar - \pit)^\top r]$. For each instance of an environment, we run an algorithm $5$ times to account for the stochasicity of each algorithm. We plot the average and 95\% confidence interval of the expected sub-optimality gap across $25$ instances over $T=10^6$ iterations. For each run, the initial policy is uniform, i.e. $\pi_{\theta_0}(a) = \nicefrac{1}{A}$ for all $a \in \gA$.

\textbf{Environment Details:} 
Each environment's underlying reward distribution is either a Bernoulli, Gaussian, or Beta distribution with a fixed  mean reward vector $r \in \R^A$ and support $[0, 1]$. The difficulty of the environment is determined by the maximum reward gap $\bar{\Delta} := \min_{a^* \neq a} r(a^*) - r(a)$. In easy environments $\bar{\Delta} = 0.5$ and in the hard environments $\bar{\Delta} = 0.1$. For each environment, $r$ is randomly generated for each run. 

\textbf{Methods:} We compare stochastic softmax PG with exponentially decreasing step-size (\texttt{SPG-ESS}) to prior work that uses the full gradient (\texttt{SPG-O-G}) \citep{mei2021understanding} and the reward gap (\texttt{SPG-O-R}) \citep{mei2023stochastic} when setting the step-size. For \texttt{SPG-ESS}, we select $\beta = 1$ and $\eta_0 = \frac{1}{18}$ for all experiments. For \texttt{SPG-O-R} and \texttt{SPG-O-G}, we use the corresponding theoretical step-size 
of $\etat = \frac{\Delta^2}{(40) \, 10^{\nicefrac{3}{2}}}$ and $\etat = \frac{1}{12} \, \left\|\frac{d \dpd{\pi_{\theta_t}, r}}{d \thetat}\right\|$ respectively. We emphasize that both these step-sizes depend on the unknown mean reward vector, making the resulting methods impractical. 

In our experiments, we observed that \texttt{SPG-ESS} slows down and stops making progress because of overly conservative step-sizes. To counteract this, we additionally try a ``doubling trick'' (\texttt{SPG-ESS [D]}). This is a common trick when adapting algorithms that depend on a fixed number of iterations \citep{auer1995gambling, hazan2014beyond}. For this ``doubling trick'', we first start with a smaller time horizon $\cT_0 << T$ when setting the step-size, i.e. for $t \leq \cT_0$, $\eta_t = \eta_0 \, \left(\frac{\beta}{\cT_0}\right)^{\frac{t}{\cT_0}}$. After $\cT_0$ iterations, we restart the step-size schedule, double the length of the next time horizon i.e. $\cT_1 = 2 \, \cT_0$ and set $\eta_t$ with the time horizon equal to $\cT_1$. This process repeats until the desired number of iterations is reached. For \texttt{SPG-ESS [D]} we select $\beta=1$, $\eta_0 = \frac{1}{18}$ and $\cT_0 = 5000$ for all environments.

\textbf{Results:} From \cref{fig:spg}, we conclude that \texttt{SPG-ESS} and \texttt{SPG-ESS [D]} are consistently comparable to \texttt{SPG-O-G} and \texttt{SPG-O-R} without access to any oracle-like knowledge. While \texttt{SPG-O-R} has the best theoretical convergence rate, its step-size is proportional to the reward gap. When the reward gap is small, so is the resulting step-size which results in its poor empirical performance.

\section{Discussion}
\label{sec:discussion}
We designed (stochastic) softmax policy gradient (PG) methods for bandits and tabular Markov decision processes (MDPs). Throughout, we demonstrated that the proposed methods offer similar theoretical guarantees as the state-of-the art results, but do not require the knowledge of oracle-like quantities. Concretely, in the exact setting, we empirically demonstrated that using softmax PG with Armijo line-search to set the step-size is competitive to GNPG without requiring knowledge of the concentrability coefficient to set the step-size. In the stochastic setting, we used exponentially decreasing step-sizes and showed that the resulting algorithm is robust to problem-dependent constants and can interpolate between slow and fast rates. For future work, we hope to analyze the convergence rate when using the ``doubling trick'' with exponentially decreasing step-sizes. Finally, we aim to generalize our results to support complex (non)-linear policy parameterization.

\bibliographystyle{rlc}

\newpage
\appendix
\newcommand{\appendixTitle}{%
\vbox{
    \centering
	\hrule height 4pt
	\vskip 0.2in
	{\LARGE \bf Supplementary Material}
	\vskip 0.2in
	\hrule height 1pt 
}}
\appendixTitle
\section*{Organization of the Appendix}\label{appendix:org}
\begin{itemize}
   \item[\ref{appendix:defn}] \hyperref[appendix:defn]{Definitions}
   \item[\ref{appendix:pg_proofs}] \hyperref[appendix:pg_proofs]{Proofs of \cref{section:pg_wo_entropy}}
   \item[\ref{appendix:spg_proofs}] \hyperref[appendix:spg_proofs]{Proofs of \cref{section:spg_wo_entropy}} 
  \item[\ref{appendix:entropy}] \hyperref[appendix:entropy]{Policy Gradient with Entropy Regularization}  %
   \begin{itemize}
       \item[\ref{appendix:pg_entropy}] \hyperref[appendix:pg_entropy]{Policy Gradient with Entropy Regularization in the Exact Setting} 
       \item[\ref{appendix:spg_entropy}] \hyperref[appendix:pg_entropy]{Policy Gradient with Entropy Regularization in the Stochastic Setting} 
   \end{itemize}
   \item[\ref{appendix:pg_entropy_proofs}] \hyperref[appendix:pg_entropy_proofs]{Proofs of \cref{appendix:pg_entropy}} 
   \item[\ref{appendix:spg_entropy_proofs}] \hyperref[appendix:spg_entropy_proofs]{Proofs of \cref{appendix:spg_entropy}} 
   \item[\ref{appendix:experiments}] \hyperref[appendix:experiments]{Additional Experiments} 
   \item[\ref{appendix:extra_lemmas}] \hyperref[appendix:extra_lemmas]{Extra Lemmas} 
\end{itemize}
\section{Definitions}
\label{appendix:defn}
A function $f$ is $L$-smooth if for all $\theta$ and $\theta'$
\begin{equation}
    \abs{f(\theta) - f(\theta') - \dpd{\grad{\theta'}, \theta - \theta'}} \leq \frac{L}{2} \normsq{\theta - \theta'}.
\end{equation}

A function $f$ is $L_1$-non-uniform smooth if for all $\theta$ and $\theta'$
\begin{equation}
    \abs{f(\theta) - f(\theta') - \dpd{\grad{\theta'}, \theta - \theta'}} \leq \frac{L_1 \norm{\grad{\theta'}}}{2} \normsq{\theta - \theta'}.
\end{equation}

A function $f$ satisfies the non-uniform \L ojasiewciz condition of degree $\xi$ for $\xi \in [0, 1]$ is defined as
\begin{equation}
     \norm{\grad{\theta}}\geq C(\theta) \abs{f^* - f(\theta)}^{1 - \xi} \tag{$f^* := \sup_{\theta} f(\theta)$}
\end{equation}
where $C: \theta  \rightarrow \R > 0$.

A function $f$ satisfies the reversed \L ojasiewciz condition if for all $\theta$
\begin{equation}
    \norm{\grad{\theta}} \leq \nu \, [f^* - f(\theta)]
\end{equation}
where $\nu > 0$.

\clearpage

\section{Proofs in \cref{section:pg_wo_entropy}}\label{appendix:pg_proofs}
\subsection{Proof Of Theorem \ref{theorem:pg_line_search}}\label{appendix:proof_pg_line_search}
\theorempgls*
\begin{proof}
From \cref{eq:armijo}, Armijo line-search selects a step-size that satisfies the following condition where $h \in (0, 1)$ is a hyper-parameter
\begin{equation}
    f(\thetat + \etat \, \grad{\thetat}) \geq f(\thetat) + h \, \etat \, \normsq{\grad{\thetat}}. 
\end{equation}
\begin{align}
       \intertext{For any $L$-smooth function the step-size $\etat$ returned by the Armijo line-search is guaranteed to satisfy $\eta_{\mathrm{max}} \geq \etat \geq \min \braces*{\frac{2 \,(1-h)}{L}, \eta_{\mathrm{max}}}$ \citep{armijo1966} which implies that}
        f(\thetatt) &\geq f(\thetat) + \min\braces*{\frac{2 \, h \, (1-h)}{L}, h \, \eta_{\mathrm{max}}} \, \normsq{\grad{\thetat}}  \\
       \intertext{Adding $f^*$ to both sides and multiplying by $-1$}
        f^* - f(\thetatt) &\leq f^* -  f(\thetat) - \min\braces*{\frac{2\,h\,(1-h)}{L}, h \, \eta_{\mathrm{max}}} \, \normsq{\grad{\thetat}}  \\
        \intertext{Let $\delta(\thetat) := f^* - f(\thetat)$}
        \delta(\thetatt) &\leq \delta(\thetat) - \min\braces*{\frac{2\,h\,(1-h)}{L}, h \, \eta_{\mathrm{max}}} \, \normsq{\grad{\thetat}}  \\
       \intertext{Since $f$ satisfies the non-uniform \L ojasiewciz condition with $\xi = 0$}
        &\leq \delta(\thetat) - \min\braces*{\frac{2\,h\,(1-h)}{L}, h \, \eta_{\mathrm{max}}} \, [C(\thetat)]^2 \, [\delta(\thetat)]^2 \\
        \intertext{Assuming $\mu := \inf_{t \geq 1} [C(\thetat)]^2 > 0$}
        &\leq \delta(\thetat) - \underbrace{\mu \, \min\braces*{\frac{2\,h(1-h)}{L}, h \, \eta_{\mathrm{max}}}}_{:= \frac{1}{C}} \, [\delta(\thetat)]^2 \label{eq:theorem_1_descent_amount}\\
        \intertext{Dividing by $\delta(\thetat) \, \delta(\thetatt)$}
        \implies \frac{1}{\delta(\thetat)} &\leq \frac{1}{\delta(\thetatt)} - \frac{1}{C} \, \frac{\delta(\thetat)}{\delta(\thetatt)} \label{eq:linesearch_recurse}
        \intertext{Using \cref{eq:linesearch_recurse} and recursing from $t=1$ to $T$ }
        \frac{1}{\delta(\theta_1)} &\leq \frac{1}{\delta(\theta_{T+1})} - \frac{1}{C} \sum_{t=1}^{T} \frac{\delta(\thetat)}{\delta(\thetatt)} \\
        &\leq \frac{1}{\delta(\theta_{T+1})} - \frac{T}{C} \tag{$\frac{\delta(\thetat)}{\delta(\thetatt)} \geq 1$} \\
        \implies \frac{T}{C} &\leq \frac{1}{\delta(\theta_{T+1})}.
\end{align}
Therefore
\begin{equation}
    f^* - f(\theta_{T+1}) \leq \max\braces*{\frac{L}{2\,h\,(1-h)} \, \frac{1}{h \, \eta_{\max}}} \, \frac{1}{\mu \, }.
\end{equation}
\end{proof}

\begin{corollary} In the bandit setting, using \cref{update:dpg} with Armijo line-search to set the step-size results in the following convergence:
\begin{equation}
(\pi^* - \pi_{\theta_{T+1}})^\top r \leq \max\braces*{\frac{5}{4 \, h \, (1 - h)}, \frac{1}{h \, \eta_{\max}}} \, \frac{1}{\mu \, T}
\end{equation}
where $h \in (0, 1)$, $\eta_{\max}$ is the upper-bound on the step-size, and $\mu := \inf_{t \geq 1} [\pit(a^*)]^2 > 0$. 
\end{corollary}
\begin{proof}
We can extend \cref{theorem:pg_line_search} to the bandit setting since: 
\begin{itemize}
    [nolistsep]
    \item by \cref{lemma:lemma_2_mei_global}, $f$ is $\frac{5}{2}$-smooth
    \item by \cref{lemma:lemma_3_mei_global}, $f$ is non-uniform \L ojsiewciz with $\xi = 0$ and $C(\theta) = \pitheta(a^*)$
    \item we observe that \citep[Lemma 5]{mei2020global} works for any step-size sequence guaranteeing monotonic improvement $\mu := \inf_{t \geq 1}[C(\thetat)]^2 > 0$
\end{itemize}
\end{proof}

\begin{corollary} 
Assuming $\min_{s \in \gS} \rho(s) > 0$, in the tabular MDP setting, using \cref{update:dpg} with Armijo line-search to set the step-size results in the following convergence:
\begin{equation}
V^{*}(\rho) - V^{\pi_{\theta_{T+1}}}(\rho)  \leq \max\braces*{\frac{8}{2\,h \, (1-h) \, (1-\gamma)^3} \, \frac{1}{\eta_{\max} \,h }} \, \frac{1}{\mu \, T}
\end{equation}
where $h \in (0, 1)$, $\eta_{\max}$ is the upper-bound on the step-size, and
$\mu := \inf_{t \geq 1}\parens*{\frac{\min_s \pit(a^*(s) | s)}{\sqrt{S} \, \supnorm{\nicefrac{d^{\pistar}_\rho}{d^{\pit}_\rho}}}}^2 > 0$.
\end{corollary}
\begin{proof}
We can extend \cref{theorem:pg_line_search} to the tabular MDP setting since:
\begin{itemize}
    [nolistsep]
    \item by \cref{lemma:lemma_7_mei_global}, $f$ is $\frac{8}{(1 - \gamma)^3}$-smooth,
    \item  by \cref{lemma:lemma_8_mei_global}, $f$ is non-uniform \L ojsiewciz with $\xi = 0$ and $C(\theta) = \frac{\min_s \pitheta(a^*(s) | s)}{\sqrt{S} \,  \supnorm{\nicefrac{d^{\pistar}_\rho}{d^{\pitheta}_\rho}}}$
    \item we observe that \citep[Lemma 9]{mei2020global} works for any step-size sequence guaranteeing monotonic improvement $\mu := \inf_{t \geq 1}[C(\thetat)]^2 > 0$
\end{itemize}
\end{proof}

\subsection{Proof Of Theorem \ref{theorem:pg_line_search_exp_new}}\label{appendix:proof_pg_line_search_exp_new}
\theorempglsexpnew*
\begin{proof}
Since the rewards are bounded, we will overload the notation and let $f^* - f(\tht)$ denote the normalized sub-optimality gap. This implies that $f^* - f(\tht) \leq 1$.
Using backtracking line-search using Armijo condition in \cref{eq:trans_line_search} selects a step-size that satisfies the following condition where $h \in (0, 1)$ is a hyper-parameter:
\begin{align}
\ln(f^* - f(\thetat + \etat \gradf{\thetat})) &\leq \ln(f^* - f(\thetat)) - h \, \etat \, \frac{\normsq{\gradf{\thetat}}}{f^* - f(\tht)}
\intertext{Applying $\exp(\cdot)$ to both sides}
f^* - f(\thetat + \etat \, \gradf{\thetat}) &\leq [f^* - f(\thetat)]\,\exp\parens*{-h \, \etat \, \frac{\normsq{\gradf{\thetat}}}{f^* - f(\tht)}}
\intertext{By \cref{lemma:trans_armijo_step_size_lb}, we can guarantee that the backtracking line-search is guaranteed to satisfy $\etat \geq \min\left\{\eta_{\text{max}}, \frac{2(1-h)}{L_1 \, \nu\, [f^* - f(\tht)]}\right\}$ which implies that}
f^* - f(\thtt) &\leq  [f^* - f(\thetat)] \, \exp\parens*{- \min\left\{\eta_{\max} \, h , \frac{2 \, h \, (1 - h )}{L_1 \, \nu \, [f^* - f(\tht)]}\right\} \, \frac{\normsq{\gradf{\thetat}}}{f^* - f(\thetat)}}  \\
\intertext{Assuming that for a target $\eps \in (0, 1)$, $\eps < f^* - f(\tht)$ for $t \in [1, T]$, selecting $\eta_{\max} = \frac{C}{\eps}$ for $C > 0$ implies $\eta_{\max} > \frac{C}{f^* - f(\tht)}$}
&\leq  [f^* - f(\thetat)] \, \exp\parens*{- \min\left\{C \, h, \frac{2 \, h \, (1 - h )}{L_1 \, \nu}\right\} \, \frac{\normsq{\gradf{\thetat}}}{(f^* - f(\thetat))^2}}  \\
    \intertext{Since $f$ satisfies the non-uniform \L ojasieciz condition with $\xi = 0$}
    & \leq  [f^* - f(\thetat)] \, \exp\parens*{- \min\left\{C \, h, \frac{2\, h \, (1 - h)}{L_1 \, \nu}\right\} \, [C(\thetat)]^2 }  \\
    \intertext{Assuming $\mu := \inf_{t \geq 1}[C(\thetat)]^2 > 0$}
    \implies f^* - f(\thetatt) & \leq  [f^* - f(\thetat)] \, \exp\parens*{- \min\left\{C \, h, \frac{2 \, h  \, (1 - h )}{L_1 \, \nu}\right\} \, \mu}.  \label{eq:trans_armijo_recurse}
   \intertext{Using \cref{eq:trans_armijo_recurse} and recursing from $t=1$ to $T$ we have} 
   f^* - f(\theta_{T+1}) &\leq [f^* - f(\theta_1)] \, \exp\parens*{-\min\left\{C \, h, \frac{2 \, h \, (1 - h )}{L_1 \, \nu}\right\}  \, \mu \, T}.
\end{align}
\end{proof}

\begin{corollary} In the bandit setting, for a given $\eps \in (0, 1)$, using \cref{update:dpg} with backtracking line-search using the Armijo condition in \cref{eq:trans_line_search} and setting $\eta_{\max} = \nicefrac{C}{\eps}$ results in the following convergence: 
 \\ If $(\pistar - \pit)^\top r > \eps$ for all $t \in [1, T]$, then, 
\begin{equation}
(\pi^* - \pi_{\theta_{T+1}})^\top r \leq (\pi^* - \pi_{\theta_1})^\top r \, \exp\parens*{-\min\braces*{C \, h, \frac{2\,h\, (1-h) \, \Delta^*}{3 \, \sqrt{2}}} \, \mu \, T}
\end{equation}
where $C > 0$ and $h \in (0, 1)$ are hyper-parameters, $\Delta^* := r(a^*) - \max_{a \neq a^*} r(a)$, and $\mu := \inf_{t \geq 1} [\pit(a^*)]^2 > 0$. Otherwise $\min_{t \in [1, T]} (\pistar - \pit)^\top r \leq \eps$.
\end{corollary}
\begin{proof}
We can extend \cref{theorem:pg_line_search_exp_new} to the bandit setting since:
\begin{itemize}
[nolistsep]
    \item by \cref{lemma:bandit_ns}, $f$ is $3$-non-uniform smooth
    \item by \cref{lemma:lemma_3_mei_global}, $f$ is non-uniform \L ojsiewciz with $\xi = 0$ and $C(\theta) = \pitheta(a^*)$
    \item by \cref{lemma:reverse_nl_bandit}, $f$ satisfies the reverse \L ojasiewciz condition with $\nu = \frac{\sqrt{2}}{\Delta^*}$
    \item since we observe that Lemma 5 in \citep{mei2020global} works for any step-size sequence guaranteeing monotonic improvement, $\mu := \inf_{t \geq 1}[C(\thetat)]^2 > 0$
\end{itemize}
\end{proof}
\begin{corollary} 
Assuming $\min_{s \in \gS} \rho(s) > 0$, in the tabular MDP setting, for a given $\eps \in (0, 1)$, using \cref{update:dpg} with backtracking line-search using the Armijo condition in \cref{eq:trans_line_search} and setting $\eta_{\max} = \nicefrac{C}{\eps}$ results in the following convergence: \\ If $V^*(\rho) - V^{\pit}(\rho) > \eps$ for all $t \in [1, T]$, then,
\begin{equation}
V^{*}(\rho) - V^{\pi_{\theta_{T+1}}}(\rho) \leq \bracks*{V^{*}(\rho) - V^{\pi_{\theta_1}}(\rho)} \, \exp\parens*{-\min\braces*{C\, h,  \frac{2 \, h \, (1 - h) \, (1 - \gamma) \, \Delta^*}{D \, \sqrt{2}}} \, \mu \ T} 
\end{equation}
where $C > 0$ and $h \in (0, 1)$ are hyper-parameters, 
$D := \bracks*{3 + \frac{2 \, C_\infty - (1 - \gamma)}{(1 -\gamma) \, \gamma}} \, \sqrt{S}$, $C_\infty := \max_\pi \supnorm{\frac{d^\pi_\rho}{\rho}} \leq \frac{1}{\min_s \rho(s)} < \infty$, $\Delta^* := \min_{s \in \gS} \braces*{Q^*(s, a^*(s)) - \max_{a(s) \neq a^*(s)} Q^{*}(s, a)}$, and $\mu := \inf_{t \geq 1}\parens*{\frac{\min_s \pit(a^*(s) | s)}{\sqrt{S} \supnorm{\nicefrac{d^{\pistar}_\rho}{d^{\pit}_\rho}}}}^2 > 0$. Otherwise $\min_{t \in [1, T]} V^*(\rho) - V^{\pit}(\rho) \leq \eps$.
\end{corollary}
\begin{proof}
We can extend \cref{theorem:pg_line_search_exp_new} to the tabular MDP setting since:
\begin{itemize}
[nolistsep]
    \item by \cref{lemma:mdp_ns}, $f$ is $D$-non-uniform smooth where $D := \bracks*{3 + \frac{2 \, C_\infty - (1 - \gamma)}{(1 -\gamma) \, \gamma}} \, \sqrt{S}$ and $C_\infty := \max_\pi \supnorm{\frac{d^\pi_\rho}{\rho}} \leq \frac{1}{\min_s \rho(s)} < \infty$
    \item by \cref{lemma:lemma_8_mei_global}, $f$ is non-uniform \L ojsiewciz with $\xi = 0$ and $C(\theta) = \frac{\min_s \pitheta(a^*(s) | s)}{\sqrt{S} \supnorm{\nicefrac{d^{\pistar}_\rho}{d^{\pitheta}_\rho}}}$
    \item by \cref{lemma:reverse_nl_mdp} $f$ satisfies the reverse \L ojsiewciz condition with $\nu = \frac{\sqrt{2}}{(1 - \gamma) \, \Delta^*}$ and $\Delta^* := \min_{s \in \gS} \braces*{Q^*(s, a^*(s)) - \max_{a(s) \neq a^*(s)} Q^{*}(s, a)}$
    \item since we observe that Lemma 9 in \citep{mei2020global} works for any step-size sequence guaranteeing monotonic improvement $\mu := \inf_{t \geq 1}[C(\thetat)]^2 > 0$
\end{itemize}
\end{proof}

\subsection{Additional Lemmas}
\vspace{0.4ex}
\begin{thmbox}
\begin{lemma}\label{lemma:armijo_like_smoothness}
Suppose that (i) $f$ is $L_1$-non-uniform smooth and (ii) satisfies a reversed \L ojasiewciz inequality then 
$\theta \rightarrow \ln(f^* - f(\theta))$ is $L_1 \, \nu$-smooth.
\end{lemma}
\end{thmbox}
\begin{proof}
Let $g(\theta) := \ln(f^* - f(\theta))$.  By Taylor's theorem it suffices to show that the Hessian is bounded by $L_1 \, \nu$
\begin{align}
     \nabla^2 g(\theta) &= \frac{- \nabla^2 f(\theta) \, (f^* - f(\theta)) - [\nabla f(\theta)] \, [\gradf{\theta}]^\top}{(f^* - f(\theta))^2} \\
     \intertext{Since for any $x \in \R^{SA}$ $x \, x^\top \succeq 0$}
    &\preceq \frac{\nabla^2 f(\theta)}{f^* - f(\theta)} \\
    \intertext{Since $f$ is $L_1$-non-uniform smooth,}
    &\preceq \frac{L_1 \norm{\gradf{\theta}}}{f^* - f(\theta)} \\
    \intertext{Since $f$ satisfies the reverse \L ojsaiewciz inequality}
    &\preceq L_1 \, \nu \, I_{SA}. 
    \end{align}
\end{proof}
\begin{thmbox}
\begin{lemma}\label{lemma:trans_armijo_step_size_lb}
The (exact) backtracking procedure with the following Armijo condition on the log-loss:
\begin{equation}
    \ln(f^* - f(\thetat + \etat \gradf{\thetat})) \leq \ln(f^* - f(\thetat)) - h \, \etat \, \frac{\normsq{\gradf{\thetat}}}{f^* - f(\tht)}
\end{equation}
terminates and returns 
\begin{align}\label{eq:eta_armijo_like_lb}
    \etat \geq \min\left\{\eta_{\max}, \frac{2(1-h)}{L_1 \, \nu\, [f^* - f(\tht)]}\right\}
\end{align}
where $h \in (0, 1)$ is a hyper-parameter.
\end{lemma}
\end{thmbox}
\begin{proof}
Let $g(\theta) = \ln(f^* - f(\theta))$. By \cref{lemma:armijo_like_smoothness}, $g$ is $L_1 \, \nu$-smooth. Starting with the quadratic bound using the smoothness of $g$:
\begin{align}
    g(\thetatt) &\leq g(\thetat) - \etat \, \dpd{\frac{\gradf{\thetat}}{f^* - f(\thetat)}, \gradf{\thetat}} + \frac{L_1 \, \nu \, \etat^2}{2} \, \normsq{\gradf{\thetat}} \\
    &\leq \underbrace{g(\thetat) - \normsq{\gradf{\thetat}} \, \parens*{\frac{\etat}{f^* - f(\tht)} - \frac{L_1 \, \nu \, \etat^2}{2}}}_{:= h_1(\etat)}
    \intertext{From \cref{eq:trans_line_search}}
    g(\thetat + \etat \gradf{\thetat}) &\leq \underbrace{g(\thetat) - h \, \etat \, \frac{\normsq{\gradf{\thetat}}}{f^* - f(\tht)}}_{:= h_2(\etat)}
\end{align}
If \cref{eq:trans_line_search} is satisfied, the backtracking line-search procedure terminates.
If $\eta_{\max} \leq \frac{2(1-h)}{L_1 \, \nu \, [f^* - f(\tht)]}$ then $g(\thetatt) \leq h_1(\eta_{\text{max}}) \leq h_2(\eta_{\text{max}})$ implying the line-search terminates and $\etat = \eta_{\text{max}}$. 
Otherwise, if $\eta_{\max} > \frac{2(1-h)}{L_1 \, \nu \, [f^* - f(\tht)]}$ and \cref{eq:trans_line_search} is satisfied for step-size $\etat$ then
\begin{align}
   \ln(\thetat + \etat \grad{\thetat}) &\leq h_2(\etat) \leq h_1(\etat) \\ 
   \implies \frac{h \etat}{f^* - f(\tht)} &\geq \frac{\etat}{f^* - f(\tht)} - \frac{L_1 \, \nu \etat^2}{2} \\
   \implies \etat &\geq \frac{2(1-h)}{L_1\, \nu \, [f^* - f(\tht)]}
\end{align}
Putting the above conditions together, we have:
\begin{align}
    \etat \geq \min\left\{\eta_{\text{max}}, \frac{2(1-h)}{L_1 \, \nu\, [f^* - f(\tht)]}\right\}.
\end{align}
\end{proof}
\pagebreak
\begin{thmbox}
\begin{lemma}[Lemma 17 in \citep{mei2020global}]\label{lemma:reverse_nl_bandit}
For any $r \in [0, 1]^A$. Denote $\Delta^* := r(a^*) - \max_{a \neq a^*} r(a)$. Then,
\begin{equation}
\norm*{\frac{d \dpd{\pitheta, r}}{d \theta}} \leq \frac{\sqrt{2}}{\Delta^*} \, \dpd{\pistar - \pitheta, r}.
\end{equation}
\end{lemma}
\end{thmbox}
\begin{thmbox}
\begin{lemma}[Lemma 28 in \citep{mei2020global}]\label{lemma:reverse_nl_mdp}
Denote $\Delta^*(s) := Q^*(s, a^*(s)) - \max_{a \neq a^*(s)} Q^{*}(s, a)$ as the optimal value gap of state $s$, where $a^*(s)$ is the action that the optimal policy selects under state $s$, and $\Delta^* := \min_{s \in \gS} \Delta^*(s) > 0$ as the optimal value gap of the MDP. Then we have
\begin{equation}
\norm*{\frac{\partial V^{\pitheta}(\rho)}{\partial \theta}} \leq \frac{1}{1 - \gamma} \, \frac{\sqrt{2}}{\Delta^*} \, [V^*(\rho) - V^{\pitheta}(\rho)].
\end{equation}
\end{lemma}
\end{thmbox}

\clearpage

\section{Proofs in \cref{section:spg_wo_entropy}}\label{appendix:spg_proofs}
\subsection{Proof Of Theorem \ref{theorem:spg_ess}}\label{appendix:proof_spg_ess}
\thereomspgess*
\begin{proof}
\begin{align}
    \intertext{Starting with the smoothness of $f$}
   \abs*{f(\thetatt) - f(\thetat) - \dpd{\grad{\thetat}, \thetatt - \thetat}} &\leq \frac{L}{2} \, \normsq{\thetat - \thetat} \\
   f(\thetatt) - f(\thetat) - \dpd{\grad{\thetat}, \thetatt - \thetat} &\geq -\frac{L}{2} \, \normsq{\thetat - \thetat} \\
   \intertext{Using \cref{update:spg}, $\thetatt = \thetat + \etat \, \hgrad{\thetat}$}
   f(\thetatt) - f(\thetat) - \etat \dpd{\grad{\thetat}, \hgrad{\thetat}} &\geq -\frac{L}{2} \, \etat^2 \,  \normsq{\hgrad{\thetat}} \\
   \implies f(\thetatt) &\geq f(\thetat) + \etat \, \dpd{\grad{\thetat}, \hgrad{\thetat}} -\frac{L}{2} \, \etat^2 \, \normsq{\hgrad{\thetat}} \\
   \intertext{Multiplying both sides by $-1$ and adding $f^*$}
   f^* - f(\thetatt) &\leq f^* - f(\thetat) - \etat \, \dpd{\grad{\thetat}, \hgrad{\thetat}} +\frac{L}{2} \, \etat^2 \, \normsq{\hgrad{\thetat}} 
\end{align}
\begin{align}
   \intertext{Taking expectation with respect to the randomness in iteration $t$ on both sides}
   \underbrace{\E[f^* - f(\thetatt)]}_{:= \delta(\thetatt)} &\leq \underbrace{\E[f^* - f(\thetat)]}_{:= \delta(\thetat)} - \etat \, \dpd{\grad{\thetat}, \ev{\hgrad{\thetat}}} +\frac{L \, \etat^2}{2} \,  \ev{\normsq{\hgrad{\thetat}}} \\
   \intertext{Assuming that the gradient is unbiased}
   \implies \delta(\thetatt) &= \delta(\thetat) - \etat \, \normsq{\grad{\thetat}}  +\frac{L \, \etat^2}{2} \, \ev{\normsq{\hgrad{\thetat}}} \label{eq:start_sgc_2_here}\\
   &\leq \delta(\thetat) - \etat \, \normsq{\grad{\thetat}}  +\frac{L \, \etat^2}{2} \,   \ev{\normsq{\hgrad{\thetat} - \grad{\thetat} + \grad{\thetat}}} \\
    \intertext{Expanding the square and since $\ev{\dpd{\grad{\thetat}, \hgrad{\thetat} - \grad{\thetat}}} = 0$}
    &\leq \delta(\thetat) - \etat \,  \normsq{\grad{\thetat}} + \frac{L \, \etat^2}{2} \ev{ \normsq{\hgrad{\thetat} - \grad{\thetat}}} +  \frac{L \, \etat^2}{2} \, \ev{ \normsq{\grad{\thetat}}} \\
    \intertext{Assuming that the variance is bounded by $\sigma^2$}
    &\leq \delta(\thetat) - \etat \, \normsq{\grad{\thetat}} + \frac{L \, \etat^2}{2} \, \parens*{\sigma^2 +  \ev{ \normsq{\grad{\thetat}}}} \\
    &\leq \delta(\thetat) - \frac{\etat}{2} \, \normsq{\grad{\thetat}} + \frac{L \, \etat^2}{2} \, \sigma^2 \tag{$\etat \leq \frac{1}{L}$} \\
   \intertext{Since $f$  satisfies the non-uniform \L ojsaiewciz condition with $\xi=0$}
   &\leq \delta(\thetat) - \frac{\etat}{2} \, [\delta(\thetat)]^2 \, [C(\theta_t)]^2  + \frac{L \, \etat^2}{2} \,  \sigma^2 \\
   \intertext{Assuming $m := \inf_{t \geq 1} [C(\thetat)]^2 > 0 $}
   &\leq \delta(\thetat) \, \parens*{1 -  \frac{\etat \, m}{2} \, \delta(\thetat) }   + \frac{L \, \etat^2}{2} \, \sigma^2.  \label{eq:spg_ess_sgc_start}
   \intertext{Taking expectation with respect to all previous iterations $t \geq 1$ on  both sides}
   \implies \E[\delta(\thtt)] &\leq  \E[\delta(\tht)] - \frac{\etat}{2} \E[m \, [\delta(\tht)]^2] + \frac{L \, \etat^2}{2} \, \sigma^2
\end{align}
To lower-bound $\E[m \, [\delta(\tht)]^2]$
\begin{align}
\E[\delta(\tht)] &= \E\left[\frac{1}{\sqrt{m}} \, \sqrt{m} \, \delta(\tht) \right] \\
\intertext{Using Cauchy-Schwarz since $m > 0$ and $\delta(\tht) > 0$  }
&\leq \sqrt{\E\left[\frac{1}{m}\right]} \, \sqrt{\E\left[m \, [\delta(\tht)]^2 \right]} \\
\implies \underbrace{\left[\E\left[\frac{1}{m}\right]\right]^{-1}}_{:= \mu} \, \E[\delta(\tht)]^2 &\leq \E\left[m \, [\delta(\tht)]^2 \right] 
\intertext{Hence}
\E[\delta(\thtt)] &\leq  \E[\delta(\tht)] \, \parens*{1 - \frac{\etat \, \mu}{2}   \, \E[\delta(\tht)]} + \frac{L \, \etat^2}{2} \, \sigma^2
\end{align}
If for some $t \in [1, T]$ we have $\E[\delta(\thetat)] < \eps$ then we are done and have converged to a $\eps$-neighbourhood within $T$ iterations and have achieved
\begin{equation}
    \min_{t \in [1, T]} \E[f^* - f(\theta_t)] \leq \eps.
\end{equation}
Otherwise, we have $\E[\delta(\thetat)] \geq \eps$ and thus
\begin{align}
\E[\delta(\thetatt)] &\leq \E[\delta(\thetat)] \, \parens*{1 -  \frac{\etat \, \mu \, \eps}{2} \, \etat}  + \frac{L \, \sigma^2}{2} \, \etat^2 \\
&= \E[\delta(\thetat)] \, \parens*{1 - \frac{\eta_0  \, \mu \, \eps}{2} \, \alpha^t}  + \frac{ \alpha^{2t} \, L \, \eta_0^2 \, \sigma^2}{2} \tag{$\etat = \eta_0 \, \alpha^t$}
\intertext{Define $\frac{1}{\kappa} := \frac{\eta_0  \, \mu \, \eps}{2}$ and since $\eta_0 = \frac{1}{L}$}
&\leq \E[\delta(\thetat)] \, \parens*{1 - \frac{1}{\kappa} \, \alpha^t}  + \frac{\alpha^{2t} \,  \sigma^2}{2 \, L}. \label{eq:bad_pl_ess_recurse}
\end{align}
Using \cref{eq:bad_pl_ess_recurse} and recursing from $t=1$ to $T$ we have
\begin{align}\label{eq:bad_pl_ess_bound}
    \E[\delta(\theta_{T+1})] &\leq \E[\delta(\theta_1)] \, \prod_{t=1}^T \, \parens*{1 - \frac{1}{\kappa} \, \alpha^t }  + \frac{\sigma^2}{2 \, L} \sum_{t=1}^T \alpha^{2t} \prod_{i=t+1}^T \, \parens*{1 - \frac{1}{\kappa} \, \alpha^i} \\
    \intertext{Using $1 - x \leq \exp(-x)$ and by summing up the geometric series}
    &\leq \E[\delta(\theta_1)] \, \exp \parens*{- \frac{1}{\kappa} \frac{\alpha - \alpha^{T+1}}{1 - \alpha}}  + \frac{\sigma^2}{2 \, L} \, \sum_{t=1}^T \alpha^{2t} \,\exp\parens*{- \frac{1}{\kappa} \, \frac{\alpha^{t+1} -\alpha^{T+1}}{1 - \alpha}}.
\end{align}
Let us now bound the second term on the RHS
\begin{align}
    \frac{\sigma^2}{2 \, L} \, \sum_{t=1}^{T} \alpha^{2t} \, \exp\left( - \frac{1}{\kappa} \frac{\alpha^{t+1} - \alpha^{T+1}}{ 1 - \alpha} \right)
    &= \frac{\sigma^2}{2 \, L} \,\exp\left(\frac{\alpha^{T+1}}{\kappa \, ( 1- \alpha)}\right) \sum_{t=1}^{T} \alpha^{2t}  \, \exp\left( -\frac{\alpha^{t+1}}{\kappa \, (1- \alpha)}\right)\\
    \intertext{By \cref{lemma:xgamma_bound}, $\exp(-x) \leq \left(\frac{2}{e \, x}\right)^2$}
    &\leq \frac{\sigma^2}{2 \, L} \,\exp\left(\frac{\alpha^{T+1}}{\kappa \, ( 1- \alpha)}\right) \sum_{t=1}^{T} \alpha^{2t} \, \left( \frac{2\,(1- \alpha)\,\kappa}{ e \, \alpha^{t+1}}\right)^2 \\
    &= \frac{\sigma^2}{2 \, L} \, \exp\left(\frac{\alpha^{T+1}}{\kappa \, ( 1- \alpha)}\right) \frac{4 \, (1-\alpha)^2 \, \kappa^2}{e^2 \, \alpha^2 } \, T \\
    \intertext{Since $1 - x \leq \ln\parens*{\frac{1}{x}}$ and using it to bound $(1 - \alpha)^2$ where $\alpha = \parens*{\frac{\beta}{T}}^{\frac{1}{T}}$}
    &\leq \frac{ \sigma^2}{2 \, L} \, \exp\left(\frac{\alpha^{T+1}}{\kappa \, ( 1- \alpha)}\right) \frac{4 \, \kappa^2}{e^2 \, \alpha^2 } \frac{\ln^2\parens*{\frac{T}{\beta}}}{T}.
\end{align}
Putting everything together
\begin{align}
    \E[\delta(\theta_{T+1})] &\leq  \E[\delta(\theta_1)] \, \exp \parens*{- \frac{1}{\kappa} \frac{\alpha - \alpha^{T+1}}{1 - \alpha}}  + \frac{\sigma^2}{2 \, L}  \, \exp\left(\frac{\alpha^{T+1}}{\kappa \, ( 1- \alpha)}\right) \,  \frac{4 \, \kappa^2}{e^2 \, \alpha^2 } \,  \frac{\ln^2\parens*{\frac{T}{\beta}}}{T}  \\
    &= \E[\delta(\theta_1)] \, \exp \parens*{\frac{\alpha^{T+1}}{\kappa \, (1 - \alpha)}} \, \exp \parens*{-\frac{\alpha}{\kappa \, (1 - \alpha)}}  + \frac{\sigma^2}{2 \, L}  \, \exp\left(\frac{\alpha^{T+1}}{\kappa \, ( 1- \alpha)}\right) \frac{4 \, \kappa^2}{e^2 \, \alpha^2 } \, \frac{\ln^2\parens*{\frac{T}{\beta}}}{T} 
 \intertext{By \cref{lemma:alpha_bound}, $\frac{\alpha^{T+1}}{1- \alpha} \leq \frac{2\beta}{\ln\parens*{\nicefrac{T}{\beta}}}$}
 &\leq \E[\delta(\theta_1)] \,  \exp \parens*{\frac{2 \,  \beta}{\kappa \, \ln(\nicefrac{T}{\beta})}} \, \exp \parens*{-\frac{\alpha}{\kappa \, (1 - \alpha)}}  + \frac{\sigma^2}{2 \, L}  \, \exp \parens*{\frac{2 \,  \beta}{\kappa \, \ln(\nicefrac{T}{\beta})}} \, \frac{4 \, \kappa^2}{e^2 \, \alpha^2 } \, \frac{\ln^2\parens*{\frac{T}{\beta}}}{T} 
 \intertext{Since $1 - x \leq \ln \parens*{\frac{1}{x}}$, $\frac{\alpha}{(1 - \alpha)} \geq \frac{\alpha \, T}{\ln(\nicefrac{T}{\beta})}$}
 &\leq \E[\delta(\theta_1)] \,  \exp \parens*{\frac{2 \, \beta}{\kappa \, \ln(\nicefrac{T}{\beta})}} \, \exp \parens*{-\frac{\alpha \, T}{\kappa \, \ln(\nicefrac{T}{\beta})}} + \frac{\sigma^2}{2 \, L}  \, \exp \parens*{\frac{2 \, \beta}{\kappa \, \ln(\nicefrac{T}{\beta})}} \, \frac{4 \,  \kappa^2}{e^2 \, \alpha^2 } \, \frac{\ln^2\parens*{\frac{T}{\beta}}}{T}.
 \end{align}
Making the dependence on the constants explicit
\begin{align}
\MoveEqLeft
\implies  \E[f^* - f(\theta_{T+1})] \notag \\ &\leq \E[f^* - f(\theta_1)] \, \exp \parens*{\frac{2 \, \beta}{\kappa \, \ln(\nicefrac{T}{\beta})}} \, \exp \parens*{-\frac{\alpha \,T}{\kappa \, \ln(\nicefrac{T}{\beta})}} +\frac{\sigma^2}{2 \, L} \, \exp \parens*{\frac{2 \, \beta}{\kappa \, \ln(\nicefrac{T}{\beta})}} \, \frac{4 \, \kappa^2}{e^2 \, \alpha^2 } \frac{\ln^2\parens*{\frac{T}{\beta}}}{T}  \\
\intertext{Since $\eps < 1$}
&= \E[f^* - f(\theta_1)] \, \exp \parens*{\frac{\mu \, \beta}{L \, \ln(\nicefrac{T}{\beta})}}\, \exp \parens*{-\frac{\mu \, \eps \, \alpha \, T}{2 \, L \,  \ln(\nicefrac{T}{\beta})}} + \exp \parens*{\frac{\mu \, \beta}{L \, \ln(\nicefrac{T}{\beta})}} \, \frac{32 \, L \, \sigma^2 \,  \ln^2\parens*{\frac{T}{\beta}}}{e^2 \, \alpha^2 \,  \mu^2 \, \eps^2 \, T}.
\end{align}
\end{proof}

\begin{corollary} 
In the bandit setting, for a given $\eps \in (0, 1)$,  using \cref{update:spg} with exponentially decreasing step-sizes $\eta_t = \eta_0 \, \alpha^t$  where $\eta_0 = \frac{5}{2}$ and $\alpha = \left(\frac{\beta}{T}\right)^\frac{1}{T}$, $\beta \geq 1$ results in the following convergence: \\ If $\E[(\pistar - \pi_{\thetat})^\top r] \geq \eps$ for all $t \in [1, T]$, then,
\begin{equation}
 \E[(\pistar - \pi_{\theta_{T+1}})^\top r] \leq 
 \E[(\pistar - \pi_{\theta_{1}})^\top r]
 \,C_1 \, \exp \parens*{-\frac{\alpha \, \eps \,  T}{\kappa \, \ln(\nicefrac{T}{\beta})}}  +  \frac{C_1 \, C_2 \, \ln^2\parens*{\frac{T}{\beta}}}{ \eps^2 \, T} 
\end{equation}
where $\mu := \left[\E\left[\min_{t \in [1, T]} [\pit(a^*)]\right]^{-2}\right]^{-1} > 0$, $\kappa := \frac{5}{\mu}$, $C_1 := \exp \parens*{\frac{2 \, \beta}{\kappa \, \ln(\nicefrac{T}{\beta})}}$ and $C_2 := \frac{4 \kappa^2}{ 5 \, e^2 \alpha^2}$. Otherwise, $\min_{t \in [1, T]} \E[f^* - f(\thetat)] \leq \eps$.
\end{corollary}
\begin{proof}
We can extend \cref{theorem:spg_ess} to the bandit setting since:
\begin{itemize}
[nolistsep]
    \item by \cref{lemma:lemma_2_mei_global}, $f$ is $\frac{5}{2}$-smooth
    \item by \cref{lemma:lemma_3_mei_global}, $f$ is non-uniform \L ojsiewciz with $\xi = 0$ and $C(\theta) = \pitheta(a^*)$
    \item since $T$ is finite and the updates are bounded, $\mu := \left[\E\left[\min_{t \in [1, T]} [\pit(a^*)]^{-2}\right]\right]^{-1} > 0$
    \item by \cref{lemma:lemma_5_mei_understanding}, the stochastic gradient is unbiased and $\sigma^2 \leq 2$
\end{itemize}
\end{proof}

\begin{corollary} 
In the tabular MDP setting, for a given $\eps \in (0, 1)$, using \cref{update:spg} with exponentially decreasing step-sizes $\eta_t = \eta_0 \, \alpha^t$  where $\eta_0 = \frac{(1 - \gamma)^3}{8}$ and $\alpha = \left(\frac{\beta}{T}\right)^\frac{1}{T}$, $\beta \geq 1$ results in the following convergence:  \\ If $\E[V^{*}(\rho) - V^{\pi_{\theta_t}}(\rho)] \geq \eps$ for all $t \in [1, T]$, then,
\begin{equation}
\E[V^{*}(\rho) - V^{\pi_{\theta_{T+1}}}(\rho)] \leq 
 \E[V^{*}(\rho) - V^{\pi_{\theta_{1}}}(\rho)]
 \,C_1 \, \exp \parens*{-\frac{\alpha \, \eps \,  T}{\kappa \, \ln(\nicefrac{T}{\beta})}} + \frac{C_1 \, C_2 \, \ln^2\parens*{\frac{T}{\beta}}}{ \eps^2 \, T} 
\end{equation}
where $\mu := \left[\E\left[\min_{t \in [1, T]} \parens*{\frac{\min_s \pitheta(a^*(s) | s)}{\sqrt{S} \, \supnorm{\nicefrac{d^{\pistar}_\rho}{d^{\pitheta}_\rho}}}}\right]^{-2}\right]^{-1} > 0$, $\kappa := \frac{16}{\mu \, (1 - \gamma)^3}$, $C_1 := \exp \parens*{\frac{2 \, \beta}{\kappa \, \ln(\nicefrac{T}{\beta})}}$ and $C_2 := \frac{A \, \kappa^2}{4 \, (1 - \gamma) \, e^2 \, \alpha^2}$. Otherwise, $\min_{t \in [1, T]} \E[V^{*}(\rho) - V^{\pi_{\theta_t}}(\rho)] \leq \eps$.
\end{corollary}
\begin{proof}
We can extend \cref{theorem:spg_ess} to the tabular MDP setting since:
\begin{itemize}
[nolistsep]
    \item by \cref{lemma:lemma_7_mei_global}, $f$ is $\frac{8}{(1 - \gamma)^3}$-smooth
    \item by \cref{lemma:lemma_8_mei_global}, $f$ is non-uniform \L ojsiewciz with $\xi = 0$ and $C(\theta) = \frac{\min_s \pitheta(a^*(s)  |  s)}{\sqrt{S} \, \supnorm{\frac{d^{\pistar}_\rho}{d^{\pitheta}_\rho}}}$
    \item since $T$ is finite and the updates are bounded, $\mu := \left[\E\left[\min_{t \in [1, T]} \parens*{\frac{\min_s \pitheta(a^*(s) |  s)}{\sqrt{S} \supnorm{\nicefrac{d^{\pistar}_\rho}{d^{\pitheta}_\rho}}}}^{-2}\right]\right]^{-1} > 0$
    \item by \cref{lemma:lemma_11_mei_understanding}, the stochastic gradient is unbiased and $\sigma^2 \leq \frac{2 \, S}{(1 - \gamma)^4}$
\end{itemize}
\end{proof}

\subsection{Proof of Theorem \ref{theorem:spg_ess_sgc}}\label{appendix:proof_spg_ess_sgc}
\theoremspgesssgc*
\begin{proof}
Assuming $f$ is $L_1 \, \norm{\grad{\theta}}$ non-uniform smooth and the stochastic gradients are bounded, i.e. $\norm{\hgrad{\theta}} \leq B$, by \cref{lemma:bound_gtheta_zeta}, using \cref{update:spg} with $\etat \in \parens*{0, \frac{1}{L_1 B}}$
\begin{align}
    \abs*{f(\thetatt) - f(\thetat) - \dpd{\grad{\thetat}, \thetatt - \thetat}} &\leq \frac{1}{2} \, \frac{L_1 \, \norm{\grad{\thetat}}}{1 - L_1 \, B \, \etat} \, \normsq{\thetatt - \thetat}
\end{align}
Then following the initial proof of \cref{theorem:spg_ess} we obtain
\begin{align}
\underbrace{\E[f^* - f(\thetatt)]}_{:= \delta(\thetatt)} & \leq \underbrace{\E[f^* - f(\thetat)]}_{ :=\delta(\thetat)} - \etat \, \normsq{\gradf{\tht}} + \frac{\etat^2}{2} \, \frac{L_1 \, \norm{\grad{\thetat}}}{1 - L_1 \, B \, \etat} \, \E\left[\normsq{\hgrad{\tht}}\right] \\  
\intertext{Assuming $f$ satisfies the strong growth condition, $\E \normsq{\hgrad{\tht}} \leq \varrho \, \norm{\gradf{\tht}}$}
&\leq \delta(\thetat) - \etat \, \normsq{\gradf{\tht}} +  \frac{\varrho \, \etat^2}{2} \, \frac{L_1}{1 - L_1 \, B\, \etat} \, \normsq{\grad{\thetat}} \\
\intertext{Since for all $t \geq 1$, $\etat \leq \eta_0$, $\frac{1}{1 - L_1 \, B \, \etat} \leq \frac{1}{1 - L_1 \, B \, \eta_0}$}
&\leq \delta(\thetat) - \etat \, \normsq{\gradf{\tht}} +  \frac{\varrho \, \etat^2}{2} \, \frac{L_1}{1 - L_1 \, B \, \eta_0} \,  \normsq{\grad{\thetat}} \\
\intertext{Picking $\eta_0$ such that $\frac{L_1}{1 - L_1 \, B \eta_0} < 1 \implies \eta_0 < \frac{1}{L_1^2 \, B}$}
\implies \delta(\thtt) &\leq \delta(\thetat) - \etat \, \normsq{\gradf{\tht}} + \frac{\varrho \, \etat^2}{2} \, \normsq{\grad{\thetat}}. \label{eq:phase_simplify}
\end{align}
Since $\etat$ is decreasing, we will now consider the following phases:
\begin{description}[style=unboxed,leftmargin=0cm]
    \item[Phase 1]: When $\etat$ is ``large'', i.e. $\etat > \frac{1}{\varrho}$
    \item[Phase 2]: When $\etat$ is ``small'', i.e. $\etat \leq \frac{1}{\varrho}$.
\end{description}
For $\etat \leq \frac{1}{\varrho}$ we require that 
\begin{align}
\eta_0 \, \left(\frac{\beta}{T}\right)^{\frac{t}{T}} \leq \frac{1}{\varrho} \implies t \geq T_0 := T \, \frac{\ln(\varrho \, \eta_0)}{\ln\parens*{\frac{T}{\beta}}}.
\end{align}
Hence, when $t \geq T_0$, the step-size is small enough to be in Phase 2. Let us first analyze Phase 1.

\textbf{Phase 1}: In Phase 1 we have $\etat > \frac{1}{\varrho}$. Starting with \cref{eq:phase_simplify}, 
\begin{align}
\delta(\thetatt)&\leq \delta(\thetat) - \etat \, \normsq{\gradf{\tht}} + \frac{\varrho \, \etat^2}{2} \, \normsq{\grad{\thetat}}.
\end{align}
To simplify $\normsq{\gradf{\tht}}$, since $f$ is $L$-smooth for any $\theta$ and $\theta'$
\begin{align}
    f(\theta') &\geq f(\theta) + \dpd{\grad{\theta}, \theta' - \theta}  - \frac{L}{2} \, \normsq{\theta' - \theta } 
    \intertext{Setting $\theta' = \theta + \frac{1}{L} \, \grad{\theta}$}
     &\geq f(\theta) + \frac{1}{L} \, \normsq{\grad{\theta}}   \\
    \implies \normsq{\grad{\theta}}  &\leq 2L \, [f(\theta') - f(\theta)] \leq 2L \, [f^* - f(\theta)] \\
    \implies \frac{\varrho}{2} \normsq{\grad{\thetat}}  &\leq \varrho \,  L \, [f^* - f(\theta)] = \varrho \,  L \, \delta(\thetat).
\end{align}
Hence,
\begin{align}
\delta(\thetatt) & \leq \delta(\thetat) - \frac{\etat}{2} \, \normsq{\gradf{\tht}} + L \, \varrho  \, \etat^2 \,  \delta(\thetat)
 \intertext{Since $f$ satisfies the non-uniform \L ojsaiewciz condition with $\xi = 0$}
 &\leq \delta(\thetat) - \frac{\etat \, [C(\thetat)]^2}{2} \, [\delta(\tht)]^2 + L \, \varrho  \, \etat^2 \,  \delta(\thetat)
 \intertext{Since $m := \inf_{t \geq 1} [C(\thetat)]^2 > 0$}
& \leq \delta(\thetat) - \frac{\etat \, m}{2} \, [\delta(\thetat)]^2 + \etat^2 \, \underbrace{L \, \varrho \, \delta(\tht)}_{:= \Gamma_t}
\intertext{Taking expectation with respect to all previous iterations $t \geq 1$ on both sides}
\implies \E[\delta(\thtt)] &\leq \E[\delta(\thetat)] - \frac{\etat}{2} \, \E[m \, \delta(\thetat)]^2] + \etat^2 \, \underbrace{L \, \varrho \, \E[\delta(\tht)]}_{:= \Gamma_t}
\intertext{Using Cauchy-Schwarz to lower-bound $\E[m \, [\delta(\tht)]^2]$}
&\leq \E[\delta(\thetat)] - \frac{\etat}{2 \, \E[m^{-1}]} \, \E[\delta(\thetat)] + \etat^2 \, \Gamma_t
\intertext{Define $\mu := \frac{1}{\E[m^{-1}]}$}
&\leq \E[\delta(\thetat)] - \frac{\etat \, \mu}{2} \, \E[\delta(\thetat)] + \etat^2 \, \Gamma_t
\intertext{If $\E[\delta(\thetat)] \leq \eps$ for some $t \in \{1,\ldots, T\}$, then we are done. Else for all $t \in \{1, \ldots, T\}$, $\E[\delta(\thetat)] > \eps$. Hence,}
\E[\delta(\thetatt)] & \leq \E[\delta(\thetat)] \, \parens*{1 - \frac{\etat \, \mu \, \eps}{2}} + \etat^2 \, \Gamma_t  \\
&= \E[\delta(\thetat)] \, \parens*{1 - \frac{\eta_0 \, \mu \, \eps}{2} \, \alpha^t } + \eta_0^2 \, \alpha^{2t} \, \Gamma_t \\
\intertext{Define $\frac{1}{\kappa} := \frac{\eta_0 \, \mu \, \eps}{2}$}
\implies \E[\delta(\thtt)] &= \E[\delta(\thetat)] \, \parens*{1 - \frac{1}{\kappa} \, \alpha^t } +  \eta_0^2 \, \alpha^{2t} \, \Gamma_t.
\label{eq:phase1-recursion}
\end{align}
Recall we are in Phase 1 when $t < T_0$. Using~\cref{eq:phase1-recursion} and recursing from $t = 1$ to $T_0 - 1$
\begin{align}
\E[\delta(\theta_{T_0})] &\leq \E[\delta(\theta_1)] \, \prod_{t=1}^{T_0 - 1} \left(1 -  \frac{1}{\kappa} \, \alpha^t \right) + \eta_0^2 \, \sum_{t=1}^{T_0 - 1} \alpha^{2t} \, \Gamma_t \, \prod_{i= t+1}^{T_0 - 1} \left( 1 - \frac{1}{\kappa} \, \alpha^i \right) \\
\intertext{Using $1 - x \leq \exp(-x)$ and by summing up the geometric series}
\implies \E[\delta(\theta_{T_0})] &\leq \E[\delta(\theta_1)] \, \exp \left(-\frac{1}{\kappa} \, \frac{\alpha - \alpha^{T_0}}{1- \alpha}\right)  + \eta_0^2 \, \sum_{t=1}^{T_0 - 1} \, \alpha^{2t} \, \Gamma_t \, \exp\left( - \frac{1}{\kappa} \, \frac{\alpha^{t+1} - \alpha^{T_0}}{ 1 - \alpha} \right).
\end{align}
Let us now bound the second term on the RHS
\begin{align}
    \eta_0^2 \, \sum_{t=1}^{T_0 - 1} \, \alpha^{2t} \, \Gamma_t \,\exp\left( - \frac{1}{\kappa} \,  \frac{\alpha^{t+1} - \alpha^{T_0}}{ 1 - \alpha} \right)
    &= \eta_0^2 \, \exp\left(\frac{\alpha^{T_0}}{\kappa \, ( 1- \alpha)}\right) \sum_{t=1}^{T_0 - 1} \, \alpha^{2t} \,  \Gamma_t \, \exp\left( -\frac{\alpha^{t+1}}{\kappa \, (1- \alpha)}\right)\\
    \intertext{{By \cref{lemma:xgamma_bound}, $\exp(-x) \leq \left(\frac{2}{e x}\right)^2$}}
    &\leq \eta_0^2 \, \exp\left(\frac{\alpha^{T_0}}{\kappa \, ( 1- \alpha)}\right) \,  \sum_{t=1}^{T_0 - 1} \alpha^{2t} \, \Gamma_t \,\left( \frac{2 \, (1- \alpha) \, \kappa}{ e \, \alpha^{t+1}}\right)^2 \\
    &= \exp\left(\frac{\alpha^{T_0}}{\kappa \, ( 1- \alpha)}\right) \, \frac{4 \, \eta_0^2 \, (1-\alpha)^2 \, \kappa^2}{e^2 \, \alpha^2 } \sum_{t=1}^{T_0 - 1} \Gamma_t \\
    \intertext{Since $1 - x \leq \ln\parens*{\frac{1}{x}}$ and using it to bound $(1 - \alpha)^2$ where $\alpha = \parens*{\frac{\beta}{T}}^{\frac{1}{T}}$}
    &\leq \exp\left(\frac{\alpha^{T_0}}{\kappa \, ( 1- \alpha)}\right) \, \frac{4 \,  \eta_0^2 \, \kappa^2}{e^2 \, \alpha^2 } \frac{\ln^2\parens*{\frac{T}{\beta}} \, \sum_{t=1}^{T_0 - 1} \Gamma_t}{T^2}.
\end{align}
Putting everything together,
\begin{align}
    \implies  \E[\delta(\theta_{T_0})] &\leq \E[\delta(\theta_1)] \,  \exp \left(-\frac{1}{\kappa} \, \frac{\alpha - \alpha^{T_0}}{1- \alpha}\right) + \exp\left(\frac{\alpha^{T_0}}{\kappa \, ( 1- \alpha)}\right) \,  \frac{4 \, \eta_0^2 \, \kappa^2}{e^2 \, \alpha^2 } \, \frac{\ln^2\parens*{\frac{T}{\beta}} \, \sum_{t=1}^{T_0 - 1} \Gamma_t}{T^2}.   \label{eq:end_of_phase_1}
\end{align}

Now let us consider Phase 2.  \\
\textbf{Phase 2}: We are in Phase 2 when $\etat \leq \frac{1}{\varrho}$.
Starting with \cref{eq:phase_simplify}, 
\begin{align}
\delta(\thetatt)&\leq \delta(\thetat) - \etat \, \normsq{\gradf{\tht}} +  \frac{\varrho \, \etat^2}{2} \, \normsq{\grad{\thetat}} \\
\intertext{Since $f$ satisfies the non-uniform \L ojsaiewciz condition with $\xi = 0$}
&\leq \delta(\thetat) - \frac{\etat \, [C(\thetat)]^2}{2} \, [\delta(\thetat)]^2  \\
\intertext{Since $m := \inf_{t \geq 1} [C(\thetat)]^2 > 0$}
&\leq \delta(\thetat) - \frac{\etat \, m}{2} \, [\delta(\thetat)]^2
\intertext{Taking expectation with respect to all previous iterations $t \geq 1$ on both sides}
\E[\delta(\thtt)] &\leq \E[\delta(\tht)] - \frac{\etat}{2} \, \E[m [\delta(\tht)]^2]
\intertext{Using Cauchy-Schwarz to lower-bound $\E[m \, [\delta(\tht)]^2]$}
\E[\delta(\thtt)] &\leq \E[\delta(\tht)] - \frac{\etat}{2 \, \E[m^{-1}]} \, \E[\delta(\tht)]
\intertext{Define $\mu := \frac{1}{\E[m^{-1}]}$}
\E[\delta(\thtt)] &\leq \E[\delta(\tht)] - \frac{\etat \, \mu}{2} \, \E[\delta(\tht)]
\intertext{If $\E[\delta(\tht)] \leq \eps$ for some $t \in \{1, \dots, T\}$, then we are done. Else for all $t \in \{1, \dots, T\}$, $\E[\delta(\tht)] > \eps$. Hence,}
\E[\delta(\thetatt)] & \leq \E[\delta(\thetat)] \, \parens*{1 - \frac{\etat \, \mu \, \eps}{2}}. \label{eq:phase2-recursion}
\end{align}
Recall we are in Phase 2 when $t \geq T_0$. Using~\cref{eq:phase2-recursion} and recursing from $t = T_0$ to $T$
\begin{align}
\E[\delta({\theta_{T+1}})] & \leq \prod_{t=T_0}^T \parens*{1 - \frac{\etat \, \mu \, \eps}{2}}  \, \E[\delta(\theta_{T_0})]
\intertext{Using $1 - x \leq \exp(-x)$}
 \E[\delta(\theta_{T+1})] &\leq \exp\left(-\frac{\mu \, \eps}{2} \, \sum_{t = T_0}^{T} \etat \right) \, \E[\delta(\theta_{T_0})] \\
\intertext{Since $\eta_t = \eta_0 \, \alpha^t$ and summing up the geometric series}
 \implies \E[\delta(\theta_{T+1})] &\leq \exp\left(-\frac{\eta_0 \, \mu \, \eps}{2} \, \frac{\alpha^{T_0} - \alpha^{T+1}}{1 - \alpha} \right) \, \E[\delta(\theta_{T_0})]  \\
 \intertext{Since $\frac{1}{\kappa} = \frac{\eta_0 \, \mu \, \eps}{2}$}
&= \exp\left(-\frac{1}{\kappa} \, \frac{\alpha^{T_0} - \alpha^{T+1}}{1 - \alpha} \right) \, \E[\delta(\theta_{T_0})]. \label{eq:end_of_phase_2}
\end{align}

\begin{align}
\looseness-1
\intertext{Combining the results of Phase 1 (\cref{eq:end_of_phase_1}) and Phase 2 (\cref{eq:end_of_phase_2})}
 \E[\delta(\theta_{T+1})] &\leq \exp\left(-\frac{1}{\kappa} \, \frac{\alpha^{T_0} - \alpha^{T+1}}{1 - \alpha} \right) \notag \\  &\left[\E[\delta(\theta_1)] \,  \exp \left(-\frac{1}{\kappa}  \, \frac{\alpha - \alpha^{T_0}}{1- \alpha}\right) + \exp\left(\frac{\alpha^{T_0}}{\kappa \, ( 1- \alpha)}\right) \frac{4 \,  \eta_0^2 \, \kappa^2}{e^2 \, \alpha^2 } \, \frac{\ln^2\parens*{\frac{T}{\beta}} \, \sum_{t=1}^{T_0 - 1} \Gamma_t}{T^2} \right] \\
 &= \E[\delta(\theta_1)] \, \exp\left(-\frac{1}{\kappa}  \, \frac{\alpha - \alpha^{T+1}}{1 - \alpha} \right) + \exp\left(\frac{\alpha^{T+1}}{\kappa \, ( 1- \alpha)}\right)  \, \frac{4 \,  \eta_0^2 \, \kappa^2}{e^2 \, \alpha^2 } \,  \frac{\ln^2\parens*{\frac{T}{\beta}} \, \sum_{t=1}^{T_0 - 1} \Gamma_t}{T^2} \\
 &= \E[\delta(\theta_1)] \, \exp\left(\frac{\alpha^{T+1}}{\kappa \, (1 - \alpha)} \right)  \, \exp\left(-\frac{\alpha}{\kappa \, (1 - \alpha)} \right) \notag \\ &+ \exp\left(\frac{\alpha^{T+1}}{\kappa \, (1 - \alpha)} \right)  \, \frac{4 \, \eta_0^2 \, \kappa^2}{e^2 \, \alpha^2 } \, \frac{\ln^2\parens*{\frac{T}{\beta}} \, \sum_{t=1}^{T_0 - 1} \Gamma_t}{T^2} \\
 \intertext{By \cref{lemma:alpha_bound}, $\frac{\alpha^{T+1}}{( 1- \alpha)} \leq \frac{2\beta}{\ln\parens*{\nicefrac{T}{\beta}}}$}
 &\leq \E[\delta(\theta_1)] \, \exp\left(\frac{2\beta}{\kappa \, \ln\parens*{\nicefrac{T}{\beta}}} \right)  \, \exp\left(-\frac{\alpha}{\kappa \, (1 - \alpha)} \right) \notag \\ &+ \exp\left(\frac{2 \, \beta}{\kappa \, \ln\parens*{\nicefrac{T}{\beta}}} \right)  \, \frac{4 \,  \eta_0^2 \, \kappa^2}{e^2 \, \alpha^2 } \frac{\ln^2\parens*{\frac{T}{\beta}} \, \sum_{t=1}^{T_0 - 1} \Gamma_t}{T^2} \\
 \intertext{Since $1 - x \leq \ln \parens*{\frac{1}{x}}$, $\frac{\alpha}{1 - \alpha} \geq \frac{\alpha T}{\ln(\nicefrac{T}{\beta})}$}
 &\leq \E[\delta(\theta_1)] \, \underbrace{\exp\left(\frac{2 \, \beta}{\kappa \, \ln\parens*{\nicefrac{T}{\beta}}} \right)}_{:= C_1}  \, \exp\left(-\frac{\alpha \, T}{\kappa \, \ln(\nicefrac{T}{\beta})} \right) \notag \\ &+ \underbrace{\exp\left(\frac{2 \, \beta}{\kappa \, \ln\parens*{\nicefrac{T}{\beta}}} \right)  \, \frac{4 \, \eta_0^2 \, \kappa^2}{e^2 \, \alpha^2 } \ln^2\parens*{\frac{T}{\beta}}}_{:= C_2} \, \frac{\sum_{t=1}^{T_0 - 1} \Gamma_t}{T^2} \\
 \implies \E[\delta(\theta_{T+1})] & \leq C_1 \, \E[\delta(\theta_1)] \,  \exp\left(-\frac{\alpha T}{\kappa \, \ln(\nicefrac{T}{\beta})} \right)  + C_2 \, \frac{\sum_{t=1}^{T_0 - 1} \Gamma_t}{T^2}.
\end{align}
\begin{align}
\intertext{Making the dependence on the constants explicit}
\implies &\E[\delta(\theta_{T+1})] \notag \\
&\leq \E[\delta(\theta_1)] \, \exp\left(\frac{\mu \, \eps \, \eta_0 \, \beta}{\ln(\nicefrac{T}{\beta})} \right) \, \exp\left(\frac{-\mu \, \eps \, \eta_0 \, \alpha \, T}{2 \, \ln(\nicefrac{T}{\beta})} \right) + \exp\left(\frac{\mu \, \eps \, \eta_0 \, \beta}{\ln(\nicefrac{T}{\beta})} \right) \,  \frac{16 \, L \, \varrho \, \ln^2(\nicefrac{T}{\beta})}{e^2 \, \alpha^2 \, \mu^2 \, \eps^2}  \, \frac{\sum_{t=1}^{T_0 - 1} \E[\delta(\thetat)]}{T^2}  
\intertext{Since $\eps < 1$}
&\leq \E[\delta(\theta_1)] \, \exp\left(\frac{\mu \, \eta_0 \, \beta}{\ln(\nicefrac{T}{\beta})} \right) \, \exp\left(\frac{-\mu \, \eps \, \eta_0 \, \alpha \, T}{2 \, \ln(\nicefrac{T}{\beta})} \right) + \exp\left(\frac{\mu \, \eta_0 \, \beta}{\ln(\nicefrac{T}{\beta})} \right) \,  \frac{16 \, L \, \varrho \, \ln^2(\nicefrac{T}{\beta})}{e^2 \, \alpha^2 \, \mu^2 \, \eps^2}  \, \frac{\sum_{t=1}^{T_0 - 1} \E[\delta(\thetat)]}{T^2}  
\end{align}
\end{proof}
\corollarybanditssgcsgp*
\begin{proof}
We can extend \cref{theorem:spg_ess_sgc} to the bandit setting since:
\begin{itemize}
[nolistsep]
    \item by \cref{lemma:lemma_2_mei_global}, $f$ is $\frac{5}{2}$-smooth
    \item by \cref{lemma:bandit_ns}, $f$ is $3$-non-uniform smooth
    \item by~\cref{lemma:lemma_3_mei_global}, $f$ is non-uniform \L ojsiewciz with $\xi = 0$ and $C(\theta) = \pitheta(a^*)$
    \item since $T$ is finite and the updates are bounded, $\mu := \left[\E\left[\min_{t  \in [1, T]} \pit(a^*)^{-2}\right]\right]^{-1} > 0$
    \item by \cref{lemma:lemma_5_mei_understanding}, the stochastic gradient is unbiased
    \item by \cref{lemma:bandit_sgc}, the stochastic gradient satisfies the strong growth condition with $\varrho = \frac{8 \, A^{3/2}}{\Delta^2}$ where $\Delta := \min_{a \neq a'} \abs{r(a) - r(a')}$
    \item by \citet[Equation 52]{mei2023stochastic} $\norm{\frac{d \dpd{\pi_{\theta}, r}}{d \theta}} \leq \sqrt{2}$ and $\eta_0  := \frac{1}{18} < \frac{1}{L_1^2 B} = \frac{1}{9 \sqrt{2}}$.
\end{itemize}
\end{proof}

\begin{corollary}\label{corolllary:mdp_sgc_ess}
Assuming $\min_{s \in \gS} \rho(s) > 0$, in the tabular MDP setting, for a given $\eps \in (0, 1)$, using \cref{update:spg} with exponentially decreasing step-sizes $\eta_t = \eta_0 \, \alpha^t$  where $\eta_0 < \frac{1}{C^2 B}$ and $\alpha = \left(\frac{\beta}{T}\right)^\frac{1}{T}$, $\beta \geq 1$ results in the following convergence:  \\ If  $\E[V^{\pistar}(\rho) - V^{\pit}(\rho)] \geq \eps$ for all $t \in [1, T]$, then,
\begin{equation}
 \E[V^{\pistar}(\rho) - V^{\pi_{\theta_{T+1}}}(\rho)] \leq 
 \E[V^{\pistar}(\rho) - V^{\pi_{\theta_{1}}}(\rho)]
 \,C_1 \, \exp \parens*{-\frac{\alpha \, \eps \,  T}{\kappa \, \ln(\nicefrac{T}{\beta})}}  +   \frac{C_2 \, \sum_{t=1}^{T_0 - 1} \E[V^{\pistar}(\rho) - V^{\pit}(\rho)]}{ \eps^2 \, T^2} 
\end{equation}
where $C := \bracks*{3 + \frac{2 \, C_\infty - (1 - \gamma)}{(1 -\gamma) \, \gamma}} \, \sqrt{S}$, $C_\infty := \max_\pi \supnorm{\frac{d^\pi_\rho}{\rho}} \leq \frac{1}{\min_s \rho(s)} < \infty$, $B := \frac{\sqrt{2 \, S}}{(1 - \gamma)^4}$, $\kappa := \frac{2}{\mu \, \eta_0}$, $C_1 := \exp\parens*{\frac{2 \, \beta}{\kappa \, \ln(\nicefrac{T}{\beta})}}$, $C_2 := \exp\parens*{\frac{2 \, \beta}{\kappa \, \ln(\nicefrac{T}{\beta})}} \frac{128 \, \varrho \, \kappa^2}{(1 - \gamma)^3 \, e^2 \, \alpha^2} \, \ln^2(\nicefrac{T}{\beta})$, $T_0 := T\, \max \left\{\frac{\ln(\varrho \, \eta_0)}{\ln(\nicefrac{T}{\beta})}, 0 \right\}$ and 
$\mu := \left[\E\left[\min_{t \in [1, T]} \parens*{\frac{\min_s \pitheta(a^*(s) | s)}{\sqrt{S} \supnorm{\nicefrac{d^{\pistar}_\rho}{d^{\pitheta}_\rho}}}}^{-2}\right]\right]^{-1} > 0$. Otherwise, ${\min_{t \in [1, T]} \E[V^{\pistar}(\rho) - V^{\pi_{\theta_{t}}}(\rho)] \leq \eps}$.
\end{corollary}
\begin{proof}
We can extend \cref{theorem:spg_ess_sgc} to the tabular MDP setting since:
\begin{itemize}
[nolistsep]
    \item by \cref{lemma:lemma_7_mei_global}, $f$ is $\frac{8}{(1 - \gamma)^3}$-smooth
    \item by \cref{lemma:mdp_ns}, $f$ is $C$-non-uniform smooth where $C := \bracks*{3 + \frac{2 \, C_\infty - (1 - \gamma)}{(1 -\gamma) \gamma}} \, \sqrt{S}$ and $C_\infty := \max_\pi \supnorm{\frac{d^\pi_\rho}{\rho}} \leq \frac{1}{\min_s \rho(s)} < \infty$,
    \item by \cref{lemma:lemma_8_mei_global}, $f$ is non-uniform \L ojsiewciz with $\xi = 0$ and $C(\theta) = \frac{\min_s \pitheta(a^*(s) | s)}{\sqrt{S} \,  \supnorm{\nicefrac{d^{\pistar}_\rho}{d^{\pitheta}_\rho}}}$
    \item since $T$ is finite and the update is bounded, $\mu := \left[\E\left[\min_{t \in [1, T]} \parens*{\frac{\min_s \pitheta(a^*(s) | s)}{\sqrt{S} \supnorm{\nicefrac{d^{\pistar}_\rho}{d^{\pitheta}_\rho}}}}^{-2}\right]\right]^{-1} > 0$
    \item by \cref{lemma:lemma_11_mei_understanding}, the stochastic gradient is unbiased
    \item by \cref{theorem:mdp_sgc}, the stochastic gradient satisfies the strong growth condition with $\varrho = \frac{4 \, A^{3/2} \, S^{1/2}}{(1 - \gamma)^4 \, \Delta^2}$ where $\Delta := \min_s \min_{a \neq a'} \abs*{Q^{\pitheta}(s, a) - Q^{\pitheta}(s, a')}$
    \item by \cref{eq:sgc_mdp_ub_0}, $\norm{\hgrad{\thetat}} \leq B := \frac{\sqrt{2 \, S}}{(1 - \gamma)^2}$
\end{itemize} 
\end{proof}
\clearpage
\begin{corollary} 
Assuming $\rho(s) = \frac{1}{S}$ for all $s \in \gS$, in the tabular MDP setting, for a given $\eps \in (0, 1)$,  using \cref{update:spg} with exponentially decreasing step-sizes $\eta_t = \eta_0 \, \alpha^t$  where  $\eta_0 < \frac{1}{C^2 B}$ and $\alpha = \left(\frac{\beta}{T}\right)^\frac{1}{T}$, $\beta \geq 1$ results in the following convergence:  \\ If $\E[V^{\pistar}(\rho) - V^{\pit}(\rho)] \geq \eps$ for all $t \in [1, T]$, then,
\begin{equation}
 \E[V^{\pistar}(\rho) - V^{\pi_{\theta_{T+1}}}(\rho)] \leq 
 \E[V^{\pistar}(\rho) - V^{\pi_{\theta_{1}}}(\rho)]
 \,C_1 \, \exp \parens*{-\frac{\alpha \, \eps \,  T}{\kappa \, \ln(\nicefrac{T}{\beta})}}  +   \frac{C_2 \sum_{t=1}^{T_0 - 1} \E[V^{\pistar}(\rho) - V^{\pit}(\rho)]}{ \eps^2 \, T^2} 
\end{equation}
where $C := \bracks*{3 + \frac{2 \, A^{-1} - (1 - \gamma)}{(1 -\gamma) \, \gamma}} \, \sqrt{S}$, $B := \frac{\sqrt{2 \, S}}{(1 - \gamma)^2}$, $\kappa := \frac{2}{\mu \, \eta_0}$, $C_1 := \exp\parens*{\frac{2 \, \beta}{\kappa \, \ln(\nicefrac{T}{\beta})}}$, $C_2 := \exp\parens*{\frac{2 \ ,\beta}{\kappa \, \ln(\nicefrac{T}{\beta})}} \, \frac{128 \, \varrho \, \kappa^2}{(1 - \gamma)^3 \, e^2 \, \alpha^2} \, \ln^2(\nicefrac{T}{\beta})$, $T_0 := T\, \max \left\{\frac{\ln(\varrho \, \eta_0)}{\ln(\nicefrac{T}{\beta})}, 0 \right\}$ and  $\mu := \left[\E\left[\min_{t \in [1, T]} \parens*{\frac{\min_s \pitheta(a^*(s) | s)}{\sqrt{S} \supnorm{\nicefrac{d^{\pistar}_\rho}{d^{\pitheta}_\rho}}}}^{-2}\right]\right]^{-1} > 0$. Otherwise, ${\min_{t \in [1, T]} \E[V^{\pistar}(\rho) - V^{\pi_{\theta_{t}}}(\rho)] \leq \eps}$.
\end{corollary}
\begin{proof}
Follows from \cref{corolllary:mdp_sgc_ess}.
\end{proof}

\subsection{Strong Growth Condition - Dependence of Reward Gap}\label{appendix:proof_sgc_delta}
We first show that the dependence of the reward gap $\Delta$ in the SGC constant $\varrho$ cannot be removed. 
\begin{thmbox}
\begin{restatable}{prop}{proprhonecessary}\label{proposition:spg_rho}
    The dependence of $\Delta$ in the strong growth condition in \cref{lemma:bandit_sgc} is necessary.
\end{restatable}
\end{thmbox}
\begin{proof}
Consider a 2-arm bandit problem with deterministic rewards: $r_1 := r(1)$ and $r_2 := r(2)$. Assume that $\Delta := r_1 - r_2 > 0$, and hence arm $1$ is the optimal arm. We will show that in SGC in \cref{lemma:bandit_sgc}, the dependence of $\Delta$ in the SGC constant $\varrho$ is necessary. Let $\hat{r}(a) := \frac{\indicator{a_t = a}}{\pit(a)} \, r(a)$ for all $a \in \gA$. The stochastic gradient estimate satisfies the following SGC: 
\begin{align}
\E_t \left[ \normsq{\frac{d [\dpd{\pi_{\theta_t}, \hat{r}_t} ]}{d \theta_t}} \right] & \leq \varrho \, \norm*{\frac{d [\dpd{\pit,  r} ]}{d \theta_t}}.
\label{eq:SGC}    
\end{align}
Calculating the LHS
\begin{align}
\frac{d \dpd{\pit, \hat{r}_t}}{d \theta_t(a)} &= \left[ \indicator{a_t = a} - \pit(a) \right] \, r(a_t) \\ \implies \normsq{\frac{d  \dpd{\pi_{\theta_t}, \hat{r}_t}}{d \theta_t}} &= \sum_{a} \left[ \left[ \indicator{a_t = a} - \pit(a) \right] \, r(a_t)\right]^{2} \\
\intertext{Let $p := \pit(a_1)$ as the probability of pulling the optimal arm}
& = \left[ \left[ \indicator{a_t = a_1} - p \right] \, r(a_t)\right]^{2} + \left[ \left[ \indicator{a_t = a_2} - (1 - p) \right] \, r(a_t)\right]^{2}. 
\end{align}
\begin{align}
\E_t \left[ \normsq{\frac{d \dpd{\pi_{\theta_t}, \hat{r}_t}}{d \theta_t}} \right] &= \E_t \left[ \normsq{\frac{d \dpd{ \pi_{\theta_t},  \hat{r}_t}}{d \theta_t}} \vert \, a_t = a_1 \right] \, \Pr[a_t = a_1] \notag \\ &+ \E_t \left[ \normsq{\frac{d \dpd{ \pi_{\theta_t}, \hat{r}_t}}{d \theta_t}} \vert \, a_t \neq a_1 \right] \, \Pr[a_t \neq a_1] \\
&=  \left( (1 - p)^2 \, r_1^2 + (1-p)^2 \, r_1^2 \right) \, p + \left( p^2 \, r_2^2 + p^2 \, r_2^2 \right) \, (1-p) \\
\implies \text{LHS} &= 2 p \, (1-p)^2 \, r_1^2 + 2 (1-p) \, p^2 \, r_2^2 = 2 p \, (1-p) \left[(1 - p) \, r_1^2 + p \, r_2^2 \right].
\end{align}
Calculating the RHS 
\begin{align}
\frac{d \dpd{ \pi_{\theta_t}, r} }{d \theta_t(a)} &= \pi_{\theta_t}(a) \,
\left[r_{a} -  \dpd{\pi_{\theta_t}, r}  \right] \\
\implies \normsq{\frac{d  \dpd{\pi_{\theta_t}, r} }{d \theta_t}} &= \sum_{a} \pi_{\theta_t}(a)^{2} \, \left[r_{a} - \dpd{\pi_{\theta_t}, r} \right]^{2} \\
&= p^2 \, \left[r_{1} -  \dpd{\pi_{\theta_t}, r} \right]^{2} + (1-p)^2 \, \left[r_{2} - \dpd{\pi_{\theta_t}, r} \right]^{2} \\
\intertext{Since $\dpd{\pi_{\theta_t}, r} = p \, r_1 + (1-p) \, r_2$}
&= p^2 \, \left[r_{1} - [p \, r_1 + (1-p) \, r_2] \right]^{2} + (1-p)^2 \, \left[r_{2} - [p \, r_1 + (1-p) \, r_2] \right]^{2} \\
&= p^2 \, (1-p)^2 \, \Delta^2 + (1 - p)^2 \, p^2 \, \Delta^2 = 2 \, p^2 \, (1- p)^2 \, \Delta^2 \\
\implies \text{RHS} &= \norm*{\frac{d  \dpd{\pi_{\theta_t}, r} }{d \theta_t}} = \sqrt{2} \, p \, (1-p) \, \Delta.
\end{align}
Hence, 
\begin{align*}
\text{LHS} &=  \frac{\sqrt{2} \, \left[(1 - p) \, r_1^2 + p \, r_2^2 \right] }{\Delta} \, \text{RHS} \implies \varrho = \frac{\sqrt{2} \, \left[(1 - p) \, r_1^2 + p \, r_2^2 \right] }{\Delta}.
\end{align*}
For rewards $r_1 > r_2 > 0$, the numerator depends on the magnitude of the rewards, while the denominator depends on their gap. Since we have derived an equality, the dependence on $\frac{1}{\Delta}$ in $\varrho$ is necessary. 
\end{proof}

\subsection{Strong Growth Condition - Tabular MDP Setting, IS Parallel Estimator}\label{appendix:proof_sgc_mdp}
Following \citep[Definition 3]{mei2021understanding}, we first consider stochastic gradients using the on-policy parallel IS estimator.

\begin{restatable}[On-policy parallel IS estimator]{defn}{defnismdp}\label{defn:importance_sample_mdp}
In the tabular MDP setting, at iteration $t$, under each state $s \in \gS$ sample one action $a_t(s) \sim \pit(\cdot | s)$. The IS state-action value estimator $\hat{Q}^{\pit}$ is constructed as $\hat{Q}^{\pit}(s, a) = \frac{\indicator{a_t(s) = a}}{\pit(a | s)} \, Q^{\pit}(s, a)$ 
for all $(s, a) \in \gS \times \gA$.
\end{restatable}

Using this parallel IS parallel estimator, the following PG estimator constructed in \cref{alg:softmax_pg_general_on_policy_stochastic_gradient} satisfies the SGC.
\begin{figure}[H]
\centering
\vskip -0.1in
\begin{algorithm2e}[H]
\DontPrintSemicolon
\caption{Softmax PG, on-policy stochastic gradient}
\label{alg:softmax_pg_general_on_policy_stochastic_gradient}
\KwIn{Learning rate $\eta > 0$.}
\KwOut{Policy $\pi_{\theta_t} = \mathrm{softmax}(\theta_t)$.}
Initialize parameters $\theta_1(s, a)$ for all $(s, a) \in \gS \times \gA$\;
\While{$t \geq 1$}
{
    Sample $a_t(s) \sim \pi_{\theta_t}(\cdot | s)$ for all $s \in \gS$\;
    $\hat{Q}^{\pi_{\theta_t}}(s,a) \gets \frac{ \sI\left\{ a_t(s) = a \right\} }{ \pi_{\theta_t}(a | s) } \, Q^{\pi_{\theta_t}}(s,a)$ \;
    $\hat{g}_t(s, \cdot) \gets \frac{ 1 }{1-\gamma} \, d_{\rho}^{\pi_{\theta_t}}(s) \, { \left[ \sum_{a} \frac{\partial \pi_{\theta_t}(a | s)}{\partial \theta_t(s, \cdot)} \, \hat{Q}^{\pi_{\theta_t}}(s,a) \right] }$ \;
    $\theta_{t+1} \gets \theta_t + \eta \, \hat{g}_t$ \;
}
\end{algorithm2e}
\end{figure}

Recall in the tabular MDP setting, the PG theorem \citep{sutton1999policy} states
\begin{align}
    \frac{\partial V^{\pit}(\rho)}{\partial \theta} = \frac{1}{1 - \gamma} \E_{s' \sim d^{\pitheta}_\rho}\bracks*{\sum_{a' \in \gA} \frac{\partial \pitheta(a' | s')}{\partial \theta} \, Q^{\pitheta}(s', a')}.
\end{align}
For tabular softmax policy for any $s' \neq s$ and any $a \in \gA$, $\frac{\partial \pitheta(a | s')}{\partial \theta(s, \cdot)} = \mathbf{0}$. Hence,
\begin{align}
    \frac{V^{\pitheta}(\rho)}{\partial \theta(s, a)} &= \frac{1}{1 - \gamma} \, d^{\pitheta}_\rho(s) \, \pitheta(a |s) \, \parens*{Q^{\pitheta}(s, a) - \dpd{\pitheta(\cdot), Q^{\pitheta}(s, \cdot)}}. 
\end{align}
In contrast, in \cref{alg:softmax_pg_general_on_policy_stochastic_gradient} the stochastic gradient is 
\begin{align}
    \hat{g}(s, a) &= \frac{1}{1 - \gamma} \, d^{\pitheta}_\rho(s) \, \pitheta(a |s) \, \parens*{\hat{Q}^{\pitheta}(s, a) - \dpd{\pitheta(\cdot), \hat{Q}^{\pitheta}(s, \cdot)}}. 
\end{align}

\begin{theorem}\label{theorem:mdp_sgc}
In the tabular MDP setting, using \cref{update:spg} with the on-policy parallel IS estimator, we have for any $\theta$,
\begin{equation}
   \E\bracks*{\sum_{s \in \gS} \sum_{a \in \gA} \frac{d^{\pitheta}_\rho(s)^2}{(1 - \gamma)^2} \pitheta(a \, | \, s)^2 \parens*{\hat{Q}^{\pitheta}(s, a) - \dpd{\pitheta(\cdot \, | \, s), \hat{Q}^{\pitheta}(s, \cdot)}}^2} \leq \frac{4 \, A^{3/2}  \, S^{1/2}}{(1 - \gamma)^4 \, \Delta^2}  \tnorm{\frac{\partial V^{\pitheta}(\rho)}{\partial \theta}}
\end{equation}
where $\Delta := \min_s \min_{a \neq a'} \abs*{Q^{\pitheta}(s, a) - Q^{\pitheta}(s, a')}$.
\end{theorem}
\begin{proof}
In the tabular MDP setting we have
\begin{align}
    \normsq{\hgrad{\theta}} &= \sum_{s \in \gS} \sum_{a \in \gA} \frac{d^{\pitheta}_\rho(s)^2}{(1 - \gamma)^2} \, \pitheta(a \, | \, s)^2 \, \parens*{\hat{Q}^{\pitheta}(s, a) - \dpd{\pitheta(\cdot | s), \hat{Q}^{\pitheta}(s, \cdot)}}^2.
\end{align}
Let us first bound the RHS. For a fixed $s \in \gS$.
\begin{align}
\MoveEqLeft
\sum_{a \in \gA} \pitheta(a \, | \, s)^2 \, \parens*{\hat{Q}^{\pitheta}(s, a) - \dpd{\pitheta(\cdot \, | \, s), \hat{Q}^{\pitheta}(s, \cdot)}}^2\\ 
&= \sum_{a \in \gA}  \, \pitheta(a \, | \, s)^2 \, \left[ \frac{\indicator{a(s) = a}}{\pitheta(a \, | \, s)^2} \, Q^{\pitheta}(s, a)^2 - 2 \, \frac{\indicator{a(s) = a}}{\pit(a \, | \, s)} Q^{\pitheta}(s, a) \, \dpd{\pitheta(\cdot | s), \hat{Q}^{\pitheta}(s, \cdot)} \right. \notag \\ 
&\left. \qquad \qquad \qquad \qquad + \parens*{\dpd{\pitheta(\cdot | s), \hat{Q}^{\pitheta}(s, \cdot)}}^2 \right] \\
&= Q^{\pitheta}(s, a(s))^2 - 2 \, \pitheta(a(s) | s) \, Q^{\pitheta}(s, a(s))^2 + Q^{\pitheta}(s, a(s))^2 \, \sum_{a \in \gA} \pitheta(a |s)^2 \\
&= (1 - \pitheta(a(s) | s))^2 \, Q^{\pitheta}(s, a(s))^2 + Q^{\pitheta}(s, a(s))^2 \, \sum_{a \neq a(s)} \pitheta(a | s)^2 \\
&= \frac{1}{(1 - \gamma)^2} (1 - \pitheta(a(s) | s))^2  +  \sum_{a \neq a(s)} \pitheta(a | s)^2 \tag{$Q^{\pitheta}(s, a) \leq \frac{1}{1 - \gamma}$} \\
&\leq \frac{1}{(1 - \gamma)^2} \parens*{(1 - \pitheta(a(s) | s))^2  +  \parens*{\sum_{a \neq a(s)} \pitheta(a | s)}^2} \tag{$\tnorm{x} \leq \| x \|_1$}\\
&= \frac{2}{(1 - \gamma)^2} (1 - \pitheta(a(s) | s))^2
\end{align}
Accounting for every $s \in \gS$,
\begin{align}
\implies 
\normsq{\hgrad{\theta}} &\leq \frac{2}{(1 - \gamma)^4} \sum_{s \in \gS}  \bracks*{d^{\pitheta}_\rho(s)}^2 (1 - \pitheta(a(s) | s))^2 \label{eq:sgc_mdp_ub_0}
\end{align}

In \cref{alg:softmax_pg_general_on_policy_stochastic_gradient}, the only source of stochasticy is from sampling $a(s) \sim \pitheta(\cdot | s)$ for each $s \in \gS$. Therefore 
\begin{equation}
    \E \bracks*{\tnorm{\hgrad{\theta}}} = \E_{a_1 \sim \pitheta(\cdot | s_1)}\bracks*{
    \E_{a_2 \sim \pitheta(\cdot | s_2)}\bracks*{\dots \E_{a_{S} \sim \pitheta(\cdot | s_{S})}\bracks*{\normsq{\hgrad{\theta}}}}}.
\end{equation}

Let us first consider $\E_{a_1 \sim \pitheta(\cdot | s_1)}\bracks*{\normsq{\hgrad{\theta}}}$. By \cref{eq:sgc_mdp_ub_0}
\begin{align}
\MoveEqLeft
\E_{a_1 \sim \pitheta(\cdot | s_1)}\bracks*{\normsq{\hgrad{\tht}}} \\
&\leq \frac{2}{(1 - \gamma)^4} \sum_{a_1 \in \gA} \pitheta(a_1 | s_1) \bracks*{
        \bracks*{d^{\pitheta}_\rho(s_1)}^2 \, (1 - \pitheta(a_1 | s_1))^2 + 
        \sum_{s \neq s_1} \bracks*{d^{\pitheta}_\rho(s)}^2 \, (1 - \pitheta(a(s) | s))^2
    } \\
&= \frac{2}{(1 - \gamma)^4} \\ &\left[
    \underbrace{\bracks*{d^{\pitheta}_\rho(s_1)}^2 \, \sum_{a_1 \in \gA} \pitheta(a_1 | s_1) \, (1 - \pitheta(a_1 | s_1))^2}_{:= C_{s_1}} + \underbrace{\sum_{a_1 \in \gA} \pitheta(a_1 | s_1)}_{ = 1}\sum_{s \neq s_1} \bracks*{d^{\pitheta}_\rho(s)}^2 \, (1 - \pitheta(a(s) | s))^2
    \right]\\
&= \frac{2}{(1  - \gamma)^4} \bracks*{C_{s_1} + \sum_{s \neq s_1} \bracks*{d^{\pitheta}_\rho(s)}^2 \, (1 - \pitheta(a(s) | s))^2}.
\end{align}
Next let us consider $\E_{a_2 \sim \pitheta(\cdot | s_2)} \E_{a_1 \sim \pitheta(\cdot | s_1)}\bracks*{\normsq{\hgrad{\theta}}}$ and by the same argument
\begin{align}
\E_{a_2 \sim \pitheta(\cdot | s_2)} \E_{a_1 \sim \pitheta(\cdot | s_1)}\bracks*{\normsq{\hgrad{\theta}}} &\leq \frac{2}{(1  - \gamma)^4} \bracks*{C_{s_1}+C_{s_2}+\sum_{\substack{s \neq s_1 \\ s \neq s_2}} \bracks*{d^{\pitheta}_\rho(s)}^2 \, (1 - \pitheta(a(s) | s))^2}
\end{align}
Continuing in the same way for the remaining $s \in \gS$ we have
\begin{align}
    \E\bracks*{\normsq{\hgrad{\theta}}} &\leq \frac{2}{(1 - \gamma)^4} \sum_{s \in \gS} C_s \\
    &= \frac{2}{(1 - \gamma)^4} \sum_{s \in \gS} \bracks*{d^{\pitheta}_\rho(s)}^2 \, \sum_{a \in \gA} \pitheta(a | s) \, (1 - \pitheta(a | s))^2
    \intertext{Denote $k(s) := \argmax_{a \in \gA} \pitheta(a | s)$ as the action with the largest probability at state $s$}
    &= \frac{2}{(1 - \gamma)^4} \sum_{s \in \gS} \bracks*{d^{\pitheta}_\rho(s)}^2 \bracks*{\pitheta(k(s) | s) \, (1 - \pitheta(k(s) | s))^2 + \sum_{a \neq  k(s)} \, \pitheta(a | s) \, (1 - \pitheta(a | s))^2} \\
    &\leq \frac{2}{(1 - \gamma)^4} \sum_{s \in \gS} \bracks*{d^{\pitheta}_\rho(s)}^2 \, \bracks*{(1 - \pitheta(k(s) | s)) \, + \sum_{a \neq  k_t(s)}\pitheta(a(s | s)} \\
    &= \frac{4}{(1 - \gamma)^4} \sum_{s \in \gS} \, \bracks*{d^{\pitheta}_\rho(s)}^2 \, (1 - \pitheta(k(s) | s))  \tag{$\pitheta(a | s) \in [0, 1]$}\\
    \intertext{Since $d^{\pitheta}_\rho(s) \leq 1$ for all $s \in \gS$}
    &\leq \frac{4}{(1 - \gamma)^4} \sum_{s \in \gS} d^{\pitheta}_\rho(s) \, (1 - \pitheta(k(s) | s))
\end{align}
\begin{align}
\implies \E \bracks*{\normsq{\hgrad{\theta}}} &\leq \frac{4}{(1 - \gamma)^4} \sum_{s \in \gS} d^{\pitheta}_\rho(s) \, (1 - \pit(k(s) | s)).  \label{eq:sgc_mdp_ub_1}
\end{align}

Now we lower bound $\normsq{\frac{V^{\pitheta}(\rho)}{\partial \theta}}$
\begin{align}
\MoveEqLeft
\normsq{\frac{V^{\pitheta}(\rho)}{\partial \theta}} \\
&= \frac{1}{(1 - \gamma)^2} \, \parens*{\sum_{s \in \gS} \sum_{a \in \gA} d^{\pitheta}_\rho(s)^2 \, \pitheta(a | s)^2 \, A^{\pitheta}(s, a)^2} \\
\intertext{Multiplying and dividing by $\sum_{(s', a')} A^{\pitheta}(s, a)^2$}
&= \frac{1}{(1 - \gamma)^2}
    \parens*{
    \sum_{s' \in \gS}\sum_{a' \in \gA} A^{\pitheta}(s', a')^2 \, 
    \sum_{s \in \gS}\sum_{a \in \gA} \parens*{\underbrace{d^{\pitheta}_\rho(s) \, \pitheta(a | s)}_{:= w(s,a)}}^2 \underbrace{  \frac{(A^{\pitheta}(s, a))^2}{\sum_{(s',  a') \in \gS \times \gA}A^{\pitheta}(s', a')^2}}_{:= p(s, a)}
    }
    \\
    \intertext{Since $p(s, a) \geq 0$ and $\sum_{s, a} p(s, a) = 1$, using Jensen's inequality, \newline $\sum_{s, a} w(s,a)^2 \, p(s,a) \geq (\sum_{s, a} w(s, a) \, p(s, a))^2$}
    &\geq \frac{1}{(1 - \gamma)^2}
    \parens*{
    \sum_{s' \in \gS}\sum_{a' \in \gA}A^{\pitheta}(s', a')^2  \,
        \bracks*{
        \sum_{s \in \gS}\sum_{a \in \gA} d^{\pitheta}_\rho(s) \, \pit(a | s) \, \frac{A^{\pitheta}(s, a)^2}{\sum_{(s', a') \in \gS \times \gA}A^{\pitheta}(s', a')^2}
        }^2
    } \\
    &= \frac{1}{(1 - \gamma)^2}
    \parens*{
    \frac{1}{\sum_{(s', a') \in \gS \times \gA}A^{\pitheta}(s', a')^2}
        \bracks*{
        \sum_{s \in \gS}\sum_{a \in \gA} d^{\pitheta}_\rho(s) \, \pitheta(a | s) \, A^{\pitheta}(s, a)^2
        }^2
    }
\end{align}
\begin{align}
    \intertext{Since $A^{\pitheta}(s, a) \leq \frac{1}{1 - \gamma}$, $\frac{1}{\sum_{(s', a')} A^{\pitheta}(s', a')^2} \geq \frac{(1 - \gamma)^2}{S \, A}$}
    \implies \normsq{\frac{\partial V^{\pit}(\rho)}{\partial \theta}} &\geq \frac{1}{S \, A} \,
        \bracks*{
        \sum_{s \in \gS} \sum_{a \in \gA}d^{\pitheta}_\rho(s) \, \pitheta(a | s) A^{\pitheta}(s, a)^2
        }^2 \\
    \implies  \sum_{s \in \gS}\sum_{a \in \gA}d^{\pitheta}_\rho(s) \, \pitheta(a | s) \, A^{\pitheta}(s, a)^2 &\leq 
    \sqrt{S \, A} \, \tnorm{\frac{\partial V^{\pitheta}(\rho)}{\partial \theta}}.
\label{eq:sgc_mdp_ub_3}
\end{align}
To connect \cref{eq:sgc_mdp_ub_1} and \cref{eq:sgc_mdp_ub_3} for a fixed $s \in \gS$ 
\begin{align}
\MoveEqLeft
    \sum_{a \in \gA} \pitheta(a | s) \, A^{\pitheta}(s, a)^2  \notag \\
    &=  \sum_{a \in \gA} \pitheta(a | s) \, (Q^{\pitheta}(s, a) - V^{\pitheta}(s))^2  \\
    &=  \sum_{a \in \gA} \pitheta(a | s) \, \bracks*{Q^{\pitheta}(s, a)^2 - 2 \, V^{\pitheta}(s) \, Q^{\pitheta}(s, a) + V^{\pitheta}(s)^2 } \\   
    &=  \sum_{a \in \gA} \pitheta(a | s) \, Q^{\pitheta}(s, a)^2 -  2 \, V^{\pitheta}(s) \, \underbrace{\sum_{a \in \gA} \pitheta( a | s) \, Q^{\pitheta}(s, a)}_{=V^{\pitheta}(s)} + V^{\pitheta}(s)^2 \underbrace{\sum_{a \in \gA} \pitheta(a | s)}_{=1}  \\  
    &= \sum_{a \in \gA} \pitheta(a | s) Q^{\pitheta}(s, a)^2 - \bracks*{\sum_{a \in \gA} \pitheta(a | s) \, Q^{\pitheta}(s, a)}^2 \\
    \intertext{Recall $k(s) := \argmax_{a \in \gA} \pitheta(a | s)$, by \cref{lemma:lecture_11_slide_4},}
    &\geq \pitheta(k(s) | s) \, \sum_{a \neq k(s)} \pitheta(k(s) | s) \,  (Q^{\pitheta}(s, k(s)) - Q^{\pitheta}(s, a))^2
    \intertext{Let $\Delta_s := \min_{a \neq a'}\abs{Q^{\pitheta}(s, a) - Q^{\pitheta}(s, a')}$ and since $\pitheta(k(s) | s) \geq \frac{1}{A}$,}
    &\geq (1 - \pitheta(k(s) | s) \, \frac{\Delta_s^2}{A} \\
    \intertext{Let $\Delta := \min_s \Delta_s$}
    &\geq (1 - \pitheta(k(s) | s) \, \frac{\Delta^2}{A}
\end{align}
\begin{align}
\implies (1 - \pitheta(k(s) | s) &\leq \frac{A}{\Delta^2}\sum_{a \in \gA} \pitheta(a | s) \, A^{\pitheta}(s, a)^2 \label{eq:sgc_mdp_ub_2}
\end{align}
Putting everything together, by \cref{eq:sgc_mdp_ub_1}
\begin{align}
\E\bracks*{\normsq{\hgrad{\theta}}} &\leq  \frac{4}{(1 - \gamma)^4} \sum_{s} d^{\pitheta}_\rho(s) \, (1 - \pitheta(k(s) \, | \, s))
\intertext{By \cref{eq:sgc_mdp_ub_2}}
&\leq \frac{4 \, A}{(1 - \gamma)^4 \, \Delta^2} \sum_{s \in \gS} \sum_{a \in \gA} d^{\pitheta}_\rho(s) \, \pitheta(a \, | \, s) \, A^{\pitheta}(s, a)^2
\intertext{By \cref{eq:sgc_mdp_ub_3}}
 &\leq \frac{4 \, A^{3/2} \, S^{1/2}}{(1 - \gamma)^4 \, \Delta^2} \,   \tnorm{\frac{\partial V^{\pitheta}(\rho)}{\partial \theta}}. 
\end{align}
\end{proof}

\subsection{Additional Lemmas}
\begin{thmbox}
\begin{lemma}\label{lemma:bound_gtheta_zeta}
Assuming that $f$ is $L_1$-non-uniform smooth and the stochastic gradient is bounded, i.e. $\norm{\hgrad{\thetat}} \leq B$, using \cref{update:spg} with $\etat \in (0, \frac{1}{L_1 \, B})$ we have,
\begin{equation}
    \abs*{f(\thetatt) - f(\thetat) - \dpd{\grad{\thetat}, \thetatt - \thetat}} \leq \frac{1}{2} \,\frac{L_1 \, \norm{\grad{\thetat}}}{1 - L_1 \, B \, \etat} \,  \normsq{\thetatt - \thetat}.
\end{equation}
\end{lemma}
\end{thmbox}
\begin{proof}
    Following \citep[Lemma 4.2]{mei2023stochastic}, denote $\theta_\zeta := \thetat + \zeta \, (\thetatt - \thetat)$ for some $\zeta \in [0, 1]$. 
    According to Taylor's theorem, we have
    \begin{align}
        \abs*{f(\thetatt) - f(\thetat) - \dpd{\grad{\thetat}, \thetatt - \thetat}} &= \frac{1}{2} \, \abs*{(\thetatt - \thetat)^\top \nabla^2 f(\theta_\zeta) \, (\thetatt - \thetat)}
        \intertext{Assuming $f$ is $L_1$ non-uniform smooth}
        &\leq \frac{L_1\, \norm{\grad{\theta_\zeta}}}{2}  \, \normsq{\thetatt - \thetat} \label{eq:taylor_start}.
    \end{align}
    Denote $\theta_{\zeta_1} := \thetat + \zeta_1 \, (\theta_\zeta - \thetat)$ for some $\zeta_1 \in [0, 1]$. By the fundamental theorem of calculus,
    \begin{align}
       \norm{\grad{\theta_\zeta} - \grad{\thetat}} &= \norm*{\int_0^1 \dpd{\nabla^2 f(\theta_{\zeta_1}), \theta_\zeta - \thetat} d\zeta_1} \\
       \intertext{Using Cauchy-Schwarz}
       &\leq \int_0^1 \norm*{\nabla^2 f(\theta_{\zeta_1})} \, \norm{\theta_\zeta - \theta_t}  d \zeta_1 
       \intertext{Since $f$ is $L_1$-non-uniform smooth}
       &\leq \int_0^1 L_1 \, \norm*{\nabla f(\theta_{\zeta_1})} \, \norm{\theta_\zeta - \theta_t} d \zeta_1  \\
       &= \int_0^1 L_1 \, \norm*{\nabla f(\theta_{\zeta_1})} \, \zeta \, \norm{\thetatt - \theta_t} d \zeta_1  \tag{$\theta_\zeta := \theta_t + \zeta \, (\thetatt - \thetat)$}
       \intertext{Since $\zeta \in [0, 1]$ and using \cref{update:spg}, $\thetatt = \thetat + \etat \hgrad{\thetat}$}
       \implies \norm{\grad{\theta_\zeta} - \grad{\thetat}} &\leq L_1 \etat \, \norm{\hgrad{\thetat}} \, \int_0^1  \norm*{\nabla f(\theta_{\zeta_1})}  \,  d \zeta_1  \label{eq:taylor_ub}
    \end{align}
    Therefore, we have
    \begin{align}
    \norm{\grad{\theta_\zeta}} &= \norm{\grad{\theta} + \grad{\theta_\zeta} - \grad{\theta}} \\
    \intertext{Using triangle inequality}
    &\leq \norm{\grad{\thetat}} + \norm{\grad{\theta_\zeta} - \grad{\thetat}} \\
    \intertext{By \cref{eq:taylor_ub}}
    \implies \norm{\grad{\theta_\zeta}} &\leq \norm{\grad{\thetat}} + L_1 \, \etat \, \norm{\hgrad{\thetat}} \, \int_0^1  \norm*{\nabla f(\theta_{\zeta_1})}  \,  d \zeta_1 \label{eq:taylor_ub_1}
    \end{align}
    Denote $\theta_{\zeta_1} := \thetat + \zeta_2 \, (\theta_{\zeta_1} - \thetat)$ with $\theta_{\zeta_2} \in [0, 1]$. Using similar calculations when deriving \cref{eq:taylor_ub},
    \begin{equation}
        \norm{\grad{\theta_{\zeta_1}}} \leq \norm{\grad{\thetat}} + L_1 \, \etat \, \norm{\hgrad{\thetat}} \, \int_0^1 \norm{\grad{\theta_{\zeta}}} \, d \zeta_2 \label{eq:taylor_ub_2}
    \end{equation}
    Putting \cref{eq:taylor_ub_1} and \cref{eq:taylor_ub_1} together,
    \begin{equation}
        \norm{\grad{\theta_\zeta}} \leq \parens*{1 + L_1 \, \etat \, \norm{\hgrad{\thetat}}} \, \norm{\grad{\thetat}} + \parens*{L_1 \, \etat \, \norm{\hgrad{\thetat}}}^2 \, \int_0^1 \int_0^1 \norm{\grad{\theta_{\zeta_2}}} d \zeta_2 \, d \zeta_1 \label{eq:talor_recurse}
    \end{equation}
    Using \cref{eq:talor_recurse} and continuing in the same way for $\zeta_i$ as $i \rightarrow \infty$
    \begin{equation}
        \norm{\grad{\theta_\zeta}} \leq \underbrace{\sum_{i=0}^\infty \parens*{L_1 \, \etat \, \norm{\hgrad{\thetat}}}^i}_{\heartsuit} \, \norm{\grad{\thetat}}.
    \end{equation}
    To ensure that $\heartsuit$ is finite, we require that $L_1 \, \etat \, \norm{\hgrad{\thetat}} < 1$. Assuming $\norm{\hgrad{\thetat}} \leq B$ for all $t$
    \begin{align}
        L_1 \, \etat \, \norm{\grad{\thetat}} &\leq L_1 \, B \, \etat  < 1
        \implies \etat < \frac{1}{L_1 \, B}
    \end{align}
    For $\etat \in \parens*{0, \frac{1}{L_1 \, B}}$, summing the geometric series
    \begin{align}
        \norm{\grad{\theta_\zeta}} \leq \frac{\norm{\grad{\thetat}}}{1 - L_1 \, B \, \etat} \,  \label{eq:taylor_end}.
    \end{align}
    Putting \cref{eq:taylor_start} and \cref{eq:taylor_end} together, for $\etat \in \parens*{0, \frac{1}{L_1, B}}$ we have
    \begin{equation}
       \abs*{f(\thetatt) - f(\thetat) - \dpd{\grad{\thetat}, \thetatt - \thetat}} \leq \frac{1}{2} \, \frac{L_1 \, \norm{\grad{\tht}}}{ 1 - L_1 \, B \, \etat} \, \normsq{\thetatt - \thetat}.
    \end{equation}
\end{proof}

\begin{thmbox}
\begin{lemma}[Lemma 5 in \citep{vaswani2022towards}]\label{lemma:alpha_bound}
\begin{equation}
    \frac{\alpha^{T+1}}{1 - \alpha} \leq \frac{2 \beta}{\ln(\nicefrac{T}{\beta})}
\end{equation}
\end{lemma}
\end{thmbox}

\begin{thmbox}
\begin{restatable}[Lemma 4.3 in \citep{mei2023stochastic}]{lemma}{lemmabanditsgc}
\label{lemma:bandit_sgc} 
Using~\cref{update:spg}, we have for all $t \geq 1$,
\begin{equation}
    \E_t\bracks*{\normsq{\frac{d \dpd{\pi_{\thetat}, \hat{r}_t}}{d \thetat}}} \leq \frac{8 \, A^{3/2}}{\Delta^2} \, \norm*{\frac{d \dpd{\pi_{\thetat}, \hat{r}_t}}{d \thetat}}_2
\end{equation}
where $\Delta := \min_{a \neq a'}\abs{r(a) - r(a')}$.
\end{restatable}
\end{thmbox}

\begin{thmbox}
\begin{lemma}[Lemma 17 in \citep{vaswani2022towards}]
\label{lemma:xgamma_bound}
For all $x, \gamma > 0$,
\begin{equation}
\exp(-x) \leq \parens*{\frac{\gamma}{e x}}^\gamma    
\end{equation}
\end{lemma}
\end{thmbox}

\begin{thmbox}
\begin{lemma}\label{lemma:lecture_11_slide_4}
Let $p, b \in \R^K$ such that $p_1 \geq p_2 \geq \dots \geq p_K \geq 0$, $\sum_{i=1}^K p_i = 1$ and $b_i \geq 0$ for all $i$ then    
\begin{equation}
   \sum_{i=1}^K p_i \, b_i^2 - \bracks*{\sum_{i=1}^K p_i \, b_i}^2   \geq p_1 \, \sum_{j=2}^K p_j \, [b_i - b_j]^2
\end{equation}
\end{lemma}
\end{thmbox}
\begin{proof}
\begin{align}
\sum_{i=1}^K p_i \, b_i^2 - \bracks*{\sum_{i=1}^K p_i \, b_i}^2 &=  \sum_{i=1}^K p_i \, b_i^2 - \sum_{i=1}^K p_i^2 \, b_i^2 - 2 \, \sum_{i=1}^{K-1} p_i \, r_i \, \sum_{j=i+1}^K p_j \, r_j \\
&= \sum_{i=1}^K (p_i \, b_i^2 - p_i^2 \, b_i^2) - 2 \, \sum_{i=1}^{K-1} p_i \, r_i \, \sum_{j=i+1}^K p_j \, r_j \\
&= \sum_{i=1}^K p_i \, b_i^2 \, (1 - p_i) - 2 \, \sum_{i=1}^{K-1} p_i \, r_i \, \sum_{j=i+1}^K p_j \, r_j \\
&= \sum_{i=1}^K \underbrace{p_i}_{x_i} \, \underbrace{b_i^2}_{ y_i} \, \sum_{i=1, j \neq i}^K \underbrace{p_j}_{ x_j}   - 2 \, \sum_{i=1}^{K-1} p_i \, r_i \, \sum_{j=i+1}^K p_j \, r_j \tag{$p_i = 1 - \sum_{j \neq 1} p_j$} \\
\intertext{For any $x_i, y_i$, $\sum_{i=1}^K x_i \, y_i \, \sum_{j=1, j \neq i}^K x_j = \sum_{i=1}^{K-1} x_i \, \sum_{j=i+1}^K x_j \, [y_i + y_j]$}
&= \sum_{i=1}^{K-1} p_i \, \sum_{j=i+1}^K p_j\, [b_i^2 + b_j^2] - 2 \, \sum_{i=1}^{K-1} p_i \, b_i \, \sum_{j=i+1}^K p_j \, b_j  \\
&= \sum_{i=1}^{K-1} p_i \, \sum_{j=i+1}^K p_j \bracks*{b_i^2 - 2 b_i \, b_j + b_j^2} \\
&= \sum_{i=1}^{K-1} p_i \, \sum_{j=i+1}^K p_j \, [b_i - b_j]^2
\intertext{Discarding extra terms since $p_2 \geq \dots \geq p_{K-1} \geq 0$,}
&\geq p_1 \sum_{j=2}^K p_j [b_i - b_j]^2.
\end{align} 
\end{proof}

\begin{thmbox}
\begin{lemma}\label{lemma:bandit_sg_bounded}
In the bandit setting,     
\begin{equation}
    \norm*{\frac{d \dpd{\pi_\theta, \hat{r}}}{d \theta}} \leq \sqrt{2}.
\end{equation}
\end{lemma}
\end{thmbox}
\begin{proof}
    Follows from \citet[Equation 55]{mei2023stochastic}.
\end{proof}

\begin{thmbox}
\begin{lemma}\label{lemma:mdp_sg_bounded}
    In the tabular MDP setting, 
    \begin{equation}
    \norm*{\sum_{s \in \gS} \sum_{a \in \gA} \frac{d^{\pitheta}_\rho(s)^2}{(1 - \gamma)^2} \pitheta(a |s)^2 \parens*{\hat{Q}^{\pitheta}(s, a) - \dpd{\pitheta(\cdot | s), \hat{Q}^{\pitheta}(s, \cdot)}}^2} \leq \frac{\sqrt{2 \, S}}{(1 - \gamma)^2}.
    \end{equation}
\end{lemma}
\end{thmbox}
\begin{proof}
Follows from \cref{eq:sgc_mdp_ub_0}.
\end{proof}

\section{Policy Gradient with Entropy Regularization}\label{appendix:entropy}
We will next consider adding entropy regularization to the objective in the exact and stochastic settings. Entropy regularization RL, also known as maximum entropy RL, uses entropy regularization to promote action diversity and prevent premature convergence to a deterministic policy \citep{williams1992simple, haarnoja2018soft}. While it is widely believed to help with exploration, the addition of entropy regularization results in a smoother optimization landscape, enabling PG methods to escape flat regions within the optimization  landscape~\citep{ahmed2019understanding}. For example in the bandits setting, flat regions occur when a policy commits to an arm. \citet{mei2020global} showed entropy regularization helps escaping these regions when starting from a ``bad'' initialization, i.e. the initial policy selects an sub-optimal arm with high probability.

In the exact setting, where the full gradient can be computed, \citet{mei2020global}
showed softmax PG with entropy regularization obtains a fast $\gO(\log(\nicefrac{1}{\eps}))$ rate to a biased $\eps$-optimal policy. The resulting optimal policy is biased since the presence of entropy prevents convergence to a deterministic policy. Additionally, in the same setting, \citet{cen2022fast} showed NPG with entropy regularization achieves the same $\gO(\log(\nicefrac{1}{\eps}))$ convergence rate to a biased $\eps$-optimal policy. To ensure that the resulting optimal policy is unbiased, the strength of the entropy regularization term must be decayed or removed. \citet{mei2020global} introduced a two-stage approach to obtain the optimal policy when using softmax PG with entropy regularization. In the first stage, entropy regularization is used to obtain fast convergence close to the optimal policy. In the second stage, the regularizer is removed to guarantee convergence to the optimal policy. Unfortunately, the final convergence rate is $\gO(\nicefrac{1}{\eps})$ which matches the same rate as softmax PG. Additionally, in order to transition from the first to the second stage, the reward gap is needed making the resulting algorithm impractical.

In the stochastic setting, where the value function must be approximated, 
\citet{ding2021beyond} introduced a two-stage approach for stochastic softmax PG with entropy regularization. Instead of modifying the strength of the entropy regularizer across stages, the batch size is modified. The resulting algorithm requires $\gO(\nicefrac{1}{\epsilon})$ iterations at the second stage and $\tilde{\gO}(\nicefrac{1}{\eps^2})$ samples to converge to an biased $\eps$-optimal policy. The method allows for global convergence with arbitrary initiation. However, the strength of the entropy regularizer is not decayed, preventing convergence to the optimal policy. Additionally, the biased optimal policy to set the algorithm hyper-parameters making the resulting algorithm redundant. Moreover, in the stochastic setting with access to a generative model, using NPG with entropy regularization, \citet{cen2022fast} achieved a linear rate of convergence to a biased optimal policy with a $\tilde{\gO}(\nicefrac{1}{\epsilon^2})$ sample complexity.

In the following sections, we will present a multi-stage algorithm that iteratively reduces the strength of the entropy regularization term. This method obtains convergence to the optimal policy while eliminating the reliance on unknown quantities compared to prior  work. In \cref{appendix:pg_entropy_background} we first state how the objective's functional property changes when entropy regularization is added. In \cref{appendix:pg_entropy} we present the multi-stage algorithm in the exact setting and the algorithm achieves an $\gO(\nicefrac{1}{\epsilon^p})$ rate. Here, $p$ relies on the estimation of the lower bound of the non-uniform \L ojsiewciz condition of the entropy regularized objective. Next in \cref{appendix:spg_entropy}, we extend the same multi-stage algorithm in the stochastic setting with exponentially decreasing step-sizes to obtain an also $\gO(\nicefrac{1}{\eps^{2p + 1}})$ rate to the optimal policy. Finally, in \cref{appendix:pg_entropy_experiments} we compare the proposed our multi-stage algorithm to prior PG methods without entropy regularization and show that the multi-stage algorithm helps escape flat regions within the optimization landscape.

\subsection{Problem Setup}\label{appendix:pg_entropy_background}
Following \cref{sec:background}, for a policy $\pi$, the \textit{entropy regularized action-value function} is defined as $\tilde{Q}_\tau^{\pi}(s, a) := \E[\sum_{t=0}^\infty \gamma^t (r(s,a) - \tau \log \pi)]$ and the \textit{entropy regularized value function} is defined as 
$\tilde{V}_\tau^{\pi}(s) := \E_{a \sim \pi}[\tilde{Q}_\tau^{\pi}(s, a)](s)$ . The \textit{entropy regularized advantage function} is defined as  $\tilde{A}_\tau^{\pi}(s,a) := \tilde{Q}_\tau^{\pi}(s, a) - \tau \, \log \pi(a | s) - \tilde{V}_\tau^{\pi}(s)$.

Additionally, let $\ftau(\theta) := f(\theta) + \tau \, \Lambda(\pitheta)$ denote the entropy regularized objective, where $\Lambda(\pitheta)$ is the ``discounted entropy'' for a policy $\pitheta$ and $\tau \geq 0$  is the ``temperature'' or strength of the entropy regularization.   For a fixed $\tau$, $\ftau$ is $L^\tau$-uniform smooth and note that the smoothness now depends on $\tau$. Furthermore, $\ftau$ satisfies a non-uniform \L ojasiewciz condition with $C_\tau(\theta)$ and $\xi = \nicefrac{1}{2}$. Compared to $f$, whose non-uniform \L ojasiewciz degree is $\xi = 0$ (refer to \cref{table:c_theta}), the increase to $\xi = \nicefrac{1}{2}$ allows for faster convergence.
\cref{table:entropy} summarizes the entropy regularizer, uniform smoothness and non-uniform \L ojsaiewciz properties for the bandit and general MDP settings with entropy regularization. Finally, we will denote the maximum value of the regularized objective function as $f^{*_\tau} := f^\tau(\theta_\tau^*)$, where $\theta_\tau^* := \argmax_\theta f^\tau(\theta)$. 

\begin{table}[h]
\scalebox{0.95}{
\centering
\setlength\tabcolsep{2pt}
\begin{tabular}{|c|c|c|c|c|}
\hline
Setting&  
$\Lambda(\pitheta)$&
$[\nabla f^\tau(\theta)]_{s, a}$&
$L^\tau$&
$C_\tau(\theta)$
\\ \hline
Bandits &
$- \langle \pitheta, \log \pitheta \rangle$&
$\pitheta(a) \, [r(a) - \dpd{\pitheta,\, r - \tau \log \pitheta}]$&
$5/2 + 5 \, \tau \, (1 +  \log A)$&
$\sqrt{2 \tau} \min_a \pitheta(a)$ 
\\\hline
MDP & 
$\mathbb{H}(\pitheta)$& 
$\frac{d^{\pitheta}(s) \, \pitheta(a | s) \, \tilde{A}^{\pitheta}(s, a)}{1 - \gamma}$&
$\frac{8 + \tau \, (4 + 8 \log A)}{(1 - \gamma)^3}$ &
$\frac{\sqrt{\tau} \min_s \sqrt{\rho(s)} \, \min_{s, a} \pitheta(a |s)}{S \, \supnorm{\frac{d^{\pi^{*_\tau}}_\rho}{ d^{\pitheta}_\rho}}^{1/2}}$
\\ 
& & & & \\ \hline
\end{tabular}
}
\caption{Entropy regularizer, uniform smoothness and non-uniform \L ojasiewciz condition with $\xi = \nicefrac{1}{2}$ for bandits and general tabular MDPs setting with entropy regularization. Here, $\mathbb{H}(\pitheta) := \E \bracks*{\sum_{t=0}^\infty -\gamma^t \log \pitheta(a_t | s_t)}$.  
}
\label{table:entropy}
\end{table}
With the above properties of $\ftau$, we next present how to principally decay $\tau$ for softmax PG with entropy regularization to obtain convergence to the optimal policy.

\subsection{Exact Setting}\label{appendix:pg_entropy}
We first consider the exact setting as a test bed to analyze how to decay $\tau$ to obtain convergence to the optimal policy. Recall that for a constant $\tau > 0$, softmax PG with entropy regularization is unable to converge to the optimal policy since the regularizer prevents the final policy from becoming deterministic. Softmax PG with entropy regularization has the following update:
\vspace{0.9ex}
\begin{update}(Softmax PG with Entropy Regularization, True Gradient)\label{update:entropy_pg}
$\theta_{t+1} = \theta_t + \etat \nabla \ftau(\thetat)$.
\end{update}
Refer to \cref{table:entropy} for the entropy regularized policy gradient $\nabla \ftau(\theta)$ in both the bandits and the general MDP cases.
In this setting, \citet{mei2020global} prove that softmax PG with entropy regularization converges to a biased optimal policy at an $\gO(\log\nicefrac{1}{\eps})$ rate when using a fixed step-size of $\etat = \eta = \frac{1}{L^\tau}$. The optimal policy is biased since $\tau > 0$ is fixed. In order for entropy regularized objective to converge to the globally optimal policy, $\tau \rightarrow 0$ is required. In the bandits setting,~\citet{mei2020global} proposed a two-stage approach to decay $\tau$ to obtain global convergence. A fixed $\tau > 0$ is used in the first stage but is then set to be $0$ in the second stage. However, the resulting algorithm requires knowledge of the reward gap $\Delta := \max_{a^* \neq a} r(a^*) - r(a)$ in order to transition from the first stage to the second stage, rendering the method to be impractical.
Additionally, ~\citet{mei2020global} proposed an additional approach by allowing $\tau$ be a function of $t$ and slowly decreasing $\tau_t$ over time. This approach also obtain convergence to the global optimal policy. However, it required $\tau_t \propto \Delta$ and knowledge of the reward gap were again needed. Moreover, the final convergence rate to the optimal policy could not be established since it could not be proved that $\inf_{t \geq 1} C_\tau(\thetat) > 0$.

For example, in the bandits setting (refer to \cref{table:entropy}) $C_\tau(\thetat) := \sqrt{2 \tau} \min_a \pit(a)$. In order for $\pit \rightarrow \pi^*$, we must have $\min_a \pit(a) \rightarrow 0$. However, in order to guarantee convergence when $\tau > 0$, we also require $\inf_{t \geq 0} \min_a \pit(a) > 0$. We conjecture that the non-uniform \L ojasiewciz condition bound is loose which results in a pessimistic bound involving $\min_a \pitheta(a)$. We will make the benign assumption that $\ftau$ satisfies the following non-uniform \L ojasiewciz condition with $\xi = \nicefrac{1}{2}$ such that  $\mu := \inf_{t \geq 0} [C_\tau(\thetat)]^2 = \tau^p \, B_1$  for constants $p \geq 1$ and $B_1 > 0$.
\vspace{0.5ex}
\begin{restatable}{assumption}{assumptionnonuni} \label{assumption:one}
$\ftau$ satisfies the non-uniform \L ojasiewciz condition for some $C_\tau(\theta)$ and $\xi = \frac{1}{2}$ such that $\mu := \inf_{t \geq 1} [C_\tau(\thetat)]^2 = \tau^p \, B_1$ for constants $p \geq 1$ and $B_1 > 0$.
\end{restatable}
Here we will assume the next worst dependence, which is having a polynomial dependence of $\tau$ for $\mu = \tau^p \, B_1$. Recall that $f$ has a non-uniform \L ojasiewciz condition with degree $\xi = 0$ and in the bandit setting  $C(\theta) = \pitheta(a^*)$. We conjecture that as $\tau \rightarrow 0$, we switch from the non-uniform \L ojasiewciz condition with degree $\xi = \nicefrac{1}{2}$ to degree $\xi = 0$. We leave the investigate of how these two conditions interpolate as future work.

Under \Cref{assumption:one}, we propose a multi-stage algorithm (\cref{algorithm:multi_stage_abstract}) to decay $\tau$ that can obtain $\eps$-convergence to the globally optimal policy without knowledge of the reward gap or any other problem-dependent parameters. \cref{algorithm:multi_stage_abstract} consists of multiple stages, where the temperature is decreased in each stage. Specifically, in stage $i$ uses $\tau_i$ for $T_i$ iterations and is halved i.e. $\tau_{i+1} = \frac{\tau_i}{2}$ in the following stage. To prove the method achieves global convergence, we first make the following assumptions to relate the entropy regularization objective $f^\tau$ to the unregularized objective $f$:
\vspace{0.3ex}
\begin{restatable}{assumption}{assumptionsmooth} \label{assumption:two}
$\ftau$ is $L^\tau$-smooth and $L^\tau \leq L^{\max}$, where $L^{\max} = \max_{\tau \in [0,1]} L^\tau$ is a constant. Furthermore, $L^\tau \geq L^{\min}$, where $L^{\min} = \min_{\tau \in [0,1]} L^\tau > 0$ is a constant.
\end{restatable}
\vspace{0.3ex}
\begin{restatable}{assumption}{assumptionbias} \label{assumption:three}
$f^* - f(\theta_\tau^*) \leq \tau B_2$, for a constant $B_2 > 0$.
\end{restatable}
\vspace{0.3ex}
\begin{restatable}{assumption}{assumptionupperbound} \label{assumption:four}
For a constant $B_3 > 0$, $f(\theta_\tau^*) - f(\theta) \leq f^{*_\tau} - f^\tau(\theta) + \tau B_3$.
\end{restatable}
\vspace{0.3ex}
\begin{restatable}{assumption}{assumptionshifttau} \label{assumption:five}
For $\tau_2 < \tau_1$ and a constant $B_4 > 0$$, f^{*_{\tau_2}} - f^{\tau_2}(\theta) \leq f^{*_{\tau_1}} - f^{\tau_1}(\theta) + \tau_1 B_4$.
\end{restatable}

The \Cref{assumption:two,assumption:three,assumption:four,assumption:five} hold for both the bandits and tabular MPD setting and are proved in \cref{appendix:verify_bandit} and \cref{appendix:verify_mdp} respectively. 

The following theorem (proved in \cref{appendix:pg_entropy_proofs}) shows that \cref{algorithm:multi_stage_abstract} converges to the unbiased optimal policy at an $\gO(\nicefrac{1}{\epsilon^p})$ rate.
\begin{restatable}{theorem}{theoremMultiStageAbstract} \label{theorem:multi_stage_abstract}
Assuming $\ftau$ and $f$ satisfy \Cref{assumption:one,assumption:two,assumption:three,assumption:four,assumption:five}, for a given $\eps \in (0, 1)$, \cref{algorithm:multi_stage_abstract} achieves $\epsilon$-suboptimality to the globally optimal after $T_\text{total} = \frac{4 \, L^{\text{max}} \, C_1^p}{\epsilon^p \, B_1} \log\left(2 \, \left(1 + B_4\right)\right)$ iterations, where $C_1 = \max\left(1, \frac{f^{*_{\tau_0}} - f^{\tau_0}(\theta_0)}{\tau_0}\right) + B_2 + B_3$. \end{restatable}
The resulting $\gO(\nicefrac{1}{\eps^p})$ rate depends on the constant $p$ in \Cref{assumption:one}. In the best case, when $p=1$, we recover an $\gO(\nicefrac{1}{\epsilon})$ convergence rate. Otherwise, if $p$ is large, we obtain a slower rate similar to the pessimistic analysis using $C_\tau(\theta) \propto \min_a \pitheta(a | s)$. Compared to \citet{mei2020global}, when using entropy regularization, our method is able to obtain $\eps$-convergence without requiring the knowledge of the reward gap.

We compare \cref{algorithm:multi_stage_abstract} (\texttt{PG-E-MS}) assuming $p = 1$ and $B_1=0.01$ to softmax PG (\texttt{PG}) with a fixed step-size of $\etat = \frac{1}{L} = \frac{2}{5}$ and softmax PG with entropy regularization (\texttt{PG-E}) with fixed $\tau=0.1$ and $\etat = \eta = \frac{1}{L^\tau} = \frac{2}{5 + 10 \, \tau(1  + \log A)}$ in the bandits setting with $A = 10$.  For \texttt{PG-E-MS}, $p$ and $B_1$ were selected by using grid-search on separate set of bandit instances.  We test the algorithms on bandit settings of varying difficulty based on their minimum reward gap $\underset{\bar{}}{\Delta} := \min_{a^* \neq a} r(a^*) - r(a)$. The easy, medium and hard environments correspond to $\underset{\bar{}}{\Delta} = 0.2, 0.1, 0.05$ respectively. The figure plots the average and 95\% confidence interval of 50 random mean reward vectors.
\begin{figure}[h]
  \begin{center}
    \includegraphics[width=\textwidth]{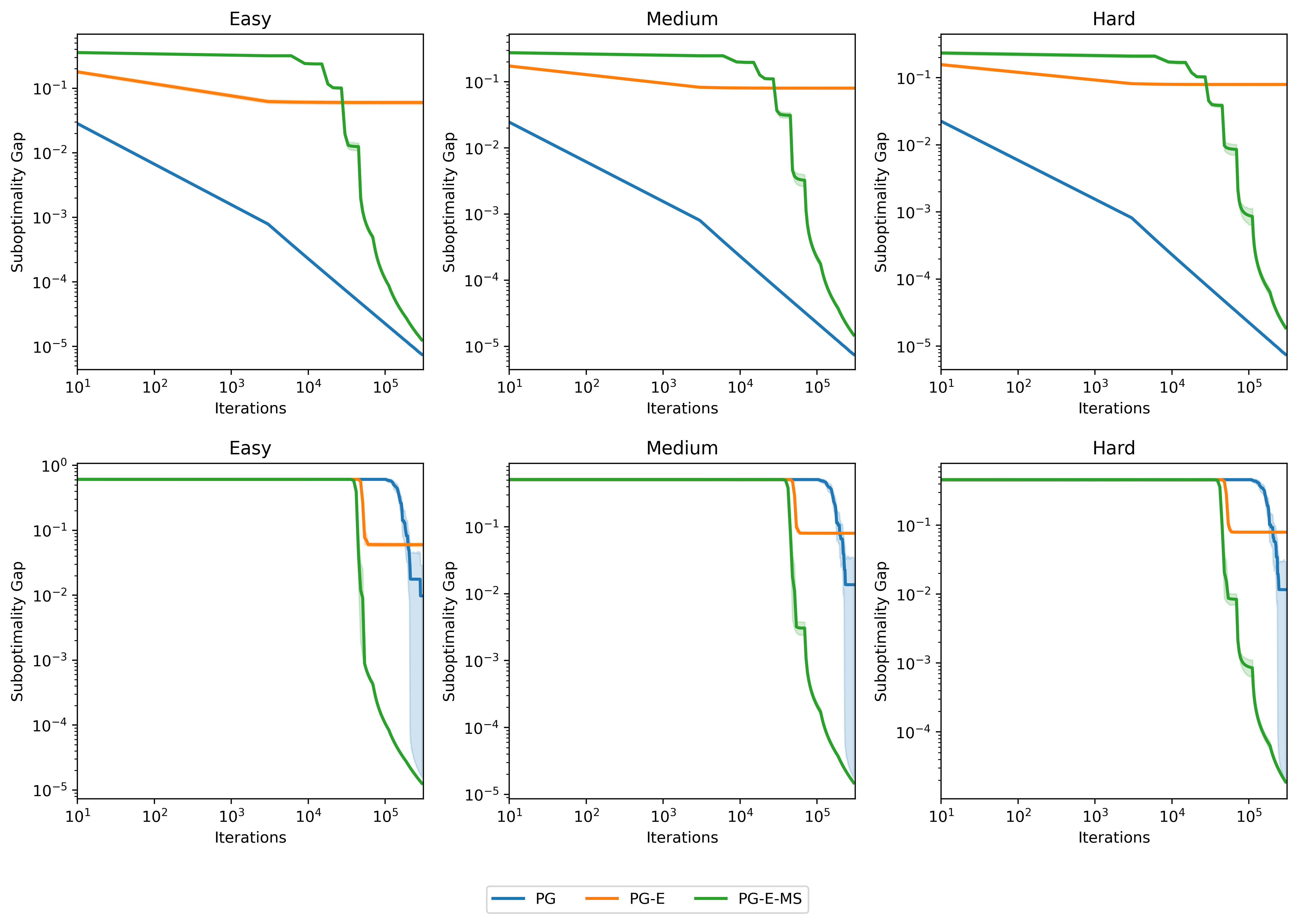}
    \caption{Sub-optimality gap across various environments and initializations. Top Row: the initial policy's parameters is uniform, i.e. $\theta_0(a) = 0 \quad \forall a$. Bottom Row: the initial policy's parameters is ``bad'', i.e. $\theta_0(a') = 12$ where $a' = \argmin_a r(a)$}
    \label{fig:det_entropy}
  \end{center}
\end{figure}

In \cref{fig:det_entropy}, \texttt{PG-E-MS} is able to converge to the optimal policy unlike \texttt{PG-E} since the temperature $\tau$ is decreasing.
Furthermore, under ``bad'' initialization, \texttt{PG-E-MS} outpreforms \texttt{PG} since the addition of entropy enables the method to able to escape the initial flat region. On the other hand, \texttt{PG-E} is able to escape the initial region quickly, but is unable to converge to the optimal policy since $\tau$ is fixed.

Additionally, from our experiments, we observe that the multi-stage algorithm with $p=1$ has a similar performance compared to softmax PG using uniform initialization. This confirms our theoretical observation that $p = 1$ results in a $\gO(\nicefrac{1}{\epsilon})$ convergence rate. We additionally investigated how entropy regularization can help when starting with a ``bad'' initialization. In this case, the worst arm has a high probability of getting chosen, which results in a flat optimization landscape. 

In most realistic scenarios it is difficult to calculate the exact gradient of the objective function.  In the next section, we investigate how to extend the presented multi-stage algorithm to the stochastic setting.

\clearpage

\subsection{Stochastic Setting}\label{appendix:spg_entropy}
Following \cref{sec:sspg}, we can construct an stochastic policy gradient using on-policy importance sampling (IS) reward estimates for the entropy regularized objective. 
Let $\htgrad{\thetat}$ denote the stochastic gradient with entropy regularization.
By \cref{lemma:entropy_unbiased_bounded}, the gradient estimators $\htgrad{\theta}$ are (i) unbiased i.e. $\E[\htgrad{\theta}] = \taugrad{\theta}$ and have (ii) bounded variance i.e. $\E \normsq{\htgrad{\theta} - \taugrad{\theta}} \leq \sigma^2$. The bound of the variance is differs compared to $\hgrad{\theta}$ since $\sigma^2$ depends on the regularization strength $\tau$. In this setting, we will consider the following update,
\vspace{0.8ex}
\begin{update}(Stochastic Softmax PG  with Entropy, Importance Sampling)\label{update:entropy_spg}
$\theta_{t+1} = \theta_t + \eta_t \htgrad{\thetat}$.
\end{update}
Under the same setting when using on-policy IS reward estimates, prior work \citep{ding2021beyond} proposes a two-stage approach that converges to a biased optimal policy by modifying the batch size to counteract the variance. However, the method requires a $\tilde{\gO}(\nicefrac{1}{\eps^2})$ sample complexity and knowledge of the biased optimal policy to set the algorithm hyper-parameters. Additionally, even with knowledge of the biased optimal policy, \citet{ding2021beyond} is unable to converge to the optimal policy.

To extend \cref{algorithm:multi_stage_abstract} to the stochastic setting we first require an additional assumption since $\inf_{t \geq 1} [C_\tau(\thetat)]^2$ is a now random variable in the stochastic setting.
\vspace{0.5ex}
\begin{restatable}{assumption}{assumptionexpt} \label{assumption:six}
$\ftau$ satisfies the non-uniform \L ojasiewciz condition for some $C_\tau(\theta)$ and $\xi = \frac{1}{2}$ such that $\mu := \E\left[\inf_{t \geq 1} [C_\tau(\thetat)]^2\right] = \tau^p \, B_1$ 
for constants $p \geq 1$ and $B_1 > 0$.
\end{restatable}
Under \Cref{assumption:six} and motivated by \cref{sec:sspg}, we will utilize exponentially decaying step-sizes \citep{li2021second,vaswani2022towards} for each stage. At stage $i$, the resulting step-size at iteration $t$ is set as: $\eta_{i, t - 1} = \frac{1}{L^{\tau_i}}\, \alpha_i^{t - \text{last}_{i-1}}$ where $\alpha_i = \parens*{\frac{\beta}{T_i}}^{\frac{1}{T_i}}$, $\beta \geq 1$, and $T_i$ is the length of stage $i$. Additionally, $\tau_i$ is the ``temperature'' of stage $i$. All together, this results in~\cref{algorithm:sto_multi_stage_abstract}.

The following theorem (proved in~\cref{appendix:proof_sto_multi_stage}) shows that \cref{algorithm:sto_multi_stage_abstract} converges to the globally optimal policy at an $\tilde{\gO}\left(\nicefrac{1}{\epsilon^p} + \nicefrac{\sigma^2}{\epsilon^{2 p + 1}}\right)$ rate.
\vspace{0.5ex}
\begin{restatable}{theorem}{theoremStoMultiStageAbstract} \label{theorem:sto_multi_stage_abstract}
Assuming $\ftau$ and $f$ satisfy \Cref{assumption:two,assumption:three,assumption:four,assumption:five,assumption:six}, for a given $\eps \in (0, 1)$, using \cref{algorithm:sto_multi_stage_abstract} with (a) unbiased stochastic gradients whose variance is bounded by $\sigma^2$ and (b) exponentially decreasing step-sizes $\eta_{i, t} = \eta_{i, \text{last}_{i-1}} \, \alpha_i^{t - \text{last}_{i-1} + 1}$ where $\eta_{i, \text{last}_{i-1}} = \frac{1}{L^{\tau_i}}$ and $\alpha_{i} = \left(\frac{\beta}{T_i}\right)^{\frac{1}{T_i}}$, $\beta = 1$, achieves $\epsilon$-sub-optimality to the globally optimal policy after $\tilde{\gO}\left(\frac{1}{\epsilon^p} + \frac{\sigma^2}{\epsilon^{2 p + 1}}\right)$ iterations.
\end{restatable}
If $p = 1$, then convergence rate matches the $\tilde{\gO}(\nicefrac{\sigma^2}{\eps^3})$ rate in \cref{theorem:spg_ess}. 
We remark that this the first stochastic softmax PG algorithm to obtain $\eps$-convergece to the optimal policy while using entropy regularization.
Unlike in prior work \citep{ding2021beyond}, oracle-like knowledge of the environment is not necessary to obtain convergence while using entropy regularization in the stochastic setting. 

In the next section, we will compare the multi-stage method with baseline methods in the bandits setting. To investigate if entropy regularization is indeed useufl, we will consider both uniform and ``bad'' initialization.

\clearpage

\subsubsection{Experimental Evaluation}\label{appendix:pg_entropy_experiments}
We evaluate the methods in multi-armed bandit environments with $A = 10$ in stochastic settings.
For each environment, we compare the various algorithms based on their expected sub-optimality gap $\E[(\pistar - \pit)^\top r]$. We plot the average and 95\% confidence interval of the expected sub-optimality gap across $25$ independent bandit instances over $T=10^6$ iterations. 
To counteract the randomness of each algorithm, for each bandit instance we additionally run each algorithm $5$ times. In total, for each algorithm, the corresponding plot is comprised of $125$ runs.
To investigate if entropy regularization is helpful in escaping flat regions, we consider uniform and ``bad'' initialization. 
For experiments with uniform initialization, the initial policy is uniform, i.e. $\pi_{\theta_0}(a) = \nicefrac{1}{A}$ for all $a \in \gA$. For experiments with bad initialization, the initial policy favours the worst arm, i.e. $\theta_0(a') = 9$ ($\pi_{\theta_0}(a') \approx 0.999$), where $a' := \argmin_{a} r(a)$.

\textbf{Environment Details:} 
Each environment's underlying reward distribution is either a Bernoulli, Gaussian, or Beta distribution with a fixed mean reward vector $r \in \R^A$ and support $[0, 1]$. The difficulty of the environment is determined by the maximum reward gap $\bar{\Delta} := \min_{a^* \neq a} r(a^*) - r(a)$. In easy environments $\bar{\Delta} = 0.5$ and in the hard environments $\bar{\Delta} = 0.1$. For each environment, $r$ is randomly generated for each run.

\textbf{Methods:} We compare the presented stochastic softmax PG multi-stage algorithm (\cref{algorithm:sto_multi_stage_abstract}) (\texttt{SPG-E-MS}) to stochastic softmax PG (\texttt{SPG-ESS}) and stochastic softmax PG with entropy regularization (\texttt{SPG-E-ESS}) with exponentially decreasing step-sizes and when using the ``doubling'' trick (\texttt{SPG-ESS [D]}). We also compare with prior work that uses the full gradient (\texttt{SPG-O-G}) \citep{mei2021understanding} and the reward gap (\texttt{SPG-O-R}) \citep{mei2023stochastic} when setting the step-size. For \texttt{SPG-ESS} and \texttt{SPG-ESS [D]}, we select $\beta=1$ and $\eta_0 = \frac{1}{18}$. For \texttt{SPG-E-ESS} we fix $\tau = 0.1$, and similarly select $\beta=1$ and $\eta_0 = \frac{1}{L^\tau} = \frac{2}{5 + 10 \, \tau \, (1 + \log A)}$. Finally, for \texttt{SPG-E-MS}, we observed that the number of iterations $T_i$ at each stage derived by \cref{lemma:expo_stepsize} for the stochastic multistage algorithm are loose due to the exponentially-decreasing step-size analysis. Furthermore, we observed in the deterministic setting that when $p = 1$, the number of iterations doubles after each stage. Therefore, instead of using the theoretical number of iterations at each stage, we use the ``doubling trick'' (refer to \cref{section:exp}). For \texttt{SPG-E-ESS} set the hyper-parameters $T_1 = 5000, \tau_0 = 0.5$, $B_1 = 1$ by employing a grid-search on a separate validation set of bandit instances. To fairly compare against \texttt{SPG-ESS} and \texttt{SPG-ESS [D]} we also select $\beta=1$.
\clearpage
\begin{figure}[!h]
  \begin{center}
    \includegraphics[scale=0.4]{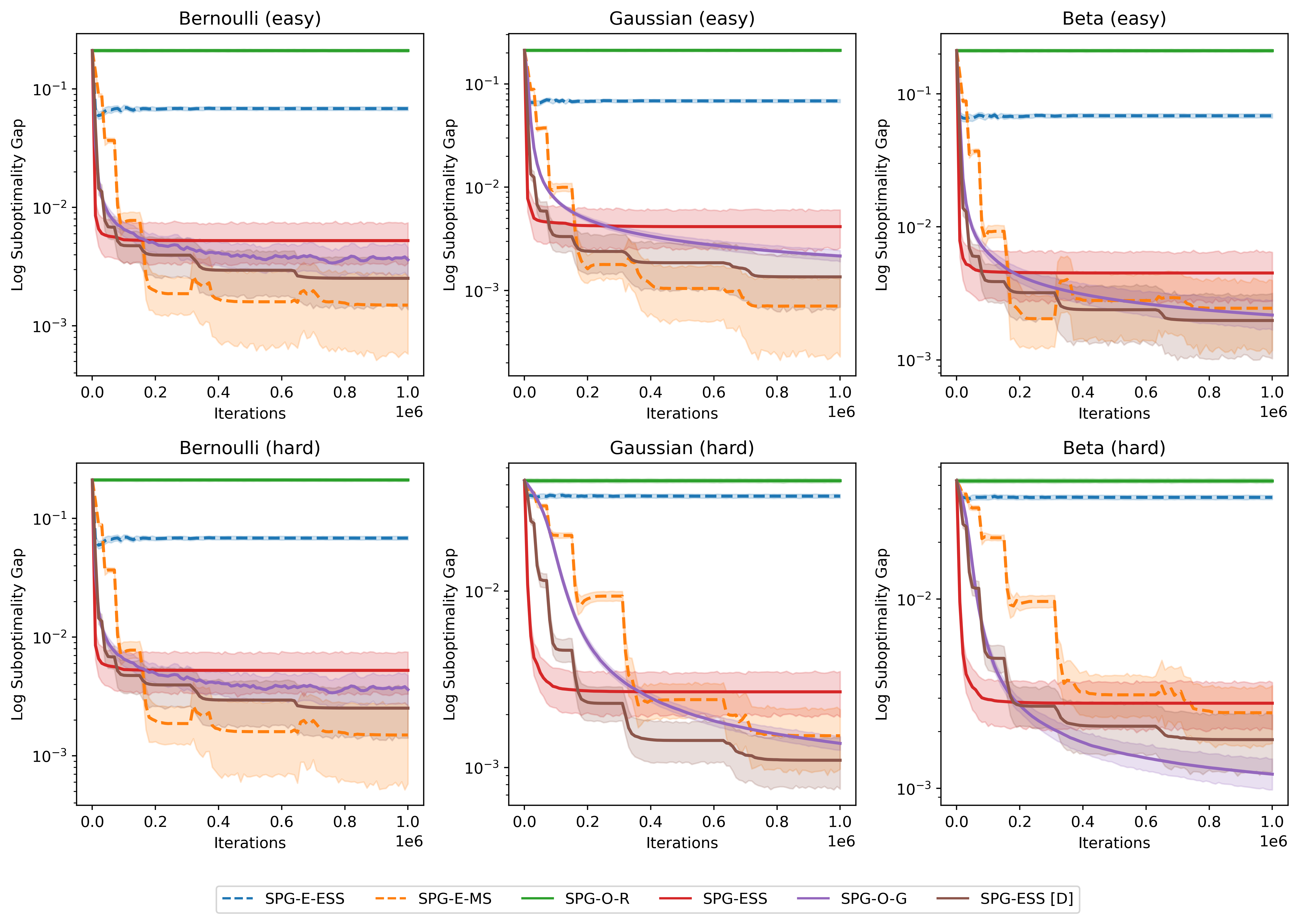}
    \caption{Expected sub-optimality gap across various environments with uniform initialization}
    \label{fig:sto_entropy_good}
  \end{center}
\end{figure}
\vspace{-1.5ex}
\begin{figure}[!h]
  \begin{center}
    \includegraphics[scale=0.4]{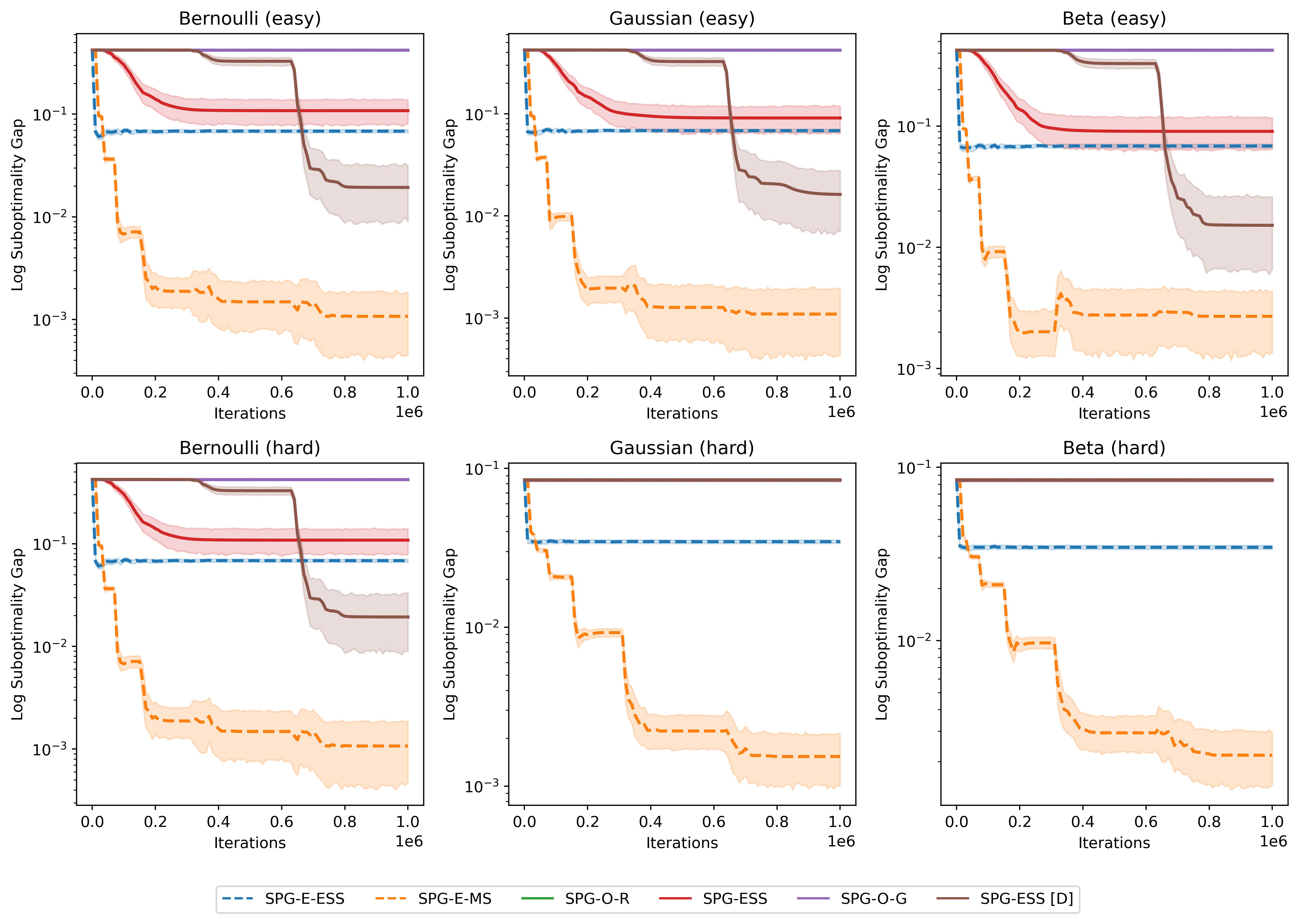}
    \caption{Expected sub-optimality gap across various environments with ``bad'' initialization}
    \label{fig:sto_entropy_bad}
  \end{center}
\end{figure}
\clearpage
\textbf{Results:} From \cref{fig:sto_entropy_good}, with uniform initialization, the performance of \texttt{SPG-E-MS} is comparable to \texttt{SPG-ESS}, \texttt{SPG-ESS [D]} and \texttt{SPG-O-G}. However, in the ``bad'' initialization settings (\cref{fig:sto_entropy_bad}), due to the presence of entropy, \texttt{SPG-E-MS} out preforms all other methods. Here we also find that entropy regularization helps escaping from flat regions in the stochastic setting. Since \texttt{SPG-E-ESS} uses a fixed entropy regularization term it is unable to converge to the optimal policy.

\subsection{Discussion}
We proposed a systematic method for (stochastic) softmax policy gradient (PG) to utilize the benefits of entropy regularization while guaranteeing convergence to the optimal policy. Under \Cref{assumption:one}, our proposed multi-stage algorithm achieves convergence the optimal policy without any oracle-like knowledge when compared to prior methods. We empirically demonstrate that our multi-stage algorithm can escape flat regions in the exact and stochastic settings, due to entropy regularization. For future work, we aim to bridge the non-uniform \L ojasiewciz conditions of $f$ and $\ftau$ as $\tau \rightarrow 0$.

\clearpage

\section{Proofs of \cref{appendix:pg_entropy}}\label{appendix:pg_entropy_proofs}
\begin{algorithm2e}[h]
\DontPrintSemicolon
\caption{Multi-Stage Softmax PG with Entropy Regularization }\label{algorithm:multi_stage_abstract}
    \KwOut{Policy $\pi_{\theta_t} = \mathrm{softmax}(\theta_t)$}
    Initialize parameters $\theta_0, \tau_0, N_\text{stages}$\;
    $t \gets 0$\;
    $\text{last}_0 \gets t$\;
    $i \gets 1$\;
    \While{$i \leq N_\text{stages}$}
    {
        $\tau_i \gets \tau_{i-1} / 2$\;
        $\eta_i \gets 1 / L^{\tau_i}$\;
        $T_i \gets \frac{2}{\eta_i \, \mu_i} \log\left(\frac{\tau_{i-1}}{\tau_i} \left(1 + B_4\right)\right)$\;
        \While{$t - \text{last}_{i-1} < T_i$}
        {
            $\theta_{t+1} \gets \theta_t + \eta_i \nabla f^{\tau_i}(\theta_t)$\;
            $t \gets t + 1$\;
        }
        $\text{last}_i \gets t$\;
        $i \gets i + 1$\;
    }
\end{algorithm2e}

\subsection{Proof of~\cref{theorem:multi_stage_abstract}}\label{appendix:abstract_proofs}

\theoremMultiStageAbstract*
\begin{proof}
Observe that in \cref{algorithm:multi_stage_abstract}, we use $\tau_i$ and $\eta_i$ at stage $i \geq 1$, which starts at iteration $\text{last}_{i-1} + 1$, runs for $T_i = \frac{2}{\eta_i \, \mu_i} \log\left(\frac{\tau_{i-1}}{\tau_i} \left(1 + B_4\right)\right)$ iterations, and ends at iteration $\text{last}_i$. Now, we prove by induction that $f^{*_{\tau_i}} - f^{\tau_i}(\theta_{\text{last}_i}) \leq \tau_i \max\left(1, \frac{f^{*_{\tau_0}} - f^{\tau_0}(\theta_0)}{\tau_0}\right)$ for all $i \geq 0$: \\
\textbf{Base Case:} For $i = 0$, we have
\begin{equation}
    f^{*_{\tau_0}} - f^{\tau_0}(\theta_{0}) \leq \max(\tau_0, f^{*_{\tau_0}} - f^{\tau_0}(\theta_0)) = \tau_0 \max\left(1, \frac{f^{*_{\tau_0}} - f^{\tau_0}(\theta_0)}{\tau_0}\right).
\end{equation}
\textbf{Induction Step:} Suppose $f^{*_{\tau_{i-1}}} - f^{\tau_{i-1}}(\theta_{\text{last}_{i-1}}) \leq \tau_{i-1} \max\left(1, \frac{f^{*_{\tau_0}} - f^{\tau_0}(\theta_0)}{\tau_0}\right)$ holds.

Since $f^{\tau_i}(\theta)$ is $L^{\tau_i}$-smooth and satisfies the non-uniform \L ojasiewciz condition with $\mu_i := \inf_{t \geq 1} C_\tau^2(\thetat)$, we use \cref{lemma:good_pl_rate_abstract} for stage $i$:
\begin{align}
    f^{*_{\tau_i}} - f^{\tau_i}(\theta_{\text{last}_i}) &\leq \exp(- \frac{\eta_i \, \mu_i}{2} \, T_i) [f^{*_{\tau_i}} - f^{\tau_i}(\theta_{\text{last}_{i-1}})]
    \intertext{If $T_i \geq \frac{2}{\eta_i \, \mu_i} \log\left(\frac{\tau_{i-1}}{\tau_i} \left(1 + B_4\right)\right)$, we have}
    &= \frac{f^{*_{\tau_i}} - f^{\tau_i}(\theta_{\text{last}_{i-1}})}{\exp\parens*{\log\left(\frac{\tau_{i-1}}{\tau_i} \left(1 + B_4\right)\right)}}
    \intertext{Under \Cref{assumption:five}}
    &\leq \frac{f^{*_{\tau_{i-1}}} - f^{\tau_{i-1}}(\theta_{\text{last}_{i-1}}) + \tau_{i-1} B_4}{\frac{\tau_{i-1}}{\tau_i} \left(1 + B_4\right)}
    \intertext{Using the inductive hypothesis}
    &\leq \frac{\tau_i \, \tau_{i-1} \left(\max\left(1, \frac{f^{*_{\tau_0}} - f^{\tau_0}(\theta_0)}{\tau_0}\right) + B_4\right)}{\tau_{i-1} \left(1 + B_4\right)} \\
    &\leq \frac{\tau_i \max\left(1, \frac{f^{*_{\tau_0}} - f^{\tau_0}(\theta_0)}{\tau_0}\right) \left(1 + B_4\right)}{1 + B_4} \\
    &= \tau_i \max\left(1, \frac{f^{*_{\tau_0}} - f^{\tau_0}(\theta_0)}{\tau_0}\right).
\end{align}
Therefore, for all $i \geq 0$
\begin{equation}
f^{*_{\tau_i}} - f^{\tau_i}(\theta_{\text{last}_i}) \leq \tau_i \max\left(1, \frac{f^{*_{\tau_0}} - f^{\tau_0}(\theta_0)}{\tau_0}\right)\label{eq:induction_end}.
\end{equation}

Define $\epsilon_i := f^* - f(\theta_{\text{last}_i})$ as the sub-optimality at the end of stage $i$. We have
\begin{align}
    \epsilon_i &= f^* - f(\theta_{\text{last}_i}) \\
    &= \left[ f^* - f(\theta_{\tau_i}^*) \right] + \left[ f(\theta_{\tau_i}^*) - f(\theta_{\text{last}_i}) \right]
    \intertext{Under \Cref{assumption:four}}
    &\leq \left[ f^* - f(\theta_{\tau_i}^*) \right] + f^{*_{\tau_i}} - f^{\tau_i}(\theta_{\text{last}_i}) + \tau_i B_3
    \intertext{By \cref{eq:induction_end},}
    &\leq \left[ f^* - f(\theta_{\tau_i}^*) \right] + \tau_i \parens*{\max\left(1, \frac{f^{*_{\tau_0}} - f^{\tau_0}(\theta_0)}{\tau_0}\right) + B_3}
    \intertext{Using \Cref{assumption:three},}
    &\leq \tau_i \, B_2 + \tau_i \parens*{\max\left(1, \frac{f^{*_{\tau_0}} - f^{\tau_0}(\theta_0)}{\tau_0}\right) + B_3} \\
    &= \tau_i \underbrace{\parens*{\max\left(1, \frac{f^{*_{\tau_0}} - f^{\tau_0}(\theta_0)}{\tau_0}\right) + B_2 + B_3}}_{:= C_1} \\
    &= 2^{-i} \, \tau_0 \, C_1. \tag{$\tau_i = 2^{-i} \, \tau_0$}
\end{align}
Therefore, the number of stages $N_{\text{stages}}$ required to obtain an $\epsilon$ sub-optimality is given as:
\begin{equation}
    2^{N_{\text{stages}}} \geq \frac{\tau_0 \, C_1}{\epsilon} \implies N_{\text{stages}} \geq \log_2\left(\frac{\tau_0 \, C_1}{\epsilon}\right).
    \label{equation:number_of_stages_abstract}
\end{equation}
On the other hand, the sufficient number of iterations at stage $i$ is:
\begin{align}
    T_i & \geq \frac{2}{\eta_i \, \mu_i} \log\left(\frac{\tau_{i-1}}{\tau_i} \left(1 + B_4\right)\right)
    \intertext{Since $\eta_i = \frac{1}{L^{\tau_i}}$}
    &= \frac{2 \, L^{\tau_i}}{\mu_i} \log\left(\frac{\tau_{i-1}}{\tau_i} \left(1 + B_4\right)\right),
    \intertext{Since $L^{\tau_i} \leq L^{\text{max}}$, it is sufficient to set $T_i$ as:}
    T_i &= \frac{2 \, L^{\text{max}}}{\mu_i} \log\left(\frac{\tau_{i-1}}{\tau_i} \left(1 + B_4\right)\right)
    \intertext{Under \Cref{assumption:one}, $\mu_i = \tau_i^p B_1$}
    &= \frac{2 \, L^{\text{max}}}{\tau_i^p \, B_1} \log\left(\frac{\tau_{i-1}}{\tau_i} \left(1 + B_4\right)\right)
    \intertext{Since $\tau_i = 2^{-i} \, \tau_0$, we have}
    &= \frac{2 \, L^{\text{max}} \, 2^{i p}}{\tau_0^p \, B_1} \log\left(2 \left(1 + B_4\right)\right)
\end{align}
Consequently, we can calculate the sufficient total number of iterations $T_\text{Total}$ in terms of $\epsilon$:
\begin{align}
    T_\text{Total} &\geq \sum_{i=1}^{N_{\text{stages}}} T_{i} =  \sum_{i=1}^{N_{\text{stages}}} \left[\frac{2 \, L^{\text{max}} \, 2^{ip}}{\tau_0^p \, B_1} \log\left(2 \left(1 + B_4\right)\right)\right] \\
    &= \frac{2 \, L^{\text{max}} \, \sum_{i=1}^{N_{\text{stages}}} (2^p)^i}{\tau_0^p \, B_1} \log\left(2 \left(1 + B_4\right)\right) \\
    \intertext{Since for all $x > 1, n \geq 0$,  $\sum_{i=0}^{n} x^i = \frac{x^{n + 1} - 1}{x - 1}$}
    &= \frac{2 \, L^{\text{max}} \, \left[\frac{(2^p)^{N_{\text{stages}} +1} - 1}{2^p - 1} - 1\right]}{\tau_0^p \, B_1} \log\left(2 \left(1 + B_4\right)\right) \\
    \intertext{Therefore, it is sufficient that}
    T_\text{Total} \geq &\frac{2 \, L^{\text{max}} \, \frac{(2^p)^{N_{\text{stages}} +1}}{2^p - 1}}{\tau_0^p \, B_1} \log\left(2 \left(1 + B_4\right)\right) \\
    &= \frac{2 \, L^{\text{max}} \, \frac{2^p \, (2^p)^{N_{\text{stages}}}}{2^p - 1}}{\tau_0^p \, B_1} \log\left(2 \left(1 + B_4\right)\right)
    \intertext{Since $p \geq 1$, we have $\frac{2^p}{2^p - 1} \leq 2$. Hence, it is sufficient to use}
    T_\text{Total} = &\frac{4 \, L^{\text{max}} \, (2^p)^{N_{\text{stages}}}}{\tau_0^p \, B_1} \log\left(2 \left(1 + B_4\right)\right) \\
    &= \frac{4 \, L^{\text{max}} \, (2^{N_{\text{stages}}})^p}{\tau_0^p \, B_1} \log\left(2 \left(1 + B_4\right)\right)
    \intertext{Using \cref{equation:number_of_stages_abstract},}
    &\geq \frac{4 \, L^{\text{max}} \, C_1^p}{\epsilon^p \, B_1} \log\left(2 \left(1 + B_4\right)\right)
\end{align}
in order to guarantee $f^* - f(\theta_{T_\text{total}}) \leq \epsilon$. 
\end{proof}

\begin{restatable}{corollary}{theoremMultiStage} \label{theorem:multi_stage}
In the bandit setting,  assuming for each stage $i$, $\mu_i = \tau_i^p B_1$ for constants $p \geq 1$ and $B_1 > 0$, for a given $\eps \in (0, 1)$, using \cref{algorithm:multi_stage_abstract} with $\eta_i = \frac{2}{5 + 10 \, \tau_i \, (1 + \log A)}$ achieves $\epsilon$-sub-optimality after $T_\text{total} = \frac{4 \, L^{\max} \, C_1^p}{\epsilon^p \, B_1} \log\left(2 \left(1 + W\left(\frac{A - 1}{e}\right) + \log{A}\right)\right)$ iterations, where $L^{\max} = \frac{5}{2} + 5 \, (1 + \log A)$ and $C_1 = \max\left(1, \frac{f^{*_{\tau_0}} - f^{\tau_0}(\theta_0)}{\tau_0}\right) + W\left(\frac{A - 1}{e}\right) +  \log{A}$.
\end{restatable}
\begin{proof}
Set $f(\theta) = {\pi_\theta}^\top r$ and $f^\tau(\theta) = {\pi_\theta}^\top (r - \tau \log \pi_\theta)$.  We can extend \cref{theorem:multi_stage_abstract} to the bandit setting since:
\begin{itemize}
[nolistsep]
    \item by~\cref{lemma:bandit_entropy_smooth}, $f^\tau$ is $L^\tau$-smooth and since $\tau \in [0, 1]$
    \begin{equation}
        \frac{5}{2} = L^{\min} \leq L^\tau = \frac{5}{2} + \tau \, 5 \, (1 + \log A) \leq \frac{5}{2} + 5 \, (1 + \log A) = L^{\max}
    \end{equation}
    \item by~\cref{lemma:softmax_bias}, we have $f^* - f(\theta_\tau^*) \leq \tau W\left(\frac{A - 1}{e}\right)$
    \item by~\cref{lemma:subopt_upperbound}, we have for all $\theta$, $f(\theta_\tau^*) - f(\theta) \leq f^{*_\tau} - f^\tau(\theta) + \tau \log A$
    \item by~\cref{lemma:soft_subopt_growth}, we have for all $\theta$, $f^{*_{\tau_2}} - f^{\tau_2}(\theta) \leq f^{*_{\tau_1}} - f^{\tau_1}(\theta) + \tau_1 W\left(\frac{A - 1}{e}\right) + \log{A}$
\end{itemize}
\end{proof}

\begin{restatable}{corollary}{theoremMultiStageMDP} \label{theorem:multi_stage_mdp}
In the tabular MDP setting, assuming for each stage $i$, $\mu_i = \tau_i^p B_1$ for constants $p \geq 1$ and $B_1 > 0$, for a given $\eps \in (0, 1)$, using \cref{algorithm:multi_stage_abstract} with $\eta_i = \frac{(1 - \gamma)^3}{8 + \tau_i \, (4 + 8 \, \log{A})}$ achieves $\epsilon$-sub-optimality after $T_\text{total} = \frac{4 \, L^{\text{max}} \, C_1^p}{\epsilon^p \, B_1} \log\left(2 \left(1 + \frac{2 \, \log{A}}{1 - \gamma} \right)\right)$ iterations, where $L^{\max} = \frac{12 + 8 \log{A}}{(1 - \gamma)^3}$ and $C_1 = \max\left(1, \frac{f^{*_{\tau_0}} - f^{\tau_0}(\theta_0)}{\tau_0}\right) + \frac{2 \, \log{A}}{1 - \gamma}$.
\end{restatable}
\begin{proof}
Set $f(\theta) = V^{\pitheta}(\rho)$ and $f^\tau(\theta) = \tilde{V}_\tau^{\pitheta}(\rho)$. We can extend \cref{theorem:multi_stage_abstract} to the tabular MDP setting since:
\begin{itemize}
[nolistsep]
    \item by~\cref{lemma:smoothness_entropy_mdp}, $f^\tau(\theta)$ is $L^\tau$-smooth and since $\tau \in [0, 1]$
    \begin{equation}
    L^{\min} = \frac{8}{(1 - \gamma)^3} \leq  L^\tau = \frac{8 + \tau (4 + 8 \log{A})}{(1 - \gamma)^3} \leq \frac{12 + 8 \log{A}}{(1 - \gamma)^3} = L^{\max}
    \end{equation}
    \item by~\cref{lemma:softmax_bias_mdp}, we have $f^* - f(\theta_\tau^*) \leq \tau \frac{\log{A}}{1 - \gamma}$
    \item by~\cref{lemma:subopt_upperbound_mdp}, we have for all $\theta$, $f(\theta_\tau^*) - f(\theta) \leq f^{*_\tau} - f^\tau(\theta) + \tau \, \frac{\log{A}}{1 - \gamma}$
   \item by~\cref{lemma:soft_subopt_growth_mdp}, we have for all $\theta$, $f^{*_{\tau_2}} - f^{\tau_2}(\theta) \leq f^{*_{\tau_1}} - f^{\tau_1}(\theta) + \tau_1 \frac{2 \log{A}}{1 - \gamma}$
\end{itemize}
\end{proof}

\subsubsection{Additional Lemmas}
\begin{thmbox}
\begin{restatable}{lemma}{lemmaGoodPLRateAbstract} \label{lemma:good_pl_rate_abstract}
Assuming $\ftau$ satisfies \Cref{assumption:one,assumption:two}, using \cref{update:entropy_pg} with $\etat = \frac{1}{L^\tau}$, we have
\begin{equation}
    f^{*_\tau} - f^\tau(\theta_{t_2}) \leq \exp\left(- \frac{\etat \, \mu}{2} \, (t_2 - t_1)\right) \, [f^{*_\tau} - f^\tau(\theta_{t_1})]
\end{equation}
where $t_1 < t_2$.
\end{restatable}
\end{thmbox}
\begin{proof}
\begin{align}
\intertext{Since $f^\tau$ is $L^\tau$-smooth}
    f^\tau(\theta_{t+1}) &\geq f^\tau(\theta_t) + \langle \nabla f^\tau(\theta_t),\, \theta_{t+1} - \theta_t \rangle - \frac{L^\tau}{2} \normsq{\theta_{t+1} - \theta_t}
    \intertext{Using \cref{update:entropy_pg}, $\theta_{t+1} = \theta_t + \etat \,  \nabla f^\tau(\theta_t)$}
    &= f^\tau(\theta_t) + \eta \normsq{\nabla f^\tau(\theta_t)} - \frac{L^\tau \, \etat^2}{2} ||\nabla f^\tau(\theta_t)||_2^2
    \intertext{Using $\etat = \frac{1}{L^\tau}$}
    &= f^\tau(\theta_t) + \frac{\etat}{2} \normsq{\nabla f^\tau(\theta_t)}
    \intertext{Assuming \Cref{assumption:one} is satisfied, $\norm{\nabla f^\tau(\theta)}_2^2 \geq \mu \, \abs{f^{*_\tau} - f^\tau(\theta)}$}
    &\geq f^\tau(\theta_t) + \frac{\eta \, \mu}{2} \, [f^{*_\tau} - f^\tau(\theta_t)] \\
    \intertext{Multiplying both sides by $-1$ and adding $f^*$}
    \implies f^{*_\tau} - f^\tau(\theta_{t+1}) &\leq \left(1 - \frac{\etat \, \mu}{2} \right) \, [f^{*_\tau} - f^\tau(\theta_t)]  \\
    \intertext{Using $1 - x \leq \exp(-x)$}
    &\leq \exp\left(- \frac{\etat \, \mu}{2}\right) \, [f^{*_\tau} - f^\tau(\theta_t)].
    \intertext{Therefore,}
    f^{*_\tau} - f^\tau(\theta_{t_2}) &\leq \exp\left(- \frac{\etat \, \mu}{2} \, (t_2 - t_1)\right) \, [f^{*_\tau} - f^\tau(\theta_{t_1})].
\end{align}
\end{proof}

\subsection{Lemmas for the Bandit Setting} \label{appendix:verify_bandit}

\subsubsection{Verifying~\cref{assumption:three}}
\begin{thmbox}
\begin{restatable}{lemma}{lemmaEqualRewards} \label{lemma:equal_rewards}
if $\nabla_r \left[(\pi^* - \pi^*_\tau)^\top r\right] = \mathbf{0}$, then all suboptimal rewards must be equal.
\end{restatable}
\end{thmbox}
\begin{proof}
Setting gradient of the bias of softmax optimal policy $(\pi^* - \pi^*_\tau)^\top r$ with respect to the reward vector $r$ equal to a zero vector, the derivative of the bias with respect to an arbitrary suboptimal reward $r(\hat{a})$, where $\hat{a}$ is a suboptimal action, should be $0$:
\begin{align}
    &\frac{d}{d r(\hat{a})} (\pi^* - \pi^*_\tau)^\top r = 0 \implies \frac{d}{d r(\hat{a})} \frac{\sum_{a \neq a^*} e^{\frac{r(a)}{\tau}} \, \Delta(a)}{\sum_{a'} e^{\frac{r(a')}{\tau}}} = 0 \\
    \implies &\frac{\left(\frac{e^{\frac{r(\hat{a})}{\tau}}}{\tau} [r(a^*) - r(\hat{a})] - e^{\frac{r(\hat{a})}{\tau}}\right) \left(\sum_a e^{\frac{r(a)}{\tau}}\right) - \frac{e^{\frac{r(\hat{a})}{\tau}}}{\tau} \left(\sum_a e^{\frac{r(a)}{\tau}} [r(a^*) - r(a)]\right)}{\left(\sum_{a'} e^{\frac{r(a')}{\tau}}\right)^2} = 0 \\
    \implies &\frac{\frac{e^{\frac{r(\hat{a})}{\tau}}}{\tau} \left(\sum_a e^{\frac{r(a)}{\tau}} [r(a) - r(\hat{a}) - \tau]\right)}{\left(\sum_{a'} e^{\frac{r(a')}{\tau}}\right)^2} = 0 \implies \sum_a e^{\frac{r(a)}{\tau}} [r(a) - r(\hat{a}) - \tau] = 0
    \intertext{Now, for any two suboptimal actions $\hat{a}_i$ and $\hat{a}_j$, we have}
    \implies &\sum_a e^{\frac{r(a)}{\tau}} [r(a) - r(\hat{a}_i) - \tau] - \sum_a e^{\frac{r(a)}{\tau}} [r(a) - r(\hat{a}_j) - \tau] = 0 - 0 \\
    \implies &\sum_a e^{\frac{r(a)}{\tau}} [r(\hat{a}_j)) - r(\hat{a}_i)] = 0 \implies r(\hat{a}_j) = r(\hat{a}_i).
\end{align}
Therefore, all suboptimal rewards must be equal.
\end{proof}

\begin{thmbox}
\begin{restatable}{lemma}{lemmaSoftmaxBias} \label{lemma:softmax_bias}
We have $(\pi^* - \pi^*_\tau)^\top r \leq \tau W\left(\frac{A - 1}{e}\right)$, where $W\colon \mathbb{R}^+ \mapsto \mathbb{R}^+$ is the principal branch of the Lambert $W$ function, which is defined by $W(x) e^{W(x)} = x \quad \forall x \geq 0$.
\end{restatable}
\end{thmbox}
\begin{proof}
We want to find an upper bound on the difference between the expected reward achieved by the optimal policy $\pi^*$ and the softmax optimal policy $\pi^*_\tau = \mathrm{softmax}(r / \tau)$. Denoting $\Delta(a) = r(a^*) - r(a)$, $\Delta = \min_{a \neq a^*} \Delta(a)$, and $a^*$ is the optimal action, we have
\begin{equation} 
    (\pi^* - \pi^*_\tau)^\top r = \sum_a \pi^*_\tau(a) \, r(a^*) - \sum_a \pi^*_\tau(a) \, r(a) = \sum_{a \neq a^*} \pi^*_\tau(a) \, \Delta(a) = \frac{\sum_{a \neq a^*} e^{\frac{r(a)}{\tau}} \, \Delta(a)}{\sum_{a'} e^{\frac{r(a')}{\tau}}}.
\end{equation}
To find the upper bound, it is enough to find a reward vector $r \in \mathbb{R}^{A}$ that maximizes the bias. To do so, we find a unique stationary point and then prove that it is the reward vector with the maximum bias. First, we show that decreasing all rewards by a constant value $c$ does not change the bias:
\begin{align}
    (\pi^* - \pi^*_\tau)^\top (r - c\mathbf{1}) = &\frac{\sum_{a \neq a^*} e^{\frac{r(a) - c}{\tau}} \, \Delta(a)}{\sum_{a'} e^{\frac{r(a') - c}{\tau}}} = \frac{e^{- \frac{c}{\tau}} \sum_{a \neq a^*} e^{\frac{r(a)}{\tau}} \, \Delta(a)}{e^{- \frac{c}{\tau}} \sum_{a'} e^{\frac{r(a')}{\tau}}} \\
    = &\frac{\sum_{a \neq a^*} e^{\frac{r(a)}{\tau}} \, \Delta(a)}{\sum_{a'} e^{\frac{r(a')}{\tau}}} = (\pi^* - \pi^*_\tau)^\top r
\end{align}
Therefore, without loss of generality, we assume that the smallest reward value equals $0$. Furthermore, according to \cref{lemma:equal_rewards}, stationary reward vectors must have equal values for all non-optimal actions. Therefore, we assume that the reward vector has a value of $r_{a^*} = \Delta$ for the optimal action and 0 values for all other actions. In this case,
\begin{equation}
    (\pi^* - \pi^*_\tau)^\top r = \frac{\sum_{a \neq a^*} e^{\frac{r(a)}{\tau}} \, \Delta(a)}{\sum_{a'} e^{\frac{r(a')}{\tau}}} = \frac{(A - 1) \Delta}{e^{\frac{\Delta}{\tau}} + A - 1}.
    \label{equation:intermediate_bias_upperbound}
\end{equation}
Now, we find the reward gap $\Delta$ that makes the first derivative of the bias with respect to $\Delta$ equal to $0$:
\begin{align} 
    &\frac{d}{d \Delta} \frac{(A - 1) \Delta}{e^{\frac{\Delta}{\tau}} + A - 1} = 0 \implies \frac{(A - 1) \left(e^{\frac{\Delta}{\tau}} + A - 1\right) - \frac{(A - 1) \Delta e^{\frac{\Delta}{\tau}}}{\tau}}{\left(e^{\frac{\Delta}{\tau}} + A - 1\right)^2} = 0 \\
    \implies &(A - 1) \left(e^{\frac{\Delta}{\tau}} + A - 1\right) - \frac{(A - 1) \Delta e^{\frac{\Delta}{\tau}}}{\tau} = 0 \implies \tau \left(e^{\frac{\Delta}{\tau}} + A - 1\right) = \Delta e^{\frac{\Delta}{\tau}} \\
    \implies &\tau (A - 1) = (\Delta - \tau) e^{\frac{\Delta}{\tau}} \implies \frac{\Delta - \tau}{\tau} e^{\frac{\Delta}{\tau}} = A - 1 \implies \frac{\Delta - \tau}{\tau} e^{\frac{\Delta - \tau}{\tau}} = \frac{A - 1}{e} \\
    \implies &W\left(\frac{A - 1}{e}\right) = \frac{\Delta - \tau}{\tau} \implies \Delta = \tau \left(W\left(\frac{A - 1}{e}\right) + 1\right),
\end{align}
where $W\colon \mathbb{R} \mapsto \mathbb{R}$ is the principal branch of the Lambert $W$ function. Since this value is the only stationary point of the bias with respect to the rewards vector, $\Delta = \tau \left(W\left(\frac{A - 1}{e}\right) + 1\right)$ is either the global maximum or the global minimum point. Since $\pi^*$ is the optimal policy, the bias $(\pi^* - \pi^*_\tau)^\top r$ is always non-negative. For $\Delta = 0$, the bias is equal to $0$, so the unique stationary point must yield the global maximum. Substituting it in \cref{equation:intermediate_bias_upperbound}, we get
\begin{align} 
    (\pi^* - \pi^*_\tau)^\top r \leq &\frac{(A - 1) \tau \left(W\left(\frac{A - 1}{e}\right) + 1\right)}{e^{W\left(\frac{A - 1}{e}\right) + 1} + A - 1}.
    \intertext{Now, since $e^{W(x)} = \frac{x}{W(x)}$,}
    = &\frac{(A - 1) \tau \left(W\left(\frac{A - 1}{e}\right) + 1\right)}{\frac{A - 1}{W\left(\frac{A - 1}{e}\right)} + A - 1} \\
    = &\tau W\left(\frac{A - 1}{e}\right).
\end{align}
\end{proof}

\clearpage
\subsubsection{Verifying~\cref{assumption:four}}
\begin{thmbox}
\begin{restatable}{lemma}{lemmaSubOptUB} \label{lemma:subopt_upperbound}
For a fixed $\theta$ and $\tau$, we have
\begin{equation}
    (\pi_\tau^* - \pi_\theta)^\top r \leq {\pi_\tau^*}^\top (r - \tau \log \pi_\tau^*) - {\pi_\theta}^\top (r - \tau \log \pi_\theta) + \tau \log{A}.
\end{equation}
\end{restatable}
\end{thmbox}
\begin{proof}
\begin{align}
    (\pi_\tau^* - \pi_\theta)^\top r &= {\pi_\tau^*}^\top (r - \tau \log \pi_\tau^*) - {\pi_\theta}^\top (r - \tau \log \pi_\theta) + \tau (\pi_\tau^* \log \pi_\tau^* - \pi_\theta \log \pi_\theta)
    \intertext{For all $\theta$, $\log \frac{1}{A} \leq {\pi_\theta}^\top \log{\pi_\theta} \leq 0$}
    &\leq {\pi_\tau^*}^\top (r - \tau \log \pi_\tau^*) - {\pi_\theta}^\top (r - \tau \log \pi_\theta) + \tau \left(0 - \log{\frac{1}{A}}\right) \\
    &= {\pi_\tau^*}^\top (r - \tau \log \pi_\tau^*) - {\pi_\theta}^\top (r - \tau \log \pi_\theta) + \tau \log{A}.
\end{align}
\end{proof}

\subsubsection{Verifying~\cref{assumption:five}}
\begin{thmbox}
\begin{restatable}{lemma}{lemmaSoftSubOptGrowth} \label{lemma:soft_subopt_growth}
Set $f^\tau(\theta) = {\pi_\theta}^\top (r - \tau \log \pi_\theta)$. For a fixed $\theta$, if $\tau_2 < \tau_1$, then
\begin{equation}
    f^{*_{\tau_2}} - f^{\tau_2}(\theta) \leq f^{*_{\tau_1}} - f^{\tau_1}(\theta) + \tau_1 W\left(\frac{A - 1}{e}\right) + \tau_1 \log{A}.
\end{equation}
\end{restatable}
\end{thmbox}
\begin{proof}
Assuming $\tau_2 < \tau_1$, we have
\begin{align}
    &[f^{*_{\tau_2}} - f^{\tau_2}(\theta)] - [f^{*_{\tau_1}} - f^{\tau_1}(\theta)] =  [f^{*_{\tau_2}} - f^{*_{\tau_1}}] - [f^{\tau_2}(\theta) - f^{\tau_1}(\theta)] \\
    = &\left[{\pi_{\tau_2}^*}^\top (r - \tau_2 \log \pi_{\tau_2}^*) - {\pi_{\tau_1}^*}^\top (r - \tau_1 \log \pi_{\tau_1}^*)\right] - [{\pi_\theta}^\top (r - \tau_2 \log \pi_\theta) - {\pi_\theta}^\top (r - \tau_1 \log \pi_\theta)] \\
    = &(\pi_{\tau_2}^* - \pi_{\tau_1}^*)^\top r - \left[\tau_2 \, {\pi_{\tau_2}^*}^\top \log \pi_{\tau_2}^* - \tau_1 \, {\pi_{\tau_1}^*}^\top \log \pi_{\tau_1}^*\right] + (\tau_2 - \tau_1) \, {\pi_\theta}^\top \log \pi_\theta
    \intertext{For all $\theta$, $\log \frac{1}{A} \leq {\pi_\theta}^\top \log{\pi_\theta} \leq 0$}
    \leq &(\pi_{\tau_2}^* - \pi_{\tau_1}^*)^\top r - \left[\tau_2 \log \frac{1}{A} - \tau_1 \, 0\right] + (\tau_2 - \tau_1) \log \frac{1}{A} \leq (\pi^* - \pi_{\tau_1}^*)^\top r + \tau_1 \log A.
\end{align}
By \cref{lemma:softmax_bias}
\begin{equation}
    \implies f^{*_{\tau_2}} - f^{\tau_2}(\theta) \leq f^{*_{\tau_1}} - f^{\tau_1}(\theta) + \tau_1 W\left(\frac{A - 1}{e}\right) + \tau_1 \log{A}.
\end{equation}
\end{proof}

\clearpage
\subsection{Lemmas for Tabular MDP Setting}
\label{appendix:verify_mdp}
\subsubsection{Verifying~\cref{assumption:three}}
\begin{thmbox}
\begin{restatable}[Equation (12) in \citep{cen2022fast}]{lemma}{lemmaSoftmaxBiasMDP} \label{lemma:softmax_bias_mdp}
$V^*(\rho) - V^{\pi^*_\tau}(\rho) \leq \tau \, \frac{\log{A}}{1 - \gamma}$.
\end{restatable}
\end{thmbox}
\subsubsection{Verifying~\cref{assumption:four}}
\begin{thmbox}
\begin{restatable}{lemma}{lemmaEntropyUBMDP} \label{lemma:entropy_upperbound_mdp}
For any $\pi$ and $\rho$, we have
\begin{equation}
    \mathbb{H}(\pi) \leq \frac{\log A}{1 - \gamma},
\end{equation}
where
\begin{equation}
    \mathbb{H}(\pi) := \mathop{\mathbb{E}}_{\substack{s_0 \sim \rho, a_t \sim \pi(\cdot | s_t), \\ s_{t+1} \sim \mathcal{P}(\cdot | s_t, a_t)}} \left[ \sum_{t=0}^\infty - \gamma^t \log \pi(a_t | s_t) \right].
\end{equation}
\end{restatable}
\end{thmbox}
\begin{proof}
\begin{align}
    \mathbb{H}(\pi) = &\mathop{\mathbb{E}}_{\substack{s_0 \sim \rho, a_t \sim \pi(\cdot | s_t), \\ s_{t+1} \sim \mathcal{P}(\cdot | s_t, a_t)}} \left[ \sum_{t=0}^\infty - \gamma^t \log \pi(a_t | s_t) \right] \\
    = &\frac{1}{1 - \gamma} \sum_{s, a} \, d_\rho^\pi(s) \, \pi(a | s) \, [- \log \pi(a | s)] \\
    = &\frac{1}{1 - \gamma} \sum_s \, d_\rho^\pi(s) \left[- \sum_a \pi(a | s) \, \log \pi(a | s) \right]
    \intertext{Since for all $\pi$, $\log \frac{1}{A} \leq \sum_a \pi(a | s) \, \log \pi(a | s) \leq 0$}
    \leq &\frac{1}{1 - \gamma} \sum_s \, d_\rho^\pi(s) \left[- \log \frac{1}{A} \right] \\
    = &\frac{1}{1 - \gamma} \sum_s \, d_\rho^\pi(s) \, \log A \\
    = &\frac{\log A}{1 - \gamma}
\end{align}
\end{proof}

\begin{thmbox}
\begin{restatable}{lemma}{lemmaSubOptUBMDP} \label{lemma:subopt_upperbound_mdp}
For a fixed $\theta$ and $\tau$, we have
\begin{equation}
    V^{\pi_\tau^*}(\rho) - V^{\pi_\theta}(\rho) \leq \tilde{V}_\tau^*(\rho) - \tilde{V}_\tau^{\pi_\theta}(\rho) + \frac{\tau \log{A}}{1 - \gamma}.
\end{equation}
\end{restatable}
\end{thmbox}
\begin{proof}
\begin{align}
    V^{\pi_\tau^*}(\rho) - V^{\pi_\theta}(\rho) = &(V^{\pi_\tau^*}(\rho) + \tau \mathbb{H}(\rho, \pi_\tau^*)) - (V^{\pi_\theta}(\rho) + \tau \mathbb{H}(\pi_\theta)) + \tau (\mathbb{H}(\pi_\theta) - \mathbb{H}(\pi_\tau^*)) \\
    = &\tilde{V}_\tau^*(\rho) - \tilde{V}_\tau^{\pi_\theta}(\rho) + \tau (\mathbb{H}(\pi_\theta) - \mathbb{H}(\pi_\tau^*))
    \intertext{Since for all $\pi$, $\mathbb{H}(\pi) \geq 0$}
    \leq &\tilde{V}_\tau^*(\rho) - \tilde{V}_\tau^{\pi_\theta}(\rho) + \tau \mathbb{H}(\pi_\theta)
    \intertext{By~\cref{lemma:entropy_upperbound_mdp}}
    \leq &\tilde{V}_\tau^*(\rho) - \tilde{V}_\tau^{\pi_\theta}(\rho) + \frac{\tau \log{A}}{1 - \gamma}
\end{align}
\end{proof}

\subsubsection{Verifying~\cref{assumption:five}}
\begin{thmbox}
\begin{restatable}{lemma}{lemmaSoftSubOptGrowthMDP} \label{lemma:soft_subopt_growth_mdp}
For a fixed $\theta$, if $\tau_2 < \tau_1$, then
\begin{equation}
    \tilde{V}_{\tau_2}^*(\rho) - \tilde{V}_{\tau_2}^{\pi_\theta}(\rho) \leq \tilde{V}_{\tau_1}^*(\rho) - \tilde{V}_{\tau_1}^{\pi_\theta}(\rho) + \frac{2 \, \tau_1 \log{A}}{1 - \gamma}.
\end{equation}
\end{restatable}
\end{thmbox}
\begin{proof}
Assuming $\tau_2 < \tau_1$, we have
\begin{align}
    \tilde{V}_{\tau_2}^*(\rho) - \tilde{V}_{\tau_2}^{\pi_\theta}(\rho) - \tilde{V}_{\tau_1}^*(\rho) - \tilde{V}_{\tau_1}^{\pi_\theta}(\rho) &=  [\tilde{V}_{\tau_2}^*(\rho) - \tilde{V}_{\tau_1}^*(\rho)] - [\tilde{V}_{\tau_2}^{\pi_\theta}(\rho) - \tilde{V}_{\tau_1}^{\pi_\theta}(\rho)] \\
    &= \left[\left(V^{\pi_{\tau_2}^*}(\rho) + \tau_2 \, \mathbb{H}(\pi_{\tau_2}^*)\right) - \left(V^{\pi_{\tau_1}^*}(\rho) + \tau_1 \, \mathbb{H}(\pi_{\tau_1}^*)\right)\right] \notag \\ &- \left[\left(V^{\pi_\theta}(\rho) + \tau_2 \, \mathbb{H}( \pi_\theta)\right) - \left(V^{\pi_\theta}(\rho) + \tau_1 \, \mathbb{H}(\pi_\theta)\right)\right] \\
    &= \left[V^{\pi_{\tau_2}^*}(\rho) - V^{\pi_{\tau_1}^*}(\rho)\right] + \left[\tau_2 \, \mathbb{H}(\pi_{\tau_2}^*) - \tau_1 \, \mathbb{H}(\pi_{\tau_1}^*)\right] \notag \\ &+ (\tau_1 - \tau_2) \, \mathbb{H}(\rho, \pi_\theta). 
    \intertext{By~\cref{lemma:entropy_upperbound_mdp}, $0 \leq \mathbb{H}(\pi) \leq \frac{\log A}{1 - \gamma}$}
    &\leq \left[V^{\pi_{\tau_2}^*}(\rho) - V^{\pi_{\tau_1}^*}(\rho)\right] + \left[\tau_2 \frac{\log A}{1 - \gamma} - \tau_1 \, 0\right] + (\tau_1 - \tau_2) \frac{\log A}{1 - \gamma} \\ 
    &\leq V^*(\rho) - V^{\pi_{\tau_1}^*}(\rho) + \tau_1 \frac{\log A}{1 - \gamma}.
\end{align}
By \cref{lemma:softmax_bias_mdp},
\begin{equation}
    \implies \tilde{V}_{\tau_2}^*(\rho) - \tilde{V}_{\tau_2}^{\pi_\theta}(\rho) \leq \tilde{V}_{\tau_1}^*(\rho) - \tilde{V}_{\tau_1}^{\pi_\theta}(\rho) + \frac{2 \tau_1 \log{A}}{1 - \gamma}.
\end{equation}
\end{proof}
\clearpage

\section{Proofs of \cref{appendix:spg_entropy}}\label{appendix:spg_entropy_proofs}
\begin{algorithm2e}[H]
\DontPrintSemicolon
\caption{Stochastic Multi-Stage Softmax PG with Entropy Regularization }\label{algorithm:sto_multi_stage_abstract}
    \KwOut{Policy $\pi_{\theta_t} = \mathrm{softmax}(\theta_t)$}
    Initialize parameters $\theta_0, \tau_0, N_\text{stages}, \beta = 1$\;
    $t \gets 0$\;
    $\text{last}_0 \gets t$\;
    $i \gets 1$\;
    \While{$i \leq N_\text{stages}$}
    {
        $\tau_i \gets \frac{\tau_{i-1}}{2}$\;
        $X_1 \gets \exp\parens*{\frac{\mu_i \, \beta}{L^\tau \, \log(\nicefrac{T}{\beta})}}$\;
        $X_2 \gets \frac{0.69}{L^\tau}$\;
        $X_3 \gets \frac{5 \, L^\tau \, X_1}{e^2}$\;
        $T_i^{'} \gets \frac{2}{X_2 \, \mu_i} \log\left(\frac{2 \, X_1 \, \tau_{i-1}}{\tau_i} \left(1 + B_4\right)\right)$\;
        $T_i^{''} \gets \frac{2 \, X_3 \, \sigma^2}{\tau_i \, \mu_i^2}$\;
        $T_i \gets \max(5583, 2 \, T_i^{'} \, \log{T_i^{'}}, 4 \, T_i^{''} \, \log^2{T_i^{''}})$\;
        $\alpha_i \gets \parens*{\frac{\beta}{T_i}}^{\frac{1}{T_i}}$\;
        $\eta_{i, t} \gets \frac{\alpha_i}{L^{\tau_i}}$\;
        \While{$t - \text{last}_{i-1} < T_i$}
        {
            $\theta_{t+1} \gets \theta_t + \eta_{i, t} \, \htgrad{\thetat}$\;
            $\eta_{i, t+1} \gets \eta_{i, t} \, \alpha_i$\;
            $t \gets t + 1$\;
        }
        $\text{last}_i \gets t$\;
        $i \gets i + 1$\;
    }
\end{algorithm2e}
\subsection{Proof of Theorem \ref{theorem:sto_multi_stage_abstract}}\label{appendix:proof_sto_multi_stage}
\theoremStoMultiStageAbstract*
\begin{proof}
Observe that in \cref{algorithm:sto_multi_stage_abstract}, we use $\tau_i$ at stage $i \geq 1$, which starts at iteration $\text{last}_{i-1} + 1$, ends at iteration $\text{last}_i$, and runs for $T_i = \max(5583, 2 \, T_i^{'} \, \log{T_i^{'}}, 4 \, T_i^{''} \, \log^2{T_i^{''}})$ iterations, where
\begin{equation}
    T_i^{'} = \frac{2 \, \log\left(\frac{2 \, X_1 \, \tau_{i-1} \left(1 + B_4\right)}{\tau_i}\right)}{X_2 \, \mu_i}, \quad T_i^{''} = \frac{2 \, X_3 \, \sigma^2}{\tau_i \, \mu_i^2},
\end{equation}
where $X_1 = \exp\parens*{\frac{\mu_i \, \beta}{L^{\tau_i} \, \log(T / \beta)}}$, $X_2 = \frac{0.69}{L^{\tau_i}}$, and $X_3 = \frac{5 \, L^{\tau_i} \, X_1}{e^2}$. Now, we will prove by induction that $\mathbb{E}[f^{*_{\tau_i}} - f^{\tau_i}(\theta_{\text{last}_i})] \leq \tau_i \max\left(1, \frac{f^{*_{\tau_0}} - f^{\tau_0}(\theta_0)}{\tau_0}\right)$ for all $i \geq 0$: \\
\textbf{Base Case:} For $i = 0$, we have
\begin{equation}
    f^{*_{\tau_0}} - f^{\tau_0}(\theta_{0}) \leq \max(\tau_0, f^{*_{\tau_0}} - f^{\tau_0}(\theta_0)) = \tau_0 \max\left(1, \frac{f^{*_{\tau_0}} - f^{\tau_0}(\theta_0)}{\tau_0}\right).
\end{equation}
\textbf{Induction Step:} Suppose $\mathbb{E}[f^{*_{\tau_{i-1}}} - f^{\tau_{i-1}}(\theta_{\text{last}_{i-1}})] \leq \tau_{i-1} \max\left(1, \frac{f^{*_{\tau_0}} - f^{\tau_0}(\theta_0)}{\tau_0}\right)$ holds. At stage $i$, by \cref{lemma:expo_stepsize}, using exponentially decreasing step-size $\eta_{i, t} = \eta_{i, \text{last}_{i-1}} \, \alpha_i^{t - \text{last}_{i-1} + 1}$, where $\eta_{i, \text{last}_{i-1}} = \frac{1}{L^{\tau_i}}$, $\alpha_i = \left(\frac{\beta}{T_i}\right)^{\frac{1}{T_i}}$ with $\beta = 1$, for $\mathbb{E}[f^{*_{\tau_i}} - f^{\tau_i}(\theta_{\text{last}_i})] \leq \tau_i \max\left(1, \frac{f^{*_{\tau_0}} - f^{\tau_0}(\theta_0)}{\tau_0}\right)$ to hold, it suffices that $T_i \geq \max(5583, 2 \, Y_i \, \log{Y_i}, 4 \, Y_i^{'} \, \log^2{Y_i^{'}})$, where
\begin{equation}
    Y_i = \frac{2 \, \log\left(\frac{2 \, X_1 \, \mathbb{E}[f^{*_{\tau_i}} - f^{\tau_i}(\theta_{\text{last}_{i-1}})]}{\tau_i \max\left(1, \frac{f^{*_{\tau_0}} - f^{\tau_0}(\theta_0)}{\tau_0}\right)}\right)}{X_2 \, \mu_i}, \quad Y_i^{'} = \frac{2 \, X_3 \, \sigma^2}{\tau_i \, \mu_i^2 \max\left(1, \frac{f^{*_{\tau_0}} - f^{\tau_0}(\theta_0)}{\tau_0}\right)}.
\end{equation}
Under \Cref{assumption:five},
\begin{align}
    Y_i \leq &\frac{2 \, \log\left(\frac{2 \, X_1 \, \left(\mathbb{E}[f^{*_{\tau_{i-1}}} - f^{\tau_{i-1}}(\theta_{\text{last}_{i-1}})] + \tau_{i-1} B_4\right)}{\tau_i \max\left(1, \frac{f^{*_{\tau_0}} - f^{\tau_0}(\theta_0)}{\tau_0}\right)}\right)}{X_2 \, \mu_i}
    \intertext{Using the inductive hypothesis}
    \leq &\frac{2 \, \log\left(\frac{2 \, X_1 \, \left(\tau_{i-1} \max\left(1, \frac{f^{*_{\tau_0}} - f^{\tau_0}(\theta_0)}{\tau_0}\right) + \tau_{i-1} B_4\right)}{\tau_i \max\left(1, \frac{f^{*_{\tau_0}} - f^{\tau_0}(\theta_0)}{\tau_0}\right)}\right)}{X_2 \, \mu_i} \\
    \leq &\frac{2 \, \log\left(\frac{2 \, X_1 \, \tau_{i-1} \max\left(1, \frac{f^{*_{\tau_0}} - f^{\tau_0}(\theta_0)}{\tau_0}\right) \, \left(1 + B_4\right)}{\tau_i \max\left(1, \frac{f^{*_{\tau_0}} - f^{\tau_0}(\theta_0)}{\tau_0}\right)}\right)}{X_2 \, \mu_i} \\
    = &\frac{2 \, \log\left(\frac{2 \, X_1 \, \tau_{i-1} \left(1 + B_4\right)}{\tau_i}\right)}{X_2 \, \mu_i} = T_i^{'}. \\
    \intertext{On the other hand, we have}
    Y_i^{'} \leq &\frac{2 \, X_3 \, \sigma^2}{\tau_i \, \mu_i^2} = T_i^{''}.
\end{align}
Therefore, $T_i = \max(5583, 2 \, T_i^{'} \, \log{T_i^{'}}, 4 \, T_i^{''} \, \log^2{T_i^{''}}) \geq \max(5583, 2 \, Y_i \, \log{Y_i}, 4 \, Y_i^{'} \, \log^2{Y_i^{'}})$. This implies $\mathbb{E}[f^{*_{\tau_i}} - f^{\tau_i}(\theta_{\text{last}_i})] \leq \tau_i \, \max\left(1, \frac{f^{*_{\tau_0}} - f^{\tau_0}(\theta_0)}{\tau_0}\right)$ holds for all $i \geq 0$. As a result, under \Cref{assumption:four}, we have
\begin{align}
    \mathbb{E}[f(\theta_{\tau_i}^*) - f(\theta_{\text{last}_i})] \leq &\mathbb{E}[f^{*_{\tau_i}} - f^{\tau_i}(\theta_{\text{last}_i})] + \tau_i \, B_3 \\
    &\leq \tau_i \, \left(\max\left(1, \frac{f^{*_{\tau_0}} - f^{\tau_0}(\theta_0)}{\tau_0}\right) + B_3\right)
\end{align}
Denote $\epsilon_i := \mathbb{E}[f^* - f(\theta_{\text{last}_i})]$ as the suboptimality at the end of stage $i$. We have
\begin{align}
    \epsilon_i = &\mathbb{E}[f^* - f(\theta_{\text{last}_i})] \\
    &= f^* - f(\theta_{\tau_i}^*) + \mathbb{E}[f(\theta_{\tau_i}^*) - f(\theta_{\text{last}_i})] \\
    \intertext{Under \Cref{assumption:three}}
    &\leq \tau_i \, C_1
    \intertext{where $C_1 = \max\left(1, \frac{f^{*_{\tau_0}} - f^{\tau_0}(\theta_0)}{\tau_0}\right) + B_2 + B_3$. Therefore, $\epsilon_i$ has an upper bound that is proportional to $\tau_i$. Now, since $\tau_i = 2^{-i} \, \tau_0$, the sub-optimality $\epsilon_i$ has an exponential rate in terms of the number of executed stages:}
    &= 2^{-i} \, \tau_0 \, C_1
\end{align}
Therefore, the required number of stages $N_{\text{stages}}$ in terms of the final sub-optimality $\epsilon := \epsilon_{N_{\text{stages}}}$ is
\begin{equation}
    2^{N_{\text{stages}}} \geq \frac{\tau_0 \, C_1}{\epsilon} \implies N_{\text{stages}} \geq \log_2\left(\frac{\tau_0 \, C_1}{\epsilon}\right).
    \label{equation:number_of_stages_sto}
\end{equation}
On the other hand, we have the sufficient number of iterations at stage $i$:
\begin{align}
    T_i &\geq \max\left(5583, \frac{4 \, \log\left(\frac{2 \, X_1 \, \tau_{i-1} (1 + B_4)}{\tau_i}\right)}{X_2 \, \mu_i} \, \log\left(\frac{\log\left(\frac{2 \, X_1 \, \tau_{i-1} (1 + B_4)}{\tau_i}\right)}{X_2 \, \mu_i}\right), \frac{8 \, X_3 \, \sigma^2}{\tau_i \, \mu_i^2} \, \log^2\left(\frac{2 \, X_3 \, \sigma^2}{\tau_i \, \mu_i^2}\right)\right)
\end{align}
Since $\tau_i \leq 1$, under \Cref{assumption:six}, we have $\mu_i = \tau_i^p \, B_1 \leq B_1$. Furthermore, $\log\left(\frac{T_i}{\beta}\right) \geq 1$, and under \Cref{assumption:two}, we have $0 < L^{\min} \leq L^{\tau_i} \leq L^{\max}$. Therefore,
\begin{align}
    X_1 &\leq A_1 = \exp\left(\frac{B_1 \, \beta}{L^{\min}}\right), \\
    X_2 &\geq A_2 = \frac{0.69}{L^{\max}}, \\
    X_3 &\leq A_3 = \frac{5 \, L^{\max} \, A_1}{e^2}.
\end{align}
Hence, we can safely substitute variables $X_1, X_2, X_3$ with their corresponding constants $A_1, A_2, A_3$. Therefore, it is sufficient to set $T_i$ as
\begin{align}
    T_i &\geq \max\left(5583, \frac{4 \, \log\left(\frac{2 \, A_1 \, \tau_{i-1} (1 + B_4)}{\tau_i}\right)}{A_2 \, \mu_i} \, \log\left(\frac{\log\left(\frac{2 \, A_1 \, \tau_{i-1} (1 + B_4)}{\tau_i}\right)}{A_2 \, \mu_i}\right), \frac{8 \, A_3 \, \sigma^2}{\tau_i \, \mu_i^2} \, \log^2\left(\frac{2 \, A_3 \, \sigma^2}{\tau_i \, \mu_i^2}\right)\right)
    \intertext{Under \Cref{assumption:six}, $\mu_i = \tau_i^p \, B_1$}
    &= \max\left(5583, \frac{4 \, \log\left(\frac{2 \, A_1 \, \tau_{i-1} (1 + B_4)}{\tau_i}\right)}{A_2 \, \tau_i^p \, B_1} \, \log\left(\frac{\log\left(\frac{2 \, A_1 \, \tau_{i-1} (1 + B_4)}{\tau_i}\right)}{A_2 \, \tau_i^p \, B_1}\right), \frac{8 \, A_3 \, \sigma^2}{\tau_i^{2 p + 1} \, B_1^2} \, \log^2\left(\frac{2 \, A_3 \, \sigma^2}{\tau_i^{2 p + 1} \, B_1^2}\right)\right)
    \intertext{Since $\tau_i = 2^{-i} \, \tau_0$}
    &= \max\left(5583, \frac{4 \, \log(4 \, A_1 \, (1 + B_4)) \, 2^{ip}}{A_2 \, \tau_0^p \, B_1} \, \log\left(\frac{\log(4 \, A_1 \, (1 + B_4)) \, 2^{ip}}{A_2 \, \tau_0^p \, B_1}\right), \right .\notag \\ & \left. \quad \quad \quad \quad \frac{8 \, A_3 \, \sigma^2 \, 2^{i (2 p + 1)}}{\tau_0^{2 p + 1} \, B_1^2} \, \log^2\left(\frac{2 \, A_3 \, \sigma^2 \, 2^{i (2 p + 1)}}{\tau_0^{2 p + 1} \, B_1^2}\right)\right)
    \intertext{Since $i \leq N_{\text{stages}}$, it is sufficient that}
    T_i = &\max\left(5583, \frac{4 \, \log(4 \, A_1 \, (1 + B_4)) \, 2^{i p}}{A_2 \, \tau_0^p \, B_1} \, Y_1, \frac{8 \, A_3 \, \sigma^2 \, 2^{i (2 p + 1)}}{\tau_0^{2 p + 1} \, B_1^2} \, Y_2\right) 
\end{align}
where $Y_1 = \log\left(\frac{\log(4 \, A_1 \, (1 + B_4)) \, (2^{N_{\text{stages}}})^p}{A_2 \, \tau_0^p \, B_1}\right)$ and $Y_2 = \log^2\left(\frac{2 \, A_3 \, \sigma^2 \, (2^{N_{\text{stages}}})^{2 p + 1}}{\tau_0^{2 p + 1} \, B_1^2}\right)$. Consequently, we can calculate the sufficient total number of iterations $T_\text{Total}$ in terms of $\epsilon$:
\begin{align}
    T_\text{Total} &\geq \sum_{i=1}^{N_{\text{stages}}} T_{i} \\
    &= \sum_{i=1}^{N_{\text{stages}}} \max\left(5583, \frac{4 \, \log(4 \, A_1 \, (1 + B_4)) \, 2^{i p}}{A_2 \, \tau_0^p\, B_1} \, Y_1, \frac{8 \, A_3 \, \sigma^2 \, 2^{i (2 p + 1)}}{\tau_0^{2 p + 1} \, B_1^2} \, Y_2\right) \\
    &= \max\left(5583 \, N_{\text{stages}}, \frac{4 \, \log(4 \, A_1 \, (1 + B_4)) \, \sum_{i=1}^{N_{\text{stages}}} (2^p)^i}{A_2 \, \tau_0^p\, B_1} \, Y_1, \frac{8 \, A_3 \, \sigma^2 \, \sum_{i=1}^{N_{\text{stages}}} (2^{2 p + 1})^i}{\tau_0^{2 p + 1} \, B_1^2} \, Y_2\right) \\
    \intertext{Since $\forall x > 1, n \geq 0$, $\sum_{i=0}^{n} x^i = \frac{x^{n + 1} - 1}{x - 1}$}
    &= \max\left(5583 \, N_{\text{stages}}, \frac{4 \, \log(4 \, A_1 \, (1 + B_4)) \, \left[\frac{(2^p)^{N_{\text{stages}} + 1} - 1}{2^p - 1} - 1\right]}{A_2 \, \tau_0^p \, B_1} \, Y_1, \right. \notag \\ &\left. \quad\quad\quad\quad \frac{8 \, A_3 \, \sigma^2 \, \left[\frac{(2^{2 p + 1})^{N_{\text{stages}} + 1} - 1}{2^{2 p + 1} - 1} - 1\right]}{\tau_0^{2 p + 1} \, B_1^2} \, Y_2\right)
    \intertext{Therefore, it is sufficient that}
    T_\text{Total} \geq &\max\left(5583 \, N_{\text{stages}}, \frac{4 \, \log(4 \, A_1 \, (1 + B_4)) \, \frac{(2^p)^{N_{\text{stages}} + 1}}{2^p - 1}}{A_2 \, \tau_0^p \, B_1} \, Y_1, \frac{8 \, A_3 \, \sigma^2 \, \frac{(2^{2 p + 1})^{N_{\text{stages}} + 1}}{2^{2 p + 1} - 1}}{\tau_0^{2 p + 1} \, B_1^2} \, Y_2\right) \\
    = &\max\left(5583 \, N_{\text{stages}}, \frac{4 \, \log(4 \, A_1 \, (1 + B_4)) \, \frac{2^p \, (2^p)^{N_{\text{stages}}}}{2^p - 1}}{A_2 \, \tau_0^p \, B_1} \, Y_1, \frac{8 \, A_3 \, \sigma^2 \, \frac{2^{2 p + 1} \, (2^{2 p + 1})^{N_{\text{stages}}}}{2^{2 p + 1} - 1}}{\tau_0^{2 p + 1} \, B_1^2} \, Y_2\right)
    \intertext{Since $p \geq 1$, we have $\frac{2^p}{2^p - 1} \leq 2$ and $\frac{2^{2 p + 1}}{2^{2 p + 1} - 1} \leq \frac{8}{7}$. Hence, it is sufficient to use}
    T_\text{Total} = &\max\left(5583 \, N_{\text{stages}}, \frac{8 \, \log(4 \, A_1 \, (1 + B_4)) \, (2^p)^{N_{\text{stages}}}}{A_2 \, \tau_0^p \, B_1} \, Y_1, \frac{64 \, A_3 \, \sigma^2 \, (2^{2 p + 1})^{N_{\text{stages}}}}{7 \, \tau_0^{2 p + 1} \, B_1^2} \, Y_2\right) \\
    = &\max\left(5583 \, N_{\text{stages}}, \frac{8 \, \log(4 \, A_1 \, (1 + B_4)) \, (2^{N_{\text{stages}}})^p}{A_2 \, \tau_0^p \, B_1} \, Y_1, \frac{64 \, A_3 \, \sigma^2 \, (2^{N_{\text{stages}}})^{2 p + 1}}{7 \, \tau_0^{2 p + 1} \, B_1^2} \, Y_2\right)
\intertext{Using \cref{equation:number_of_stages_sto}}
    &\geq \max\left(5583 \, \log_2\left(\frac{\tau_0 \, C_1}{\epsilon}\right), \frac{8 \, \log(4 \, A_1 \, (1 + B_4)) \, C_1^p \, \log\left(\frac{\log(4 \, A_1 \, (1 + B_4)) \, C_1^p}{A_2 \, B_1 \, \epsilon^p}\right)}{A_2 \, B_1 \, \epsilon^p},  \right. \notag \\
    &\quad\qquad\quad \left. \frac{64 \, A_3 \, C_1^{2 p + 1} \, \log^2\left(\frac{2 \, A_3 \, C_1^{2 p + 1} \, \sigma^2}{B_1^2 \, \eps^{2 p + 1}}\right) \, \sigma^2}{7 \, B_1^2 \, \epsilon^{2 p + 1}}\right)
\end{align}
\begin{equation}
    \implies T_\text{Total} \in \tilde{\gO}\left(\frac{1}{\epsilon^p} + \frac{\sigma^2}{\epsilon^{2 p + 1}}\right).
\end{equation}
\end{proof}

\begin{restatable}{corollary}{theoremStoMultiStage} \label{theorem:sto_multi_stage}
In the bandit setting, assuming for each stage $i$, $\mu_i = \tau_i^p B_1$ for constants $p \geq 1$, $B_1 > 0$, for a given $\eps \in (0, 1)$,
using \cref{algorithm:sto_multi_stage_abstract} with  exponentially decreasing step-sizes $\eta_{i, t} = \eta_{i, \text{last}_{i-1}} \, \alpha_i^{t - \text{last}_{i-1} + 1}$  where $\eta_{i, \text{last}_{i-1}} = \frac{2}{5 + 10 \, \tau_i \, (1 + \log A)}$ and $\alpha_{i} = \left( \frac{\beta}{T_i}\right)^{\frac{1}{T_i}}$, $\beta=1$, achieves $\epsilon$-suboptimality after $T_\text{Total} \in \tilde{\gO}\left(\frac{1}{\epsilon^p} + \frac{\sigma^2}{\epsilon^{2 p + 1}}\right)$ iterations. 
\end{restatable}
\begin{proof}
Set $f(\theta) = {\pi_\theta}^\top r$ and $f^\tau(\theta) = {\pi_\theta}^\top (r - \tau \log \pi_\theta)$.  We can extend \cref{theorem:sto_multi_stage_abstract} to the bandit setting since:
\begin{itemize}
[nolistsep]
    \item by~\cref{lemma:bandit_entropy_smooth}, $f^\tau$ is $L^\tau$-smooth and $\tau \in [0, 1]$
    \begin{equation}
        \frac{5}{2} = L^{\min} \leq  L^\tau = \frac{5}{2} + \tau \, 5 \, (1 + \log A) \leq \frac{5}{2} + 5 \, (1 + \log A) = L^{\max}
    \end{equation}
    \item by~\cref{lemma:softmax_bias}, we have $f^* - f(\theta_\tau^*) \leq \tau W\left(\frac{A - 1}{e}\right)$
    \item by~\cref{lemma:subopt_upperbound}, we have for all $\theta$, $f(\theta_\tau^*) - f(\theta) \leq f^{*_\tau} - f^\tau(\theta) + \tau \, \log{A}$
    \item by~\cref{lemma:soft_subopt_growth}, we have for all $\theta$, $f^{*_{\tau_2}} - f^{\tau_2}(\theta) \leq f^{*_{\tau_1}} - f^{\tau_1}(\theta) + \tau_1  W\left(\frac{A - 1}{e}\right) + \log{A}$
    \item by~\cref{lemma:entropy_unbiased_bounded_bandits}, the gradient estimator is unbiased and have bounded variance where $\sigma^2 = 8 \, (1 + (\tau \, \log{A})^2)$. \end{itemize} 
\end{proof}

\begin{restatable}{corollary}{theoremStoMultiStageMDP} \label{theorem:sto_multi_stage_mdp}
In the tabular MDP setting, assuming for each stage $i$ , $\mu_i = \tau_i^p \, B_1$ for constants $p \geq 1$, $B_1 > 0$, for a given $\eps \in (0, 1)$, using \cref{algorithm:sto_multi_stage_abstract} with  exponentially decreasing step-sizes $\eta_{i, t} = \eta_{i, \text{last}_{i-1}} \, \alpha_i^{t - \text{last}_{i-1} + 1}$, where $\eta_{i, \text{last}_{i-1}} = \frac{(1 - \gamma)^3}{8 + \tau_i (4 + 8 \log{A})}$ and $\alpha_{i} = \left( \frac{\beta}{T_i}\right)^{\frac{1}{T_i}}$, $\beta=1$, achieves $\epsilon$-sub-optimality after $T_\text{Total} \in \tilde{\gO}\left(\frac{1}{\epsilon^p} + \frac{\sigma^2}{\epsilon^{2 p + 1}}\right)$ iterations.
\end{restatable}
\begin{proof}
Set $f(\theta) = V^{\pitheta}(\rho)$ and $f^\tau(\theta) = \tilde{V}_\tau^{\pitheta}(\rho)$. We can extend \cref{theorem:sto_multi_stage_abstract} to the MDP setting since:
\begin{itemize}
[nolistsep]
    \item by~\cref{lemma:smoothness_entropy_mdp}, $f^\tau$ is $L^\tau$-smooth and since $\tau \in [0, 1]$
    \begin{equation}
    L^{\min} = \frac{8}{(1 - \gamma)^3} \leq L^\tau = \frac{8 + \tau (4 + 8 \log{A})}{(1 - \gamma)^3} \leq \frac{12 + 8 \log{A}}{(1 - \gamma)^3} = L^{\max}
    \end{equation}
    \item by~\cref{lemma:softmax_bias_mdp}, we have $f^* - f(\theta_\tau^*) \leq \tau \, \frac{\log{A}}{1 - \gamma}$
    \item by~\cref{lemma:subopt_upperbound_mdp}, we have for all $\theta$, $f(\theta_\tau^*) - f(\theta) \leq f^{*_\tau} - f^\tau(\theta) + \tau \, \frac{\log{A}}{1 - \gamma}$
    \item by~\cref{lemma:soft_subopt_growth_mdp}, we have for all $\theta$, $f^{*_{\tau_2}} - f^{\tau_2}(\theta) \leq f^{*_{\tau_1}} - f^{\tau_1}(\theta) + \tau_1 \, \frac{2 \log{A}}{1 - \gamma}$
    \item by \cref{lemma:entropy_unbiased_bounded}, the gradient estimators are unbiased and have bounded variance where $\sigma^2 = \frac{8}{(1 - \gamma)^2} \left(\frac{1 + (\tau \, \log{A})^2}{(1 - \gamma^{1 / 2})^2}\right)$.
\end{itemize}
\end{proof}

\clearpage
\subsubsection{Additional Lemmas}
\begin{thmbox}
\begin{restatable}{lemma}{lemmaExpoStepSize} \label{lemma:expo_stepsize}
Assuming $\ftau$ satisfies \Cref{assumption:six,assumption:two} and the gradient estimators $\htgrad{\thetat}$ are unbiased and have bounded variance $\sigma^2$,  for a given $\eps \in (0, 1)$, using~\cref{update:entropy_spg} from iteration $t_1 + 1$ to $t_2$ with exponentially decreasing step-sizes $\eta_t =  \eta_0 \, \alpha^{t - t_1 + 1}$, where $\etat = \frac{1}{L^\tau}$ and $\alpha = (\frac{\beta}{T})^{\frac{1}{T}}$, $\beta \geq 1$, and $T = t_2 - t_1 > 0$,  is achieved in $\eps$-sub-optimality is achieved in $\max(\beta + 1, 5583, 2 \, Y_1 \, \log{Y_1}, 4 \, Y_2 \, \log^2{Y_2})$ iterations, where $Y_1 = \frac{2 \, \log\left(\frac{2 \, X_1 \, \mathbb{E}[f^{*_\tau} - f^\tau(\theta_{t_1})]}{\epsilon}\right)}{X_2 \, \mu}$, $Y_2 = \frac{2 \, X_3 \, \sigma^2}{\mu^2 \, \epsilon}$ ,$X_1 = \exp\parens*{\frac{\mu \, \beta}{L^\tau \, \log(\nicefrac{T}{\beta})}}$, $X_2 = \frac{0.69}{L^\tau}$, and $X_3 = \frac{5 \, L^\tau \, X_1}{e^2}$.
\end{restatable}
\end{thmbox}
\begin{proof}
From \citep[Theorem 1]{li2021second}, using \cref{update:entropy_spg} with exponentially decreasing step-sizes results from iterations $t_1+1$ to $t_2$ results in the following convergence 
\begin{equation}
    \mathbb{E}[f^{*_\tau} - f^\tau(\theta_{t_2})] \leq X_1 \, \exp\parens*{- \frac{X_2 \, \mu}{2} \, \frac{T}{\log{\frac{T}{\beta}}}} \, \mathbb{E}[f^{*_\tau} - f^\tau(\theta_{t_1})] + \frac{X_3 \, \sigma^2}{\mu^2 \, \frac{T}{\log^2{\frac{T}{\beta}}}},
\end{equation}
where
\begin{equation}
    X_1 = \exp\parens*{\frac{\mu \, \beta}{L^\tau \, \log{\frac{T}{\beta}}}}, \quad X_2 = \frac{0.69}{L^\tau}, \quad X_3 = \frac{5 \, L^\tau \, X_1}{e^2}
\end{equation}
and $\mu := \inf_{t \geq 1} C_\tau(\theta)$ with $T = t_2 - t_1$. We show that if the inequalities $\frac{T}{\log{\frac{T}{\beta}}} \geq Y_1$ and $\frac{T}{\log^2{\frac{T}{\beta}}} \geq Y_2$ are satisfied, where
\begin{equation}
    Y_1 = \frac{2 \, \log\left(\frac{2 \, X_1 \, \mathbb{E}[f^{*_\tau} - f^\tau(\theta_{t_1})]}{\epsilon}\right)}{X_2 \, \mu}, \quad Y_2 = \frac{2 \, X_3 \, \sigma^2}{\mu^2 \, \epsilon},
\end{equation}
then $\mathbb{E}[f^{*_\tau} - f^\tau(\theta_{t_2})] \leq \epsilon$ holds since
\begin{align}
\MoveEqLeft
\mathbb{E}[f^{*_\tau} - f^\tau(\theta_{t_2})] \\ &\leq X_1 \, \exp\parens*{- \frac{X_2 \, \mu}{2} \, \frac{2}{X_2 \, \mu} \log\left(\frac{2 \, X_1 \, [f^{*_\tau} - f^\tau(\theta_{t_1})]}{\epsilon}\right)} \, \mathbb{E}[f^{*_\tau} - f^\tau(\theta_{t_1})] + \frac{X_3 \, \sigma^2}{\mu^2 \, \frac{2 \, X_3 \, \sigma^2}{\mu^2 \, \epsilon}} \\
    &=\frac{\epsilon}{2} + \frac{\epsilon}{2} \\ &= \epsilon.
\end{align}
By \cref{lemma:t_over_logt} and since $1 \leq \beta < T$, for $\frac{T}{\log(\nicefrac{T}{\beta})} \geq \frac{T}{\log{T}} \geq Y_1$ to hold, it suffices that $T \geq \max(2, 2 \, Y_1 \, \log{Y_1})$. Furthermore, according to \cref{lemma:t_over_log2t} and since $1 \leq \beta < T$, for $\frac{T}{\log^2(\nicefrac{T}{\beta})} \geq \frac{T}{\log^2{T}} \geq Y_2$ to hold, it suffices that $T \geq \max(5583, 4 \, Y_2 \, \log^2{Y_2})$. Therefore, the required number of iterations to achieve $\eps$-sub-optimality is $\max(5583, 2 \, Y_1 \, \log{Y_1}, 4 \, Y_2 \, \log^2{Y_2})$.
\end{proof}

\begin{thmbox}
\begin{restatable}{lemma}{lemmaTOverLogT} \label{lemma:t_over_logt}
For all $C > 0$, if $T \geq \max(2, 2 \, C \, \log{C})$, then $\frac{T}{\log{T}} \geq C$.
\end{restatable}
\end{thmbox}
\begin{proof}
If $C < 2$, knowing that $T \geq 2$, we have
\begin{align}
    \frac{T}{\log{T}} &> 2 > C \\
\intertext{Otherwise, if $C \geq 2$,}
    2 \, C \, \log{C} &= C (\log{C} + \log{C})
\intertext{Since $\forall C > 0$, $C \geq 2 \, \log{C}$,}
    &\geq C (\log{C} + \log(2 \log{C})) \\ &= C \log(2 \, C \, \log{C}) \\
    \implies \frac{2 \, C \, \log{C}}{\log(2 \, C \, \log{C})} &\geq C.
\intertext{Therefore, knowing that $T \geq 2 \, C \, \log{C}$, since $2 \, C \, \log{C} \geq 4 \log{2} > 2.72$, we have}
    \frac{T}{\log{T}} &\geq \frac{2 \, C \, \log{C}}{\log(2 \, C \, \log{C})} \geq C.
\end{align}
\end{proof}

\begin{thmbox}
\begin{restatable}{lemma}{lemmaTOverLog2T} \label{lemma:t_over_log2t}
For all $C > 0$, if $T \geq \max(5583, 4 \, C \, \log^2{C})$, then $\frac{T}{\log^2{T}} \geq C$.
\end{restatable}
\end{thmbox}
\begin{proof}
If $C < 75$, knowing that $T \geq 5583$, we have
\begin{equation}
    \frac{T}{\log^2{T}} > 75 > C.
\end{equation}
Otherwise, if $C \geq 75$,
\begin{align}
    4 \, C \, \log^2{C} = &C (\log{C} + \log{C})^2
    \intertext{Since $C \geq 4 \, \log^2{C} \quad \forall C \geq 75$,}
    \geq &C (\log{C} + \log(4 \log^2{C}))^2 = C \log^2(4 \, C \, \log^2{C})
\end{align}
\begin{equation}
    \implies \frac{4 \, C \, \log^2{C}}{\log^2(4 \, C \, \log^2{C})} \geq C.
\end{equation}
Therefore, knowing that $T \geq 4 \, C \, \log^2{C}$, since $4 \, C \, \log^2{C} \geq 300 \log^2{75} > 8$, we have
\begin{equation}
    \frac{T}{\log^2{T}} \geq \frac{4 \, C \, \log^2{C}}{\log^2(4 \, C \, \log^2{C})} \geq C.
\end{equation}
\end{proof}

\section{Additional Experiments}\label{appendix:experiments}
\subsection{Environmental Details}
In each of the following environments, we set the inital state distribution to be uniform, i.e. for all $s \in \gS$, $\rho(s) = \frac{1}{S}$.

\textbf{Cliff World~\citep[Example 6.6]{sutton2018reinforcement}:} The environment consists of $21$ states and $4$ actions. The objective is for an agent to each the goal state while avoiding a cliff. If the agent falls into the chasm, the agent receives a reward of $-100$. If the agent reaches the goal, the agent receives a reward of $+1$. All other rewards are $0$. In this environment $\gamma = 0.9$.

\textbf{Deep Sea Treasure~\citep{osband2019behaviour}:} The environment consists $25$ states and $2$ actions. The agent begin from the top-left corner of the grid and descends one row per each time it takes an action. The goal of the agent is to stay left in order to reach the treasure. If the agent transitions to the right, it receives a reward of $-0.02$. Otherwise if the agent reaches the treasure, it receives a reward of $+1$. In this environment $\gamma = 0.9$.

\textbf{Flat Grad~\citep{agarwal2021theory}:} The environment consists $22$ states and $4$ actions. The agent begin from the left and the objective is for the agent to reach the goal on the far right. For each state, only one action moves the agent to the right while all other actions causes the agent to remain in the same state. The agent only receives a sparse reward of $+1$ when it reaches the goal. In this environment $\gamma = \frac{22}{23}$.

\subsection{Average Run-time Experiments}
We additionally show the average runtime of the compared methods in \cref{fig:det_mdp}. 
\begin{figure}[h]
    \centering
    \includegraphics[scale=0.5]{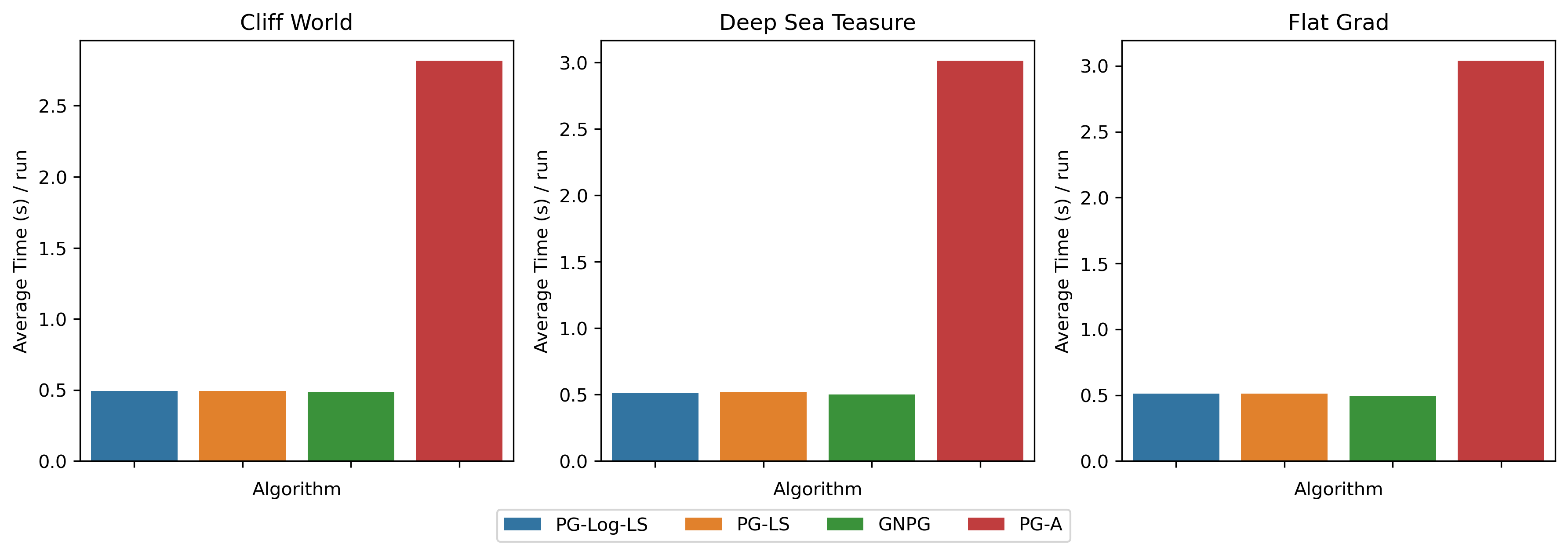}
    \caption{We compare softmax PG that (i) uses a step-size that satisfies the Armijo condition in~\cref{eq:armijo} (denoted as \texttt{PG-LS}), (ii) uses a step-sizes that satisfies the Armijo condition on the log-loss in~\cref{eq:trans_line_search} (\texttt{PG-Log-LS}) to GNPG (\texttt{GNPG}) and PG-A (\texttt{PG-A}). The figure plots the average runtime (in seconds per run) over $50$ runs for each optimization method for across all environments. Although the run time \texttt{PG-LS} and \texttt{PG-Log-LS} are longer, the methods are able to converge faster than \texttt{GNPG}. This justifies the use of line-search despite the marginal increase of runtime.}
    \label{fig:enter-label}
\end{figure}

\section{Extra Lemmas}\label{appendix:extra_lemmas}
For completeness, we append external lemmas here.
\subsection{Smoothness}
\begin{thmbox}
\begin{restatable}[Lemma 2 in \cite{mei2020global}]{lemma}{lemmatwomeiglobal}
\label{lemma:lemma_2_mei_global}
$\forall r \in [0, 1]^{A}$  $\theta \mapsto \dpd{\pitheta,  r}$ is $\frac{5}{2}$-smooth.
\end{restatable}
\end{thmbox}
\begin{thmbox}
\begin{restatable}[Lemma 14 in \citep{mei2020global}]{lemma}{lemmafourteenmeiglobal}
\label{lemma:lemma_14_mei_global}
    $\theta \rightarrow - \dpd{\pi_\theta,  \log \pi_\theta}$ is $5 \, (1 + \log K)$-smooth.
\end{restatable}
\end{thmbox}
\begin{thmbox}
\begin{restatable}{lemma}{lemmapgentropysmooth}
\label{lemma:bandit_entropy_smooth} 
    $\theta \rightarrow \dpd{\pitheta,  r - \tau \log \pitheta}$ is $\frac{5}{2} + \tau \, 5 \,(1 + \log K)$-smooth.
\end{restatable}
\begin{proof}
    By \cref{lemma:lemma_2_mei_global} and \cref{lemma:lemma_14_mei_global}.
\end{proof}
\end{thmbox}
\begin{thmbox}
\begin{restatable}[Lemma 7 in \cite{mei2020global}]{lemma}{lemmasevenmeiglobal}
\label{lemma:lemma_7_mei_global}
$\theta \rightarrow V^{\pitheta}(\rho)$ is $\frac{8}{(1 - \gamma)^3}$-smooth.
\end{restatable}
\end{thmbox}
\begin{thmbox}
\begin{restatable}[Lemmas 7 and 14 in \citep{mei2020global}]{lemma}{smoothnessmdp}
\label{lemma:smoothness_entropy_mdp}
    $\theta \rightarrow  V^{\pi_\theta}(\rho) + \tau \, \mathbb{H}(\pi_\theta)$ is $\frac{8 + \tau \, (4 + 8 \, \log{A})}{(1 - \gamma)^3}$-smooth.
\end{restatable}
\end{thmbox}
\begin{thmbox}
\begin{lemma}[Lemma 2 in \citep{mei2021leveraging}]\label{lemma:bandit_ns}
In the bandits setting, for any $r \in [0, 1]^{A}$, 
$\theta \rightarrow \dpd{\pi_{\theta}, r}$ is $3$-non-uniform smooth.
\end{lemma}
\end{thmbox}
\begin{thmbox}
\begin{lemma}[Lemma 6 in \citep{mei2021leveraging}]\label{lemma:mdp_ns}
In the tabular MDP setting, assuming $\min_{s \in \gS} \rho(s) > 0$,  $\theta \rightarrow V^{\pi_{\theta}}(\rho)$ is $C$-non-uniform smooth with 
where $C := \bracks*{3 + \frac{2 \, C_\infty - (1 - \gamma)}{(1 -\gamma) \,  \gamma}} \, \sqrt{S}$ and $C_\infty := \max_\pi \supnorm{\frac{d^\pi_\rho}{\rho}} \leq \frac{1}{\min_s \rho(s)} < \infty$.
\end{lemma}
\end{thmbox}
\subsubsection{Non-uniform \L ojasiewicz condition}
\begin{thmbox}
\begin{restatable}[Lemma 3 in \cite{mei2020global}]{lemma}{lemmathreemeiglobal}
\label{lemma:lemma_3_mei_global}
Let $\pistar := \max_{\pi \in \Pi} \dpd{\pi, r}$. Then 
\begin{equation}
  \norm*{\frac{d \dpd{\pi_\theta, r}}{d \theta}}_2  \geq C(\theta) \, \dpd{\pistar - \pi_\theta, r}
\end{equation}
where $C(\theta) := \pi_\theta(a^*)$.
\end{restatable}
\end{thmbox}
\begin{thmbox}
\begin{restatable}[Lemma 8 in \cite{mei2020global}]{lemma}{lemmaeightmeiglobal}
\label{lemma:lemma_8_mei_global}
Let $V^*(\rho) := \max_{\pi \in \Pi} V^{\pi}(\rho)$. Then
\begin{equation}
  \norm*{\frac{\partial V^{\pitheta}(\rho)}{\partial \theta}}_2  \geq C(\theta) \, \parens*{V^*(\rho) - V^{\pitheta}(\rho)}
\end{equation}
where $C(\theta) := \frac{\min_s \pitheta(a^*(s) \, | \, s)}{\sqrt{S} \supnorm{\frac{d^{\pistar}_\rho}{d^{\pitheta}_\rho}}}$.
\end{restatable}
\end{thmbox}

\begin{thmbox}
\begin{restatable}[Proposition 5 in \citep{mei2020global}]{lemma}{propfivemeiglobal}
\label{lemma:prop_5_mei_global-bandits_pl}
In the bandits setting, the non-uniform \L ojasiewicz condition is 
\begin{equation}
      \norm*{\frac{d\dpd{\pi_\theta, (r - \tau \log \pi_\theta)}}{d \theta}}_2  \geq C_\tau(\theta) \, \parens*{
      \E_{a \sim \pi^*_\tau}\bracks{r(a) - \tau \log \pi^*_\tau} 
      -
      \E_{a \sim \pi_\theta}\bracks{r(a) - \tau \log \pi_\theta} 
      } ^\frac{1}{2}
\end{equation}
with
\begin{equation}
   C_\tau(\theta) := \sqrt{2 \tau} \, \min_{a} \pi_\theta(a).
\end{equation}
\end{restatable}
\end{thmbox}

\begin{thmbox}
\begin{restatable}[Lemma 15 in \citep{mei2020global}]{lemma}{lemma15meiglobal}
\label{lemma:lemma_15_mei_global_mdp_pl}
In the tabular MDP setting, supposing $\rho(s) > 0$ for all states $s \in \gS$, the non-uniform \L ojasiewicz condition is
\begin{equation}
    \norm*{\frac{\partial \tilde{V}_\tau^{\pi_\theta}(\rho)}{\partial \theta}}_2 \geq C_\tau(\theta) \, \left[ \tilde{V}_\tau^*(\rho) - \tilde{V}_\tau^{\pi_\theta}(\rho) \right]^{\frac{1}{2}}
\end{equation}
with
\begin{equation}
   C_\tau(\theta) := \frac{\sqrt{2 \tau}}{\sqrt{S}} \, \min_{s} \sqrt{\rho(s)} \, \min_{s, a} \pi_\theta(a | s) \, \norm*{\frac{d_{\rho}^{\pi_\tau^*}}{d_{\rho}^{\pi_\theta}}}_{\infty}^{- \frac{1}{2}}.
\end{equation}
\end{restatable}
\end{thmbox}

\subsection{Stochastic Policy Gradients}
\begin{thmbox}
\begin{restatable}[Lemma 5 from \citep{mei2021understanding}]{lemma}{lemmafivemeiunderstanding}
\label{lemma:lemma_5_mei_understanding}
Let $\hat{r}$ be the IS estimator using on-policy sampling $a \sim \pitheta(\cdot)$. Then stochastic softmax PG estimator is: \\
\textbf{Unbiased}: $\E_{a \sim \pi_{\theta}} \bracks*{\hgrad{\theta}} = \grad{\theta}$
\\
\textbf{Bounded Variance}: $\E_{a \sim \pi_{\theta}} \normsq{\hgrad{\theta}} \leq 2 \Rightarrow \sigma^2 := \E_{a \sim \pi_{\theta}} \bracks*{\hgrad{\theta} - \grad{\theta}} = \E_{a \sim \pi_{\theta}} \normsq{\hgrad{\theta}} - \E_{a \sim \pi_{\theta}} \normsq{\grad{\theta}} \leq 2$.
\end{restatable}
\end{thmbox}
\begin{thmbox}
\begin{restatable}[Lemma 11 from \citep{mei2021understanding}]{lemma}{lemmaelevenmeiunderstanding}
\label{lemma:lemma_11_mei_understanding}
Let $\hat{Q}^{\pitheta}$ be the IS estimator using on-policy sampling $a(s) \sim \pitheta(\cdot | s)$. Then stochastic softmax PG estimator is: \\
\textbf{Unbiased}: $\E \bracks*{\htgrad{\theta}} = \taugrad{\theta}$.
\\
\textbf{Bounded Variance}: $\E \normsq{\hgrad{\theta}} \leq \frac{2 \, S}{(1 - \gamma)^4} \Rightarrow \sigma^2 := \E \bracks*{\hgrad{\theta} - \grad{\theta}}  \leq \frac{2 \, S}{(1 - \gamma)^4}$.
\end{restatable}
\end{thmbox}

\begin{thmbox}
\begin{lemma}[Lemma 3 and Lemma 4 from \citep{ding2021beyond}] \label{lemma:entropy_unbiased_bounded}
Let $\hat{Q}_\tau^{\pitheta}$ be the entropy regularized IS estimator using on-policy sampling $a(s) \sim \pitheta(\cdot | s)$. Then stochastic softmax PG estimator using entropy regularization is: \\
\textbf{Unbiased}: $\E \bracks*{\htgrad{\theta}} = \taugrad{\theta}$.
\\
\textbf{Bounded Variance}: $\E \normsq{\htgrad{\theta} - \E[\htgrad{\theta}]} \leq \sigma^2$, where $\sigma^2 = \frac{8}{(1 - \gamma)^2} \left(\frac{1 + (\tau \, \log{A})^2}{(1 - \gamma^{1 / 2})^2}\right)$.
\end{lemma}
\end{thmbox}

\begin{thmbox}
\begin{lemma}[Instantiation of \cref{lemma:entropy_unbiased_bounded} in the bandits setting] \label{lemma:entropy_unbiased_bounded_bandits}
Let $\hat{r}$ be the entropy regularized IS estimator using on-policy sampling $a \sim \pitheta(\cdot)$. Then stochastic softmax PG estimator using entropy regularization is: \\
\textbf{Unbiased}: $\E \bracks*{\htgrad{\theta}} = \taugrad{\theta}$.
\\
\textbf{Bounded Variance}: $\E \normsq{\htgrad{\theta} - \E[\htgrad{\theta}]} \leq \sigma^2$, where $\sigma^2 = 8 \, (1 + (\tau \, \log{A})^2)$.
\end{lemma}
\end{thmbox}


\begin{thebibliography}{45}
\providecommand{\natexlab}[1]{#1}
\providecommand{\url}[1]{\texttt{#1}}
\expandafter\ifx\csname urlstyle\endcsname\relax
  \providecommand{\doi}[1]{doi: #1}\else
  \providecommand{\doi}{doi: \begingroup \urlstyle{rm}\Url}\fi

\bibitem[Agarwal et~al.(2021)Agarwal, Kakade, Lee, and
  Mahajan]{agarwal2021theory}
Alekh Agarwal, Sham~M Kakade, Jason~D Lee, and Gaurav Mahajan.
\newblock On the theory of policy gradient methods: Optimality, approximation,
  and distribution shift.
\newblock \emph{J. Mach. Learn. Res.}, 22\penalty0 (98):\penalty0 1--76, 2021.

\bibitem[Agrawal \& Goyal(2012)Agrawal and Goyal]{agrawal2012analysis}
Shipra Agrawal and Navin Goyal.
\newblock Analysis of thompson sampling for the multi-armed bandit problem.
\newblock In \emph{Conference on learning theory}, pp.\  39--1. JMLR Workshop
  and Conference Proceedings, 2012.

\bibitem[Ahmed et~al.(2019)Ahmed, Le~Roux, Norouzi, and
  Schuurmans]{ahmed2019understanding}
Zafarali Ahmed, Nicolas Le~Roux, Mohammad Norouzi, and Dale Schuurmans.
\newblock Understanding the impact of entropy on policy optimization.
\newblock In \emph{International conference on machine learning}, pp.\
  151--160. PMLR, 2019.

\bibitem[Altman(2021)]{altman2021constrained}
Eitan Altman.
\newblock \emph{Constrained Markov decision processes}.
\newblock Routledge, 2021.

\bibitem[Armijo(1966)]{armijo1966}
Larry Armijo.
\newblock {Minimization of functions having Lipschitz continuous first partial
  derivatives.}
\newblock \emph{Pacific Journal of Mathematics}, 16\penalty0 (1):\penalty0 1 --
  3, 1966.

\bibitem[Auer et~al.(1995)Auer, Cesa-Bianchi, Freund, and
  Schapire]{auer1995gambling}
Peter Auer, Nicolo Cesa-Bianchi, Yoav Freund, and Robert~E Schapire.
\newblock Gambling in a rigged casino: The adversarial multi-armed bandit
  problem.
\newblock In \emph{Proceedings of IEEE 36th annual foundations of computer
  science}, pp.\  322--331. IEEE, 1995.

\bibitem[Auer et~al.(2002)Auer, Cesa-Bianchi, and Fischer]{auer2002finite}
Peter Auer, Nicolo Cesa-Bianchi, and Paul Fischer.
\newblock Finite-time analysis of the multiarmed bandit problem.
\newblock \emph{Machine learning}, 47:\penalty0 235--256, 2002.

\bibitem[Bhandari \& Russo(2021)Bhandari and Russo]{bhandari2021linear}
Jalaj Bhandari and Daniel Russo.
\newblock On the linear convergence of policy gradient methods for finite mdps.
\newblock In \emph{International Conference on Artificial Intelligence and
  Statistics}, pp.\  2386--2394. PMLR, 2021.

\bibitem[Cen et~al.(2022)Cen, Cheng, Chen, Wei, and Chi]{cen2022fast}
Shicong Cen, Chen Cheng, Yuxin Chen, Yuting Wei, and Yuejie Chi.
\newblock Fast global convergence of natural policy gradient methods with
  entropy regularization.
\newblock \emph{Operations Research}, 70\penalty0 (4):\penalty0 2563--2578,
  2022.

\bibitem[Ding et~al.()Ding, Zhang, Lee, and Lavaei]{ding2021beyond}
Yuhao Ding, Junzi Zhang, Hyunin Lee, and Javad Lavaei.
\newblock Beyond exact gradients: Convergence of stochastic soft-max policy
  gradient methods with entropy regularization.

\bibitem[Haarnoja et~al.(2018)Haarnoja, Zhou, Abbeel, and
  Levine]{haarnoja2018soft}
Tuomas Haarnoja, Aurick Zhou, Pieter Abbeel, and Sergey Levine.
\newblock Soft actor-critic: Off-policy maximum entropy deep reinforcement
  learning with a stochastic actor.
\newblock In \emph{International conference on machine learning}, pp.\
  1861--1870. PMLR, 2018.

\bibitem[Hazan \& Kale(2014)Hazan and Kale]{hazan2014beyond}
Elad Hazan and Satyen Kale.
\newblock Beyond the regret minimization barrier: optimal algorithms for
  stochastic strongly-convex optimization.
\newblock \emph{The Journal of Machine Learning Research}, 15\penalty0
  (1):\penalty0 2489--2512, 2014.

\bibitem[Hiraoka et~al.(2022)Hiraoka, Imagawa, Hashimoto, Onishi, and
  Tsuruoka]{hiraoka2022dropout}
Takuya Hiraoka, Takahisa Imagawa, Taisei Hashimoto, Takashi Onishi, and
  Yoshimasa Tsuruoka.
\newblock Dropout q-functions for doubly efficient reinforcement learning.
\newblock In \emph{International Conference on Learning Representations}, 2022.
\newblock URL \url{https://openreview.net/forum?id=xCVJMsPv3RT}.

\bibitem[Ji \& Telgarsky(2018)Ji and Telgarsky]{ji2018risk}
Ziwei Ji and Matus Telgarsky.
\newblock Risk and parameter convergence of logistic regression.
\newblock \emph{arXiv preprint arXiv:1803.07300}, 2018.

\bibitem[Kakade(2001)]{kakade2001natural}
Sham~M Kakade.
\newblock A natural policy gradient.
\newblock \emph{Advances in neural information processing systems}, 14, 2001.

\bibitem[Karimi et~al.(2016)Karimi, Nutini, and Schmidt]{karimi2016linear}
Hamed Karimi, Julie Nutini, and Mark Schmidt.
\newblock Linear convergence of gradient and proximal-gradient methods under
  the polyak-{\l}ojasiewicz condition.
\newblock In \emph{Machine Learning and Knowledge Discovery in Databases:
  European Conference, ECML PKDD 2016, Riva del Garda, Italy, September 19-23,
  2016, Proceedings, Part I 16}, pp.\  795--811. Springer, 2016.

\bibitem[Lan(2023)]{lan2023policy}
Guanghui Lan.
\newblock Policy mirror descent for reinforcement learning: Linear convergence,
  new sampling complexity, and generalized problem classes.
\newblock \emph{Mathematical programming}, 198\penalty0 (1):\penalty0
  1059--1106, 2023.

\bibitem[Lattimore \& Szepesv{\'a}ri(2020)Lattimore and
  Szepesv{\'a}ri]{lattimore2020bandit}
Tor Lattimore and Csaba Szepesv{\'a}ri.
\newblock \emph{Bandit algorithms}.
\newblock Cambridge University Press, 2020.

\bibitem[Li et~al.(2021)Li, Zhuang, and Orabona]{li2021second}
Xiaoyu Li, Zhenxun Zhuang, and Francesco Orabona.
\newblock A second look at exponential and cosine step sizes: Simplicity,
  adaptivity, and performance.
\newblock In \emph{International Conference on Machine Learning}, pp.\
  6553--6564. PMLR, 2021.

\bibitem[Liu et~al.(2024)Liu, Li, and Wei]{liu2024elementary}
Jiacai Liu, Wenye Li, and Ke~Wei.
\newblock Elementary analysis of policy gradient methods.
\newblock \emph{arXiv preprint arXiv:2404.03372}, 2024.

\bibitem[Mei et~al.(2020)Mei, Xiao, Szepesvari, and Schuurmans]{mei2020global}
Jincheng Mei, Chenjun Xiao, Csaba Szepesvari, and Dale Schuurmans.
\newblock On the global convergence rates of softmax policy gradient methods.
\newblock In \emph{International Conference on Machine Learning}, pp.\
  6820--6829. PMLR, 2020.

\bibitem[Mei et~al.(2021{\natexlab{a}})Mei, Dai, Xiao, Szepesvari, and
  Schuurmans]{mei2021understanding}
Jincheng Mei, Bo~Dai, Chenjun Xiao, Csaba Szepesvari, and Dale Schuurmans.
\newblock Understanding the effect of stochasticity in policy optimization.
\newblock \emph{Advances in Neural Information Processing Systems},
  34:\penalty0 19339--19351, 2021{\natexlab{a}}.

\bibitem[Mei et~al.(2021{\natexlab{b}})Mei, Gao, Dai, Szepesvari, and
  Schuurmans]{mei2021leveraging}
Jincheng Mei, Yue Gao, Bo~Dai, Csaba Szepesvari, and Dale Schuurmans.
\newblock Leveraging non-uniformity in first-order non-convex optimization.
\newblock In \emph{International Conference on Machine Learning}, pp.\
  7555--7564. PMLR, 2021{\natexlab{b}}.

\bibitem[Mei et~al.(2022)Mei, Chung, Thomas, Dai, Szepesvari, and
  Schuurmans]{mei2022the}
Jincheng Mei, Wesley Chung, Valentin Thomas, Bo~Dai, Csaba Szepesvari, and Dale
  Schuurmans.
\newblock The role of baselines in policy gradient optimization.
\newblock \emph{Advances in Neural Information Processing Systems},
  35:\penalty0 17818--17830, 2022.

\bibitem[Mei et~al.(2023)Mei, Zhong, Dai, Agarwal, Szepesvari, and
  Schuurmans]{mei2023stochastic}
Jincheng Mei, Zixin Zhong, Bo~Dai, Alekh Agarwal, Csaba Szepesvari, and Dale
  Schuurmans.
\newblock Stochastic gradient succeeds for bandits.
\newblock In Andreas Krause, Emma Brunskill, Kyunghyun Cho, Barbara Engelhardt,
  Sivan Sabato, and Jonathan Scarlett (eds.), \emph{Proceedings of the 40th
  International Conference on Machine Learning}, volume 202 of
  \emph{Proceedings of Machine Learning Research}, pp.\  24325--24360. PMLR,
  23--29 Jul 2023.
\newblock URL \url{https://proceedings.mlr.press/v202/mei23a.html}.

\bibitem[Nocedal \& Wright()Nocedal and Wright]{nocedal1999numerical}
Jorge Nocedal and Stephen~J Wright.
\newblock \emph{Numerical optimization}.
\newblock Springer.

\bibitem[Osband et~al.(2019)Osband, Doron, Hessel, Aslanides, Sezener, Saraiva,
  McKinney, Lattimore, Szepesvari, Singh, et~al.]{osband2019behaviour}
Ian Osband, Yotam Doron, Matteo Hessel, John Aslanides, Eren Sezener, Andre
  Saraiva, Katrina McKinney, Tor Lattimore, Csaba Szepesvari, Satinder Singh,
  et~al.
\newblock Behaviour suite for reinforcement learning.
\newblock \emph{arXiv preprint arXiv:1908.03568}, 2019.

\bibitem[Polyak(1987)]{polyak1987introduction}
Boris~T Polyak.
\newblock Introduction to optimization.
\newblock 1987.

\bibitem[Polyak(1963)]{POLYAK1963864}
B.T. Polyak.
\newblock Gradient methods for the minimisation of functionals.
\newblock \emph{USSR Computational Mathematics and Mathematical Physics},
  3\penalty0 (4):\penalty0 864--878, 1963.
\newblock ISSN 0041-5553.
\newblock \doi{https://doi.org/10.1016/0041-5553(63)90382-3}.
\newblock URL
  \url{https://www.sciencedirect.com/science/article/pii/0041555363903823}.

\bibitem[Puterman(2014)]{puterman2014markov}
Martin~L Puterman.
\newblock \emph{Markov decision processes: discrete stochastic dynamic
  programming}.
\newblock John Wiley \& Sons, 2014.

\bibitem[Schmidt \& Roux(2013)Schmidt and Roux]{schmidt2013fast}
Mark Schmidt and Nicolas~Le Roux.
\newblock Fast convergence of stochastic gradient descent under a strong growth
  condition.
\newblock \emph{arXiv preprint arXiv:1308.6370}, 2013.

\bibitem[Schulman et~al.(2017)Schulman, Wolski, Dhariwal, Radford, and
  Klimov]{schulman2017proximal}
John Schulman, Filip Wolski, Prafulla Dhariwal, Alec Radford, and Oleg Klimov.
\newblock Proximal policy optimization algorithms.
\newblock \emph{arXiv preprint arXiv:1707.06347}, 2017.

\bibitem[Shani et~al.(2020)Shani, Efroni, and Mannor]{shani2020adaptive}
Lior Shani, Yonathan Efroni, and Shie Mannor.
\newblock Adaptive trust region policy optimization: Global convergence and
  faster rates for regularized mdps.
\newblock In \emph{Proceedings of the AAAI Conference on Artificial
  Intelligence}, volume~34, pp.\  5668--5675, 2020.

\bibitem[Sutton \& Barto(2018)Sutton and Barto]{sutton2018reinforcement}
Richard~S Sutton and Andrew~G Barto.
\newblock \emph{Reinforcement learning: An introduction}.
\newblock MIT press, 2018.

\bibitem[Sutton et~al.(1999{\natexlab{a}})Sutton, McAllester, Singh, and
  Mansour]{10.5555/3009657.3009806}
Richard~S. Sutton, David McAllester, Satinder Singh, and Yishay Mansour.
\newblock Policy gradient methods for reinforcement learning with function
  approximation.
\newblock In \emph{Proceedings of the 12th International Conference on Neural
  Information Processing Systems}, NIPS'99, pp.\  1057–1063, Cambridge, MA,
  USA, 1999{\natexlab{a}}. MIT Press.

\bibitem[Sutton et~al.(1999{\natexlab{b}})Sutton, McAllester, Singh, and
  Mansour]{sutton1999policy}
Richard~S Sutton, David McAllester, Satinder Singh, and Yishay Mansour.
\newblock Policy gradient methods for reinforcement learning with function
  approximation.
\newblock \emph{Advances in neural information processing systems}, 12,
  1999{\natexlab{b}}.

\bibitem[Vaswani et~al.(2019)Vaswani, Bach, and Schmidt]{vaswani2019fast}
Sharan Vaswani, Francis Bach, and Mark Schmidt.
\newblock Fast and faster convergence of sgd for over-parameterized models and
  an accelerated perceptron.
\newblock In \emph{The 22nd international conference on artificial intelligence
  and statistics}, pp.\  1195--1204. PMLR, 2019.

\bibitem[Vaswani et~al.(2022)Vaswani, Dubois-Taine, and
  Babanezhad]{vaswani2022towards}
Sharan Vaswani, Benjamin Dubois-Taine, and Reza Babanezhad.
\newblock Towards noise-adaptive, problem-adaptive (accelerated) stochastic
  gradient descent.
\newblock In \emph{International Conference on Machine Learning}, pp.\
  22015--22059. PMLR, 2022.

\bibitem[Williams(1992)]{williams1992simple}
Ronald~J Williams.
\newblock Simple statistical gradient-following algorithms for connectionist
  reinforcement learning.
\newblock \emph{Machine learning}, 8:\penalty0 229--256, 1992.

\bibitem[Xiao(2022)]{xiao2022convergence}
Lin Xiao.
\newblock On the convergence rates of policy gradient methods.
\newblock \emph{arXiv preprint arXiv:2201.07443}, 2022.

\bibitem[Yuan et~al.(2022)Yuan, Gower, and Lazaric]{yuan2022general}
Rui Yuan, Robert~M. Gower, and Alessandro Lazaric.
\newblock A general sample complexity analysis of vanilla policy gradient,
  2022.

\bibitem[Zahavy et~al.(2021)Zahavy, O'Donoghue, Desjardins, and
  Singh]{zahavy2021reward}
Tom Zahavy, Brendan O'Donoghue, Guillaume Desjardins, and Satinder Singh.
\newblock Reward is enough for convex mdps.
\newblock \emph{Advances in Neural Information Processing Systems},
  34:\penalty0 25746--25759, 2021.

\bibitem[Zhang et~al.(2019)Zhang, He, Sra, and Jadbabaie]{zhang2019gradient}
Jingzhao Zhang, Tianxing He, Suvrit Sra, and Ali Jadbabaie.
\newblock Why gradient clipping accelerates training: A theoretical
  justification for adaptivity.
\newblock \emph{arXiv preprint arXiv:1905.11881}, 2019.

\bibitem[Zhang et~al.(2020{\natexlab{a}})Zhang, Koppel, Bedi, Szepesvari, and
  Wang]{zhang2020variational}
Junyu Zhang, Alec Koppel, Amrit~Singh Bedi, Csaba Szepesvari, and Mengdi Wang.
\newblock Variational policy gradient method for reinforcement learning with
  general utilities.
\newblock \emph{Advances in Neural Information Processing Systems},
  33:\penalty0 4572--4583, 2020{\natexlab{a}}.

\bibitem[Zhang et~al.(2020{\natexlab{b}})Zhang, Koppel, Zhu, and
  Basar]{zhang2020global}
Kaiqing Zhang, Alec Koppel, Hao Zhu, and Tamer Basar.
\newblock Global convergence of policy gradient methods to (almost) locally
  optimal policies.
\newblock \emph{SIAM Journal on Control and Optimization}, 58\penalty0
  (6):\penalty0 3586--3612, 2020{\natexlab{b}}.

\end{thebibliography}
\end{document}